\documentclass[11pt]{article}
\pdfoutput=1

\usepackage[margin=1in,includefoot,footskip=30pt]{geometry}
\usepackage[width=474.18663pt]{caption}

\usepackage[T1]{fontenc}
\usepackage[utf8]{inputenc}
\usepackage{lmodern}

\usepackage[table,xcdraw]{xcolor}
\usepackage{color}

\usepackage{graphicx}
\usepackage{subcaption}
\usepackage{tikz}
\usetikzlibrary{arrows, positioning, shapes.geometric}

\usepackage{csquotes}
\usepackage{lipsum}
\usepackage{framed}
\usepackage[most]{tcolorbox} 
\tcbset{
  myquote/.style={
    colback=gray!10,
    colframe=gray!50,
    width=\linewidth,
    sharp corners,
    boxrule=0.5pt,
    left=2em,
    right=2em,
    top=1em,
    bottom=1em,
    enhanced jigsaw,
  }
}

\usepackage{url}
\usepackage{hyperref}
\usepackage[capitalize,noabbrev]{cleveref}
\hypersetup{
    colorlinks=true,
    citecolor=blue,
    urlcolor=blue,
    linkcolor=red
}

\usepackage{booktabs}
\usepackage{multirow}
\usepackage{tabulary}
\usepackage{tabularx}

\usepackage{xspace}
\usepackage{enumitem}
\usepackage{float}
\usepackage[bottom]{footmisc}
\usepackage{authblk}
\usepackage{microtype}
\usepackage{multicol}

\usepackage{algorithm}
\usepackage{algorithmicx}
\usepackage[noend]{algpseudocode}

\usepackage{amsmath,amssymb,amsfonts,amsthm}
\usepackage{thmtools,thm-restate}
\usepackage{mathtools}
\usepackage{nicefrac}
\usepackage{bm}
\usepackage{bbm}
\usepackage{dsfont}
\usepackage{upgreek}

\newtheorem{theorem}{Theorem}[section]
\newtheorem{lemma}[theorem]{Lemma}

\newtheorem{fact}[theorem]{Fact}
\newtheorem{proposition}[theorem]{Proposition}
\theoremstyle{definition}
\newtheorem{definition}[theorem]{Definition}
\theoremstyle{remark}

\crefname{fact}{fact}{facts}
\crefname{claim}{claim}{claims}

\newcommand{\1}[1]{\mathbbm{1}\left[#1\right]}     %
\newcommand{\A}{\mathcal{A}}                      %
\newcommand{\Atruthful}{\A^{\textsf{truthful}}}

\newcommand{\Bern}{\mathsf{Bernoulli}}            %
\newcommand{\RSCE}{\SSCE}                         %
\newcommand{\Binomial}{\mathsf{Binomial}}         %

\newcommand{\calP}{\mathcal{P}}
\newcommand{\CM}{\mathsf{CM}}
\newcommand{\D}{\mathcal{D}}                      %
\newcommand{\ECE}{\mathsf{ECE}}
\newcommand{\eps}{\epsilon}
\newcommand{\err}{\mathsf{err}}
\newcommand{\Ex}[2]{\operatorname*{\mathbb{E}}_{#1}\left[#2\right]}
\newcommand{\F}{\mathcal{F}}

\newcommand{\OPT}{\mathsf{OPT}}
\newcommand{\poly}{\operatorname*{poly}}          %
\newcommand{\polylog}{\operatorname*{polylog}}
\newcommand{\pr}[2]{\Pr_{#1}\left[#2\right]}
\newcommand{\pstar}{p^{\star}}
\newcommand{\rmd}{\mathrm{d}}
\newcommand{\sgn}{\mathrm{sgn}}                   %
\newcommand{\smCE}{\mathsf{smCE}}
\newcommand{\SSCE}{\mathsf{SSCE}}
\newcommand{\stepCE}{\mathsf{stepCE}}
\newcommand{\stepCEsub}{\stepCE^{\textsf{sub}}}

\newcommand{\UCal}{\mathsf{UCal}}
\newcommand{\UCalsub}{\UCal^{\textsf{sub}}}
\newcommand{\Unif}{\mathsf{Unif}}
\newcommand{\Var}{\mathrm{Var}}
\newcommand{\VCal}{\mathsf{VCal}}
\newcommand{\VCalsub}{\VCal^{\textsf{sub}}}

\renewcommand{\epsilon}{\varepsilon}

\renewcommand{\tilde}{\widetilde}

\newcommand{\norm}[1]{\left\lVert#1\right\rVert}

\newcommand{\curlybrackets}[1]{\left\{#1\right\}}

\newcommand{\bset}[1]{\curlybrackets{#1}}

\newcommand{\abs}[1]{\left|#1\right|}

\DeclareMathOperator*{\Exp}{\mathbb{E}}

\newcommand{\EE}[1]{\Exp\left[#1\right]}

\newcommand{\EEs}[2]{\Exp_{#1}\left[#2\right]}

\newcommand{\EEsc}[3]{\Exp_{#1}\left[#2 \mid #3\right]}
\newcommand{\EEc}[2]{\Exp\left[#1\left|#2\right.\right]}

\newcommand{\ceil}[1]{\left\lceil#1\right\rceil}
\newcommand{\floor}[1]{\left\lfloor#1\right\rfloor}

\newcommand{\integers}{\mathbb{Z}}

\newcommand{\asseq}{\coloneqq}

\newcommand{\loss}{\ell}

\newcommand{\cA}{\mathcal{A}}

\newcommand{\cD}{\mathcal{D}}

\newcommand{\cF}{\mathcal{F}}

\newcommand{\cP}{\mathcal{P}}

\newcommand{\bp}{\boldsymbol{p}}

\newcommand{\by}{{\mathbf{y}}}

\title{{Truthfulness of Decision-Theoretic Calibration Measures}}

\author[1]{Mingda Qiao}
\author[2]{Eric Zhao}
\affil[1]{Massachusetts Institute of Technology}
\affil[2]{University of California, Berkeley}
\date{}  %

\begin{document}
\def\arxiv{1}

\maketitle

\begin{abstract}%
Calibration measures quantify how much a forecaster's predictions violates calibration, which requires that forecasts are unbiased conditioning on the forecasted probabilities. Two important desiderata for a calibration measure are its \emph{decision-theoretic implications}~\cite{KLST23} (i.e., downstream decision-makers that best-respond to the forecasts are always no-regret) and its \emph{truthfulness}~\cite{HQYZ24} (i.e., a forecaster approximately minimizes error by always reporting the true probabilities). Existing measures satisfy at most one of the properties, but not both.

We introduce a new calibration measure termed \emph{subsampled step calibration}, $\mathsf{StepCE}^{\textsf{sub}}$, that is both decision-theoretic and truthful. In particular, on any product distribution, $\mathsf{StepCE}^{\textsf{sub}}$ is truthful up to an $O(1)$ factor whereas prior decision-theoretic calibration measures suffer from an $e^{-\Omega(T)}$-$\Omega(\sqrt{T})$ truthfulness gap.  Moreover, in any smoothed setting where the conditional probability of each event is perturbed by a noise of magnitude $c > 0$, $\mathsf{StepCE}^{\textsf{sub}}$ is truthful up to an $O(\sqrt{\log(1/c)})$ factor, while prior decision-theoretic measures have an $e^{-\Omega(T)}$-$\Omega(T^{1/3})$ truthfulness gap. We also prove a general impossibility result for truthful decision-theoretic forecasting: any complete and decision-theoretic calibration measure must be discontinuous and non-truthful in the non-smoothed setting.
    
\end{abstract}

\section{Introduction}
\label{sec:intro}
Probabilistic forecasts play a central role in data-driven decision-making across broad application domains including finance, meteorology, and medicine~\cite{murphy1984probability,degroot1983comparison,WM68,jiang2012calibrating,kompa2021second,van2015calibration,berestycki2002asymptotics,crowson2016assessing}.
One of the greatest forms of utility provided by high-quality forecasting is that it enables downstream agents to confidently base their decision-making on the forecasts without any other knowledge of the future.
For example, a weather station's forecast of the probability of rain in the evening provides utility by informing individuals that may be debating whether to bring an umbrella to dinner.
A related and widely studied requirement of forecasting is \emph{calibration}~\cite{Brier50,Dawid82,FV98}, which requires that predicted probabilities align with long-run empirical frequencies of events.
Calibration requires, for example, that it rains 70\% of the days where the weather station forecasts a 70\% chance of rain.
Importantly, any downstream agent that bases their rational decision-making on perfectly calibrated forecasts will not incur any positive regret~\cite{FH21}.

While perfect calibration is generally unachievable, there exist a number of calibration measures, such as expected calibration error (ECE)~\cite{FV98}, smooth calibration error (SCE)~\cite{KF08}, and U-Calibration (UCal)~\cite{KLST23}, which formalize a notion of approximate calibration and quantify deviations from perfect calibration. 
U-Calibration is a calibration measure of particular significance for decision-making applications as it is defined as the worst regret that a rational agent can incur by blindly following a forecaster~\cite{KLST23}.
In contrast, most other calibration measures, such as smooth calibration error, are  not ``decision-theoretic'' in that they do not provide guarantees for the regret of downstream agents.

Because the Bayes optimal classifier is perfectly calibrated, it seems natural to view calibration as incentivizing a forecaster to produce predictions that are consistent with their beliefs.
This is not the case: forecasters that know the future are incentivized by most calibration measures to produce \emph{non-truthful} predictions~\cite{FH21,QV21,HQYZ24}.
For example, a forecaster may publicly forecast a 50\% chance of rain even if they know for certain that there is a 100\% chance of rain.
To this end, \cite{HQYZ24} proposed a set of desiderata for a calibration measure that includes, among common sense requirements like completeness (correct predictions have low error) and soundness (incorrect predictions have high error), a notion of \emph{truthfulness}: a calibration measure should not penalize forecasters that know the future for predicting the true probabilities of events.
\cite{HQYZ24} formalizes truthfulness by defining a calibration measure having a truthfulness gap as the asymptotic separation between the expected value of a calibration measure on a truthful forecaster and on a strategic forecaster, when both know exactly the probability with which future events will occur.
They also show that there exists a simple modification of smooth calibration error that is sound, complete, and truthful: compute smooth calibration error over randomly \emph{subsampled} timesteps rather than the entire time horizon.
However, the resulting calibration measure, like smooth calibration error, is not decision-theoretic in that it provides no meaningful guarantees for downstream agents.

In contrast, the U-Calibration measure is decision-theoretic but not truthful.
This means that minimizing the expected worst-case regret of a downstream agent requires intentionally misrepresenting one's knowledge of future events.
This is in stark contrast to the maximization of downstream utilities, which always incentivizes an aligned forecaster to predict consistently with their beliefs~\cite{Brier50}.
The literature leaves unresolved whether there exists {any} decision-theoretic calibration measure that both provides no-regret guarantees for downstream agents and satisfies the usual desiderata of truthfulness, soundness and completeness.

There are reasons to believe such a ``best of all worlds'' calibration measure is not possible. First, the technique used by \cite{HQYZ24} to derive a truthful calibration measure from the Smooth Calibration Error measure~\cite{KF08} does not appear to suffice when applied to the U-Calibration measure. Second, the best responses of downstream agents are typically discontinuous in the forecasts they are given, and discontinuities are intimately connected with non-truthfulness~\cite{HQYZ24}.
This raises the questions:
\begin{quote}
    \emph{Is there a calibration measure that both provides decision-theoretic guarantees and incentivizes honest forecasting?}
    
    \emph{Is there a fundamental conflict between minimizing the regret of downstream agents and being truthful?}
\end{quote}

In this work, we show that the answer to both questions can be ``Yes'': there is a fundamental conflict between truthfulness and decision-theoretic guarantees---but not with smoothed analysis, under which we can design a calibration measure that is the best of all worlds.

\subsection{Overview of Results}
\paragraph{U-Calibration is far from truthful.}
We identify two sources of non-truthfulness in U-Calibration that the subsampling technique of \cite{HQYZ24} does not remedy. The first source arises from the discontinuity of U-Calibration error. The second source is an incentive for forecasters to hedge their predictions, i.e., exaggerate their uncertainty, and materially contributes to the non-truthfulness of U-Calibration even under smoothed analysis.

As a result, even with subsampling and smoothed analysis, U-Calibration exhibits a \emph{truthfulness gap}. We say a calibration measure has an $\alpha$-$\beta$ truthfulness gap, which we define formally in \eqref{eq:gap}, if it gives a truthful forecaster an error of $\ge \beta$ but the optimal strategic forecaster's error is below $\alpha$.

\newenvironment{customthm}[1][Theorem]{%
  \par\medskip %
  \noindent %
  \textbf{#1.} %
  \itshape %
}{%
  \par\medskip %
}

\begin{customthm}[Propositions \ref{prop:subsamptruth}, \ref{prop:simplesubsamptruth}, and \ref{prop:subsampsmooth}, Informally]
Both the U-Calibration measure $\UCal$ and its subsampled variant $\UCalsub$ suffer from an $O(\sqrt T)$-$\Omega(T)$ truthfulness gap due to the discontinuity of $\UCal$, and  an $e^{-\Omega(T)}$-$\Omega(\poly(T))$ truthfulness gap due to the hedging incentives of $\UCal$.
\end{customthm}
\noindent
We also prove a general impossibility result that suggests the non-truthfulness of U-Calibration is, to a degree, unavoidable. We later show that this result can be softened with smoothed analysis.
\begin{customthm}[Proposition \ref{prop:completedecisionlowerbound}, Informally]
    For any calibration measure, at least one of the following must be true:\!
    \begin{itemize}
        \item It is not complete: consistently forecasting a 50\% chance of heads given a sequence of $T$ fair coins does not yield an $O(\sqrt T)$ error.
        \item It is not decision-theoretic: it does not always upper bound the regret of downstream agents.
        \item It is not truthful: there is an $O(\sqrt {T})$-$\Omega(T)$ truthfulness gap.
    \end{itemize}
\end{customthm}

\paragraph{Step calibration error.}
We introduce \emph{step calibration}: a sound, complete, and decision-theoretic calibration measure that provides no-regret guarantees for all downstream agents.
Given a sequence of events $x_1, .\hfil.\hfil., x_T \in \{0, 1\}$ and predictions $p_1, .\hfil.\hfil., p_T \in [0, 1]$, the step calibration error is defined as
\[
\stepCE(x, p) \coloneqq \sup_{\alpha \in [0, 1]} \abs{\sum_{t=1}^T (x_t - p_t) \cdot \1{p_t \leq \alpha}}.
\]
Step calibration is equivalent, up to a constant factor, to a variant of V-Calibration that uses a slightly different baseline to disincentivize hedging behavior by penalizing excessively conservative probabilistic forecasts (\Cref{fact:equivfactor}).
In addition to step calibration being a complete and sound calibration measure, we also demonstrate an algorithm that achieves an $\tilde O(\sqrt T)$ step calibration error for the adversarial prediction setting.

\begin{customthm}[\Cref{prop:stepCE-complete-sound-decision-theoretic} and Theorem~\ref{thm:alg}, Informally]
    The step calibration error is sound, complete, and decision-theoretic. Moreover, there is a forecasting algorithm that guarantees an expected step calibration error of $O(\sqrt{T \log T})$, even if the events are adversarially and adaptively chosen.
\end{customthm}

\paragraph{Truthfulness under smoothed analysis.}
We show that---under smoothed analysis---the impossibility of simultaneously providing decision-theoretic guarantees and truthfulness largely disappears.
At a high-level, smoothing negates an adversary's ability to exploit the inherently discontinuous nature of a downstream agent's decision-making.
Importantly, we show that the ``subsampled'' variant of step calibration,
\[
\stepCEsub(x,p) \coloneqq \EEs{S \sim \Unif(2^{[T]})}{\sup_{\alpha \in [0, 1]} \abs{\sum_{t=1}^T (x_t - p_t) \cdot \1{p_t \leq \alpha \land t \in S}}},
\]
is truthful under smoothed analysis.

We say that a calibration measure is $(\alpha, \beta)$-truthful gap if, given any prior distribution over the sequence of events, the error incurred by the truthful forecaster is upper bounded by the optimal strategic forecaster's error, up to a factor of $\alpha$ and and additive term of $\beta$; see \Cref{eq:gap} for a formal definition.

\begin{customthm}[Theorem~\ref{thm:strongupperbound}, Informally]
    Subsampled step calibration error is $(O(\sqrt{\log(1/c)}), \polylog(T/c))$-truthful when each conditional probability is drawn from a distribution with density $\le 1/c$.
\end{customthm}

This $O(\sqrt{\log(1/c)})$ factor is tight for $\stepCEsub$. For non-smoothed product distributions, we can obtain a stronger truthfulness result showing that $\stepCEsub$ is $(O(1), 0)$-truthful (\Cref{prop:stepCEsub-truthful-product}).
$\stepCEsub$ also retains the desiderata of $\stepCE$, and is decision-theoretic, sound, and complete, and admits an $\tilde O(\sqrt T)$ algorithm in the adversarial setting.
This is because, as we show in \Cref{lemma:stepCE-vs-stepCEsub}, 
$\frac{1}{2}\stepCE(x, p) \le \stepCEsub(x, p) \le \frac{1}{2}\stepCE(x, p) + O(\sqrt{T})$.

\subsection{Related Work}
Most closely related to our work are the previous studies of sequential binary calibration with respect to various calibration measures. The seminal work of~\cite{FV98} showed that asymptotic calibration can be achieved even for adversarially chosen events. Implicit in their paper is a sublinear rate of $O(T^{2/3})$ on the ECE incurred by the forecaster; a more detailed proof was given by~\cite{Hart22}. On the lower bound side, an $\Omega(\sqrt{T})$ bound is trivial and the first non-trivial lower bound of $\Omega(T^{0.528})$ was shown by~\cite{QV21}. A recent breakthrough of~\cite{DDFGKO24} improved the upper bound to $O(T^{2/3-\eps})$ for some constant $\eps > 0$, and gave the best known lower bound of $\Omega(T^{0.54389})$.

The analogous question has been studied for other calibration measures, including smooth calibration~\cite{KF08,QZ24}, U-Calibration~\cite{KLST23}, distance from calibration~\cite{QZ24,ACRS24}, calibration decision loss~\cite{HW24}, and subsampled smooth calibration~\cite{HQYZ24}. These calibration measures relax the ECE in different ways, so that the forecaster can achieve a faster rate of $\tilde O(\sqrt{T})$, circumventing the super-$\sqrt{T}$ lower bounds for the ECE.

The systematic study of calibration measures was initiated by~\cite{BGHN23}, who focused on the offline setting and proposed the \emph{distance from calibration} as a ground truth. Their work identified calibration measures that are continuous and consistent (i.e., being polynomially related to the distance from calibration). Subsequent work studied calibration measures that satisfy other natural axioms, including being decision-theoretic~\cite{KLST23,NRRX23,RS24,HW24} and being truthful~\cite{HQYZ24}. Remarkably, the distance from calibration, while being a natural measure, is neither decision-theoretic nor truthful in the sequential setting.

The truthfulness of calibration measures is also intimately connected to the study of multicalibration \cite{HKRR18}, which designs calibration measures that incentivize forecasters to ``truthfully'' predict the true probabilities of each feature by evaluating calibration error over many feature subsets---a practice similar to \cite{HQYZ24}'s use of subsampling to enforce truthfulness and algorithmically linked to Blackwell approachability~\cite{blackwell_analog_1956,Foster99,Hart22,haghtalab2024unifying}.
Truthfulness can also be viewed as enforcing an online, multi-timestep notion of outcome indistinguishability \cite{dwork2021outcome}.

Smoothed analysis was introduced by Spielman and Teng~\cite{ST04,spielman2009smoothed} for analyzing the ``typical'' runtime of the simplex method. \cite{KST09} introduced a smoothed analysis model for supervised learning, in which the data distribution is a product distribution over $\{0, 1\}^n$ with marginal probabilities randomly perturbed. This circumvents hard instances that are specific for the uniform distribution but easily learnable under the perturbed distribution. Subsequent work studied tensor decomposition~\cite{BCMV14} and decision tree learning~\cite{BDM20,BLQT21} in similar setups.

More closely related to our work are the smoothed analysis for online learning introduced by~\cite{RST11}. A recent series of work extends this setting to adaptive adversaries, showing that online learning against a smoothed adversary is not much harder than learning in the offline (batch) setup~\cite{HRS20,HRS24,BDGR22,HHSY22,BS22,BP23,BSR23,BST23}. In these models, the smoothed analysis limits the adversary's ability of concentrating the probability mass at a ``hard region'' in the instance space. As a result, the learner may circumvent the canonical hard instance of threshold functions, which is easily learnable in the offline setting, and cannot be learned in an online setting without smoothing. Our work applies the smoothed analysis to avoid the large truthfulness gap in the non-smoothed setting following the same intuition.

\section{Preliminaries}
\label{sec:prelims}
\subsection{Sequential Prediction and Calibration}
\paragraph{Sequential prediction.}
In the basic (non-smoothed) prediction setup, a sequence of events $x \in \{0, 1\}^T$ is sampled from a distribution $\cD$.
At each time step $t \in [T]$, the forecaster makes a prediction $p_t \in [0, 1]$, after which $x_t$ is revealed.
Formally, a deterministic forecaster is a function $\cA: \bigcup_{t=1}^{T}\{0,1\}^{t-1} \to [0, 1]$, where $\cA(b_1, b_2, \ldots, b_{t-1})$ specifies the forecaster's prediction at step $t$ upon observing $x_{1:(t-1)} = b_{1:(t-1)}$.
We will write $(x, p) \sim (\D, \A)$ to denote sampling events $x \in \{0, 1\}^T$ and predictions $p \in [0, 1]^T$ from the joint distribution naturally induced by distribution $\D$ and forecaster $\A$, i.e., by sampling $x \sim \cD$ and setting $p_t = \cA(x_1, x_2, \ldots, x_{t-1})$ for each $t \in [T]$.
We could have defined the forecaster to be randomized or a function of both the outcomes $x_{1:(t-1)}$ and its own predictions $p_{1:(t-1)}$, but restricting to deterministic functions of the outcomes $x_{1:(t-1)}$ comes without loss of generality.

\paragraph{The smoothed setting.} The prediction setting above can be equivalently viewed as the nature specifying the conditional probability of $x_t = 1$ given $x_1, x_2, \ldots, x_{t-1}$. We will also consider a \emph{smoothed} setting, where each conditional probability is perturbed by a noise of magnitude $c > 0$. Formally, the nature specifies a mapping $\cP: \bigcup_{t=1}^{T}\{0, 1\}^{t-1} \mapsto \Delta_c$, where $\Delta_c$ is the family of distributions over $[0, 1]$ with densities bounded by $1/c$ everywhere. At each step $t$, the nature realizes $\pstar_t \sim \cP(x_1, x_2, \ldots, x_{t-1})$ and credibly reveals the value of $\pstar_t$ to the forecaster.\footnote{In practice, this corresponds to the forecaster acquiring certain side information about $x_t$, thus changing their belief of $\pstar_t = \pr{}{x_t = 1 \mid \text{observations}}$. The smoothness assumption would then correspond to assumptions on the side information, which might ensure that the distribution of $\pstar_t$ is not too spiky.} The forecaster predicts $p_t$, and the event $x_t$ is sampled from $\Bern(\pstar_t)$ and revealed. Formally, the forecaster's prediction $p_t$ is a function of both $x_{1:(t-1)}$ and $\pstar_t$, i.e., $\A: \bigcup_{t=1}^{T}(\{0, 1\}^{t-1} \times [0, 1]) \to [0, 1]$.

Note that the non-smoothed setting---in which the nature specifies a fixed distribution $\cD$ over $\{0, 1\}^T$---can be viewed as a ``$0$-smoothed'' setting where each $\cP(b_1, b_2, \ldots, b_{t-1})$ is the degenerate distribution at value $\pr{x \sim \cD}{x_t = 1 \mid x_{1:(t-1)} = b_{1:(t-1)}}$. Furthermore, as in the non-smoothed setting, the nature $\cP$ and the forecaster $\A$ naturally induce a joint distribution over the triple $(x, \pstar, p)$, denoted by sampling $(x, \pstar, p) \sim (\cP, \A)$ (or a subset thereof) in the rest of the paper.

\paragraph{Calibration measures.}
A calibration measure $\CM_T: \{0, 1\}^T \times [0, 1]^T \to [0, T]$ quantifies the quality of a forecaster's prediction. We omit the subscript $T$ when it is clear from context. The expected penalty incurred by forecaster $\A$ on distribution $\D$ is defined as $\err_{\CM}(\D, \A) \coloneq \Ex{(x, p)\sim (\D, \A)}{\CM(x, p)}$. For the smoothed setting, we analogously define $\err_{\CM}(\cP, \A) \coloneqq \Ex{(x, p) \sim (\cP, \A)}{\CM(x, p)}$.

One would naturally expect a calibration measure to be both \emph{complete} (accurate predictions lead to a small penalty) and \emph{sound} (inaccurate predictions receive a large penalty). We adopt a variant of the definition in~\cite{HQYZ24}. In the following, $\vec{1}_T$ denotes the $T$-dimensional all-$1$ vector.
\begin{definition}[Completeness and soundness~\cite{HQYZ24}]\label{def:complete-and-sound}
    A calibration measure $\CM$ is complete if: (1) For any $x \in \{0, 1\}^T$, predicting  the events $x$ gives $\CM_T(x, x) = 0$; (2) For any $\alpha \in [0, 1]$, predicting the constant probability $\alpha$ gives
    $\Ex{x_1, \ldots, x_T \sim \Bern(\alpha)}{\CM_T(x, \alpha\cdot\vec{1}_T)} = o_{\alpha}(T)$; (3) For any $p \in [0,1]^T$ that is perfectly calibrated with respect to $x \in \{0, 1\}^T$, $\CM_T(x, p) = o(T)$.
    The calibration measure is sound if: (1) For any $x \in \{0, 1\}^T$, $\CM_T(x, \vec{1}_T - x) = \Omega(T)$; (2) For any $\alpha, \beta \in [0, 1]$ such that $\alpha \ne \beta$, $\Ex{x_1, \ldots, x_T \sim \Bern(\alpha)}{\CM_T(x, \beta\cdot\vec{1}_T)} = \Omega_{\alpha, \beta}(T)$. 
    Here, $o_{\alpha}(\cdot)$ and $\Omega_{\alpha, \beta}(\cdot)$ hide constant factors that depend on the subscripted parameters.
\end{definition}
We strengthened the definition of completeness in~\cite{HQYZ24} by adding a third constraint---perfectly calibrated predictions should receive a low (sublinear) penalty. 
To the best of our knowledge, all calibration measures satisfy this condition, with most satisfying this condition with $\CM(x, p) = 0$ while the SSCE introduced by~\cite{HQYZ24} satisfies this with $\CM(x, p) = O(\sqrt{T})$.

\subsection{Decision-Theoretic Calibration}
Consider a decision-making setting where an agent acts on the basis of the forecaster's predictions.
Formally, consider a repeated game where, at each round $t$, an agent chooses an action $a_t \in \cA$ informed by the forecaster's prediction $p_t$. The agent's utility at time $t$ is a function $u: \cA \times \{0, 1\} \to [-1, 1]$ of its action $a_t$ and the event $x_t$.
In this setup, the agent assumes that $p_t$ is an accurate forecast of the probability that $x_t$ occurs and thus selects $a_t$ to maximize $\mathbb{E}_{x \sim \Bern(p_t)}[u(a_t, x)]$.
One benefit of calibrated forecasts is that agents can make decisions according to the forecasts and be no-regret \cite{FH21}.
That is, the agent's expected cumulative utility $\sum_t u(a_t, x_t)$ when given calibrated forecasts $p_{1:T}$ will never be worse than when given a \emph{base rate forecaster} which predicts $p_t = \frac{1}{T} \sum_{t=1}^T x_t$ for all $t$.

Providing forecasts with no-regret guarantees for agents with any utility is equivalent to providing forecasts that satisfy the no-regret property with respect to any (piecewise linear) proper scoring rule \cite{KLST23}.
A (bounded) scoring rule is a function $S:\{0, 1\}\times[0,1] \to [-1, 1]$, where $S(x, p)$ denotes the loss incurred by the forecaster, when it predicts value $p \in [0, 1]$ on an outcome that turns out to be $x \in \{0, 1\}$. A scoring rule $S$ is \emph{proper} if, for any $\alpha \in [0, 1]$, the function $p \mapsto \Ex{x \sim \Bern(\alpha)}{S(x, p)}$ is minimized at $p = \alpha$, i.e., predicting the true probability minimizes the expected loss. The U-Calibration error of~\cite{KLST23} quantifies the mis-calibration using the worst possible external regret:
\begin{equation}\label{eq:U-calibration}
    \UCal(x, p) \coloneqq \sup_{S}\left[\sum_{t=1}^{T}S(x_t, p_t) - \inf_{\beta \in [0, 1]}\sum_{t=1}^{T}S(x_t, \beta)\right],
\end{equation}
where the supremum is taken over all proper scoring rules $S$.

More generally, we refer to any calibration measure $\CM$ as being \emph{decision-theoretic} if, for all events $x$ and predictions $p$, the calibration measure is lower bounded by U-Calibration up to a universal constant factor: $\CM(x, p) \geq \Omega(1) \cdot \UCal(x, p)$. A decision-theoretic calibration measure upper bounds the external regret of any agent that acts on the forecaster's predictions.

\paragraph{V-Calibration.}
We will work with a calibration measure known as V-Calibration \cite{KLST23} that is more technically convenient.
V-Calibration is a modification of U-Calibration obtained from limiting the supremum in its definition (\Cref{eq:U-calibration}) to a narrow class of scoring rules of the form
\begin{equation}
    \label{eq:vcaleq}
    S_\alpha(x, p) \coloneq (\alpha - x)\cdot \mathrm{sgn}(p - \alpha)
\end{equation}
for any $\alpha \in [0, 1]$.
Formally, the V-Calibration error is defined as
\begin{equation}
\label{eq:V-calibration}
    \VCal(x, p) \coloneqq \sup_{\alpha, \beta \in [0, 1]}\left[\sum_{t=1}^{T}S_\alpha(x_t, p_t) - \sum_{t=1}^{T}S_\alpha(x_t, \beta)\right].
\end{equation}
Despite its simpler form, V-Calibration error is equivalent to U-Calibration error up to constant factor:
\begin{lemma}[Theorem~8 of~\cite{KLST23}]\label{lemma:UCal-vs-VCal}
For any $x \in \{0, 1\}^T$ and $p \in [0, 1]^T$, it holds that $$\frac{1}{2}\UCal(x, p) \le \VCal(x, p) \le \UCal(x, p).$$
\end{lemma}
We can rewrite the V-Calibration measure into an alternative form without the scoring rules. We prove the following proposition in \Cref{sec:altform-proof}.
\begin{restatable}{proposition}{altform}
\label{prop:altform}
The V-Calibration error takes the alternative form
\[
    \VCal(x, p) = 2 \cdot \sup_{\alpha \in [0, 1]}\max\{X_-^{(\alpha)} - \alpha N_-^{(\alpha)}, \alpha N_+^{(\alpha)} - X_+^{(\alpha)}\},
\]
where $N_{-}^{(\alpha)} \coloneq \sum_{t=1}^{T}\1{p_t < \alpha}$, $N_{+}^{(\alpha)} \coloneq \sum_{t=1}^{T}\1{p_t > \alpha}$, $X_{-}^{(\alpha)}\coloneq \sum_{t=1}^{T}x_t\cdot\1{p_t < \alpha}$,  and $X_{+}^{(\alpha)}\coloneq \sum_{t=1}^{T}x_t\cdot\1{p_t > \alpha}$.
\end{restatable}

\subsection{Truthfulness}
Calibration measures are often seen as measuring how close a forecaster's predictions are to the true probabilities that events occur.
However, even if one knows the exact probability that an event $x_t$ will occur, calibration does not necessarily incentivize one to predict truthfully~\cite{HQYZ24}.
Consider the truthful forecaster for the non-smoothed setting specified by $\cD \in \Delta(\{0, 1\}^T)$,
\[\Atruthful(\D)(b_1, b_2, \ldots, b_{t-1}) \coloneq \pr{x \sim \D}{x_t = 1\ |\ x_{1:(t-1)} = b_{1:(t-1)}},\]
which can be argued to be the only forecaster that makes the ``right'' predictions on distribution $\D$.
Given a reasonable calibration measure $\CM$, one might expect the error of the truthful forecaster to be close to the optimal error $\OPT_{\CM}(\D) \coloneq \inf_{\A}\err_{\CM}(\D, \A)$, where $\A$ ranges over all deterministic forecasters.
This property is known as \emph{truthfulness} \cite{HQYZ24}, where we say that a calibration measure $\CM$ is $(\alpha, \beta)$-truthful if, for every $\D \in \Delta(\{0, 1\}^T)$,
\begin{equation}\label{eq:gap}\err_{\CM}(\D, \Atruthful(\D)) \le \alpha \cdot \OPT_\CM(\D) + \beta.\end{equation}
Conversely, $\CM$ is said to have an $\alpha$-$\beta$ truthfulness gap if, for some distribution $\D$, $\OPT_\CM(\D) \le \alpha$ and $\err_{\CM}(\D, \Atruthful(\D))\ge \beta$.

In smoothed settings where the conditional probabilities are sampled according to $\cP$ and revealed to the forecaster, the truthful forecaster, $\Atruthful$, simply maps $(b_{1:(t-1)}, \pstar_t)$ to $\pstar_t$ for any $t \in [T]$ and $(b_{1:(t-1)}, \pstar_t) \in \{0, 1\}^{t-1}\times[0,1]$. We define $\OPT_{\CM}(\cP) \coloneqq \inf_{\A}\err_{\CM}(\cP, \A)$, and a $\CM$ as being $(\alpha, \beta)$-truthful if $\err_{\CM}(\cP, \Atruthful) \le \alpha \cdot \OPT_\CM(\cP) + \beta$. We similarly define the $\alpha$-$\beta$ truthfulness gap for smoothed settings.

U-Calibration is known to not be a truthful calibration measure~\cite{HQYZ24}. This might be counterintuitive since, by definition, truthful forecasting minimizes the expected penalty for each \emph{fixed} proper scoring rule. However, after taking a supremum over all proper scoring rules, the truthful forecaster ceases to be optimal for the resulting measure.
\begin{proposition}[Proposition A.3 of \cite{HQYZ24}]\label{prop:HQYZ24-UCal}
    The U-Calibration error has an $O(1)$-$\Omega(\sqrt T)$ truthfulness gap.
\end{proposition}

One example of a truthful calibration measure is the \emph{Subsampled Smooth Calibration Error (SSCE)} introduced by \cite{HQYZ24}.
SSCE is a variant of the \emph{smooth calibration error} calibration measure introduced by~\cite{KF08}: $\smCE(x, p) \coloneqq \sup_{f \in \mathcal{F}} \sum_{t=1}^T f(p_t) (x_t - p_t)$, where $\mathcal{F}$ is the family of $1$-Lipschitz functions from $[0, 1]$ to $[-1, 1]$.
SSCE is defined by subsampling a subset of the time horizon, and evaluating the Smooth Calibration Error on it.
Formally, letting $\Unif(S)$ denote the uniform distribution over a finite set $S$ and $x|_S$ denote the $|S|$-dimensional vector formed by the entries of $x$ indexed by $S$:
     \begin{equation}
     \label{eq:ssce}
         \RSCE(x, p) \coloneq \Ex{S \sim \Unif(2^{[T]})}{\smCE(x|_S, p|_S)}
         =   \Ex{y \sim \Unif(\{0, 1\}^T)}{\sup_{f \in \F}\sum_{t=1}^{T}y_t \cdot f(p_t)\cdot(x_t - p_t)}.
     \end{equation}
In light of \Cref{prop:completedecisionlowerbound}, since SSCE is complete and truthful, it cannot be decision-theoretic.

\section{Technical Overview}
\label{sec:overview}
\subsection{Non-truthfulness of U-Calibration and Its Variants}
The non-truthfulness of the U-Calibration error (\Cref{prop:HQYZ24-UCal}) comes from the incentive for a dishonest forecaster to ``patch up'' their previous mis-calibration. Specifically, \cite{HQYZ24} considered a length-$T$ sequence that consists of $T/2$ independent random bits followed by $T/2$ ones. In this case, the truthful forecaster predicts $1/2$ in the first half, and typically incurs an $\Omega(\sqrt{T})$ bias on those bits. This further translates into an $\Omega(\sqrt{T})$ U-Calibration error. On the other hand, a strategic forecaster may deliberately predict a biased value of $5/8$ on the first half, and continue predicting $5/8$ on the second half until the bias is close to $0$.\footnote{This happens with high probability, as the mean of the first half concentrates around $1/2 < 5/8$, while the mean of the entire sequence concentrates around $3/4 > 5/8$.} The resulting U-Calibration error can then be bounded by $O(1)$ in expectation.

This $O(1)$-$\Omega(\sqrt{T})$ truthfulness gap, however, would vanish once we apply the \emph{subsampling} technique of~\cite{HQYZ24}. The subsampled version, $\UCalsub$, evaluates the U-Calibration error on a random subset of the horizon. This introduces a $\Theta(\sqrt{T})$ error in the resulting penalty, so the strategic forecaster at best outperforms the truthful one by a constant factor. One might naturally wonder whether $\UCalsub$ is truthful in general. Unfortunately, as we outline below, there exist two additional ``failure modes'' of the U-Calibration error that cannot be remedied by subsampling alone.

\paragraph{Example 1: Non-truthfulness due to discontinuity.} We start by noting that the U-Calibration error, $\UCal(x, p)$, is not continuous in $p$. Suppose that, for some small $\eps > 0$, we have
\[
    (x_t, p_t) = \begin{cases}
        (1, 1/2 - \eps), & t \le T / 2,\\
        (0, 1/2 + \eps), & t > T / 2.
    \end{cases}
\]
Note that $p$ is almost calibrated: $\tilde p = (1/2, 1/2, \dots, 1/2)$ is entry-wise close to $p$, and perfectly calibrated with respect to $x$, which implies $\UCal(x, \tilde p) = 0$. However, $\UCal(x, p)$ is much larger: Consider the equivalent formulation of the V-Calibration error in \Cref{prop:altform} and take $\alpha = 1/2$. There are $N_{-} = T/2$ steps on which $p_t < \alpha$, and the outcomes on those steps sum up to $X_{-} = T/2$. By \Cref{lemma:UCal-vs-VCal}, we have $\UCal(x, p) \ge \VCal(x, p) \ge 2(X_{-} - \alpha N_{-}) = \Omega(T)$. Taking $\eps \to 0^{+}$ gives triples $(x, p, \tilde p)$ such that $\|p - \tilde p\|_{\infty} \to 0$ but $\UCal(x, p) = \Omega(T)$ and $\UCal(x, \tilde p) = 0$ are far away.

This implies that any complete and decision-theoretic calibration measure $\CM$ must be discontinuous. For the triple $(x, p, \tilde p)$ constructed as above, completeness (\Cref{def:complete-and-sound}) implies $\CM(x, \tilde p) = o(T)$ while being decision-theoretic requires $\CM(x, p) \ge \Omega(1)\cdot\UCal(x, p) = \Omega(T)$. Taking $\eps \to 0^+$ shows that $\CM$ is discontinuous.

The example above does not immediately give the $O(\sqrt{T})$-$\Omega(T)$ truthfulness gap in \Cref{prop:subsamptruth}. Towards showing that the truthful forecaster incurs an $\Omega(T)$ U-Calibration error, we need to design a sequence of true probabilities $\pstar_1, \pstar_2, \ldots, \pstar_T \approx 1/2$ and a threshold $\alpha \in [0, 1]$, such that $x_t = 1$ whenever $\pstar_t < \alpha$ and $x_t = 0$ whenever $\pstar_t > \alpha$. This is \emph{very} unlikely to happen if each $x_t$ is independently sampled from $\Bern(\pstar_t)$.

However, when nature picks each $\pstar_t$ based on the previous outcomes $x_{1:(t-1)}$, the hoped-for property \emph{can} be guaranteed via a simple binary search. At step $t = 1$, the nature starts with an interval $[l_1, r_1] = [1/2 - \eps, 1/2 + \eps]$ and picks $\pstar_1 = (l_1 + r_1) / 2$ as the middle point. After realizing $x_1 \sim \Bern(\pstar_1)$, if $x_1 = 1$, the nature updates $[l_2, r_2] \gets [\pstar_1, r_1]$; otherwise, $[l_2, r_2] \gets [l_1, \pstar_1]$. If the nature repeats this $T$ steps, we can verify that, for $\alpha \coloneqq (l_{T+1} + r_{T+1}) / 2$ and every $t \in [T]$: (1) $\pstar_t < \alpha$ implies $x_t = 1$; (2) $\pstar_t > \alpha$ implies $x_t = 0$. Then, a similar argument shows that truthful forecasting leads to $\UCal(x, \pstar) = \Omega(T)$. Furthermore, this argument is robust to subsampling, i.e., we also have $\UCalsub(x, \pstar) = \Omega(T)$. In contrast, a strategic forecaster might choose to predict $p_t = 1/2$ at every step. Since each $\pstar_t$ is in $[1/2 - \eps, 1/2 + \eps]$, as long as $\eps = O(1/\sqrt{T})$, $\sum_{t=1}^{T}x_t$ would concentrate around $T/2 \pm O(\eps T + \sqrt{T}) = T/2 \pm O(\sqrt{T})$. This shows that $\OPT_{\UCal}, \OPT_{\UCalsub} = O(\sqrt{T})$, and thus the $O(\sqrt{T})$-$\Omega(T)$ truthfulness gap.

We also further generalize this example to show that no calibration measure can be complete, decision-theoretic and non-trivially truthful simultaneously.

\paragraph{Example 2: Non-truthfulness due to hedging.} In the previous example, it is crucial that nature specifies the conditional expectation of each bit ($\pstar_t$) adaptively and with arbitrary precision. Nevertheless, we show that $\UCal$ can still have a large truthfulness gap, even if the events are drawn from a product distribution, the marginal probabilities of which are, in turn, drawn from smooth distributions. At a high level, this is because the regret minimization for downstream agents incentivizes \emph{hedging behaviors}, where forecasters benefit from exaggerating the uncertainty of future events.

We start with a simple, non-smoothed setting: For each $t \in [T]$, we set $\pstar_t = 1/5$ if $t \le T/2$, and $\pstar_t = 4/5$ if $t > T/2$. Each $x_t$ is independently sampled from $\Bern(\pstar_t)$. As in Example~1, truthful prediction typically leads to an $\Omega(\sqrt{T})$ bias at predicted values $1/5$ and $4/5$ each, and results in an $\Omega(\sqrt{T})$ U-Calibration error. In contrast, the forecaster can significantly lower its penalty by predicting $p_t = 2/5$ at $t \le T/2$ and $p_t = 3/5$ at $t > T/2$ instead. In light of \Cref{lemma:UCal-vs-VCal} and \Cref{prop:altform}, it suffices to upper bound the value of $\max\{X_- - \alpha N_-, \alpha N_+ - X_+\}$ for different values of $\alpha$. The worst cases are when $\alpha \to (2/5)^+$ and $\alpha \to (3/5)^-$. In the former case, we have $N_- = T/2$ while $X_-$ concentrates around $(T/2)\cdot (1/5) = T/10$, and is typically (except with probability $e^{-\Omega(T)}$) smaller than $\alpha N_- = T/5$. Similarly, when $\alpha \approx 3/5$, we have $N_+ = T/2$ and $X_+$ concentrates around $(T/2) \cdot (4/5) = 2T/5$, and is extremely unlikely to be below $\alpha N_+ = 3T/10$. This establishes the $e^{-\Omega(T)}$-$\Omega(\sqrt{T})$ truthfulness gap for $\UCal$, and the same construction works for the subsampled version $\UCalsub$ as well.

In the $c$-smoothed setting for some small constant $c$, instead of setting each $\pstar_t$ to $1/5$ or $4/5$, we draw each $\pstar_t$ independently and uniformly from either $[1/5-c, 1/5+c]$ or $[4/5-c,4/5+c]$. Here, a complication is that we cannot ``catch'' a high U-Calibration error by na\"ively setting $\alpha = 1/5 + c$ (and applying \Cref{lemma:UCal-vs-VCal} and \Cref{prop:altform}). This is because we would end up with $N_- = T/2$ and $X_-$ concentrating around $(T/2)\cdot(1/5) = T/10$, which is \emph{lower} than $\alpha N_- = (1/5 + c)\cdot(T/2) = T/10 + \Omega(T)$. Instead, we pick $\alpha = 1/5 - (1 - \gamma)c$ for some $\gamma > 0$ to be chosen. Since a uniform sample from $[1/5-c, 1/5+c]$ falls into $[1/5-c, \alpha]$ with probability $(\gamma c) / 2c = \gamma / 2$, we expect $N_- = \Theta(\gamma T)$. Furthermore, conditioning on that $\pstar_t \in [1/5-c, \alpha]$, $x_t$ has an expectation of $\frac{(1/5 - c) + \alpha}{2} = \alpha - \gamma c/2$. This shows that $X_-$ concentrates around $(\alpha - \gamma c/2)\cdot N_- = \alpha N_- - O(\gamma N_-)$ up to a typical deviation of $\sqrt{N_-} = \Theta(\sqrt{\gamma T})$. If we set $\gamma = \Theta(T^{-1/3})$ appropriately, the deviation would be $\Theta(\sqrt{\gamma T}) = \Theta(T^{1/3})$, dominating the $O(\gamma N_-) = O(\gamma^2 T) = O(T^{1/3})$ bias. This leads to an $\Omega(T^{1/3})$ penalty in both $\UCal$ and $\UCalsub$. In contrast, the dishonest forecasts (that take value either $2/5$ or $3/5$)  incur an $e^{-\Omega(T)}$ penalty under either $\UCal$ or $\UCalsub$. This shows that neither $\UCal$ and $\UCalsub$ can be truthful with sub-$\poly(T)$ parameters, even in the $\Omega(1)$-smoothed setting.

\subsection{Truthfulness of Subsampled Step Calibration}
Showing that $\stepCEsub$ is truthful involves two steps: lower bounding the optimal penalty that can be achieved by a (possibly dishonest) forecaster, and upper bounding the penalty incurred by the truthful forecaster. The first part follows from a result of~\cite{HQYZ24}: regardless of the forecasting algorithm $\A$, it holds that
\begin{equation}\label{eq:HQYZ24-lower}
    \Ex{(x, p) \sim (\D, \A), y \sim \Unif(\{0, 1\}^T)}{\left|\sum_{t=1}^{T}y_t \cdot (x_t - p_t)\right|}
\ge \Omega(\Ex{}{\gamma(\Var_T})),
\end{equation}
where $\gamma(x) = \begin{cases}x, & x \le 1,\\ \sqrt{x}, & x > 1\end{cases}$ and the random variable $\Var_T$ is defined as $\Var_T \coloneqq \sum_{t=1}^{T}\pstar_t(1-\pstar_t)$. Since the left-hand side of \Cref{eq:HQYZ24-lower} is a lower bound on $\stepCEsub(x, p)$, we also have $\OPT_{\stepCEsub}(\D) \ge \Omega(\Ex{}{\gamma(\Var_T)})$.

For simplicity, we assume in this section that $\Var_T$ is always $\Omega(T)$ (e.g., when every conditional probability is bounded away from $0$ and $1$), and focus on upper bounding $\err_{\stepCEsub}(\D, \Atruthful)$ by $\tilde O(\sqrt{T})$. The general case that $\Var_T$ can be much lower than $T$ can be handled via a doubling trick similar to the technique of~\cite{HQYZ24}. Also, we will upper bound $\stepCE$ instead of $\stepCEsub$, as all the analyses would naturally generalize to the subsampled version.

\paragraph{Warm-up \#1: Product distributions.} We start with the special case that $\D$ is a product distribution, i.e., $\pstar \in [0, 1]^T$ is fixed, and each $x_t$ is independently sampled from $\Bern(\pstar_t)$. Without loss of generality, $\pstar_1 < \pstar_2 < \cdots < \pstar_T$. Then, the step calibration error can be written as:
\[
    \stepCE(x, \pstar)
=   \max_{t \in [T]}\left|\sum_{i=1}^{t}(x_i - \pstar_i)\right|.
\]

The above is simply the maximum deviation $\max_{t \in [T]}|X_t|$ of a random walk $(X_t)_{t=0}^{T}$ defined as $X_0 = 0$ and $X_t = X_{t-1} + (x_t - \pstar_t)$. Na\"ively controlling its expectation via Hoeffding's inequality and a union bound would give an $O(\sqrt{T \log T})$ upper bound. We can shave the logarithmic factor using Kolmogorov's inequality, which gives
\[
    \pr{}{\max_{t \in [T]}|X_t| \ge \tau} \le \frac{\Ex{}{X_T^2}}{\tau^2} \le \frac{T}{4\tau^2}.
\]
Integrating this tail bound would then give $\Ex{}{\max_{t \in [T]}|X_t|} = O(\sqrt{T})$.

\paragraph{Warm-up \#2: Truthfulness up to an $O(\sqrt{\log(T/c)})$ factor.}
Unfortunately, the analysis above does not immediately generalize to non-product distributions, as the conditional probabilities of $\pstar_1, \ldots, \pstar_T$ are random and may not have a fixed ordering. One might resort to a ``covering $+$ union bound'' argument, but the family of step functions does not admit a finite covering (in the $\ell_{\infty}$ sense).

Fortunately, in the smoothed setting in which each $\pstar_t$ is randomly drawn from a $c$-smoothed distribution, a simple discretization argument would suffice. Let $\eps \coloneqq c / T^2$ and consider an $\eps$-net of the interval $[0, 1]$: $V_\eps \coloneqq \{0, \eps, 2\eps, \ldots, 1\}$. We will relax $\stepCE(x, \pstar)$ into the following:
\begin{equation}\label{eq:overview-naive-net}
    \max_{\alpha \in V_\eps}\left|\sum_{t=1}^{T}(x_t - \pstar_t)\cdot\1{\pstar_t \in [0, \alpha]}\right|.
\end{equation}

Suppose that the supremum over $\alpha \in [0, 1]$ in $\stepCE(x, \pstar)$ is achieved by some $\alpha^\star \in [i\eps, (i+1)\eps]$. Then, the difference between the values of $\sum_{t=1}^{T}(x_t - \pstar_t)\cdot\1{\pstar_t \in [0,\alpha]}$ at $\alpha^\star$ versus at $\alpha' = i\eps$ is at most
\[
    \sum_{t=1}^{T}\1{\pstar_t \in (\alpha', \alpha^\star]}
\le \sum_{t=1}^{T}\1{\pstar_t \in [i\eps, (i+1)\eps]},
\]
the number of steps on which $\pstar_t$ falls into the interval $[i\eps, (i+1)\eps]$ of length $\eps$. Since each $\pstar_t$ is drawn from a distribution with density at most $1/c$, $\pr{}{\pstar_t \in [i\eps, (i+1)\eps]} \le \eps / c = 1/T^2$. It then follows that the effect of replacing $[0, 1]$ with the $\eps$-net $V_\eps$ is negligible.

It remains to upper bound the expectation of \Cref{eq:overview-naive-net}. Since for each $\alpha \in V_\eps$, $\sum_{t=1}^{T}(x_t - \pstar_t)\cdot\1{\pstar_t \in [0,\alpha]}$ is the outcome of a $T$-step martingale, applying Hoeffding's inequality with a union bound over $V_\eps$ gives an upper bound of $O(\sqrt{T\log|V_\eps|}) = O(\sqrt{T\log(T/c)})$. Extending this to $\stepCEsub$ shows that $\stepCEsub$ is truthful up to an $O(\sqrt{\log(T/c)})$ factor.

\paragraph{Removing the $\polylog(T)$ factor.} We further tighten the multiplicative factor to $\sqrt{\log(1/c)}$ via a chaining argument. Our technique amounts to controlling the maximum deviation $\max_{t \in [T]}|X_t|$ in a martingale $(X_t)_{t=0}^{T}$ without applying Kolmogorov's inequality, and rather applies a more ``combinatorial'' analysis. This analysis turns out to be generalizable to non-product distributions.

As a warm-up, we revisit the toy problem below:
\begin{quote}
    \textbf{Max-deviation of random walk:} Consider the random walk $(X_t)_{t=0}^{T}$ where $X_0 = 0$ and $X_t = X_{t-1} \pm 1$ with equal probability. Prove that $\Ex{}{\max_{t \in [T]}|X_t|} = O(\sqrt{T})$.
\end{quote}
While above would follow from Kolmogorov's inequality and an integration, here is a different proof: We consider $\approx\log_2 T$ ``levels'' of random variables. The zeroth level consists of only $X_T - X_0$. The first level contains $X_T - X_{T/2}$ and $X_{T/2} - X_0$. The second level contains $X_T - X_{3T/4}, X_{3T/4} - X_{2T/4}, \ldots$. In general, the $i$-th level divides the horizon into $2^i$ blocks of length $T/2^i$, and considers the displacement within each block. Then, we note that each $X_t$ can be written as the sum of at most $\approx\log_2 T$ terms, at most one from each level. It follows that $\Ex{}{\max_{t \in [T]}|X_t|}$ is at most $\sum_{i=0}^{\log_2T}Y_i$, where $Y_i$ is the expectation of the maximum absolute value among level $i$. Since level $i$ contains $2^i$ terms, each of which is a sum of $T/2^i$ independent samples from $\Unif(\{\pm1\})$, Hoeffding's inequality with a union bound gives $Y_i = O\left(\sqrt{(T/2^i)\log 2^i}\right)$. Summing over all $i$ shows $\Ex{}{\max_{t \in [T]}|X_t|} = O(\sqrt{T})$.

Here is how we apply the above to the analysis of the $c$-smoothed setting. We consider two discretization of $[0, 1]$: $V_c \coloneqq \{0, c, 2c, \ldots, 1\}$ and $V_\eps \coloneqq \{0, \eps, 2\eps, \ldots, 1\}$ for $\eps = c/T^2$. As argued earlier, replacing the interval $[0, 1]$ in $\stepCE$ with $V_\eps$ comes with a negligible error. If we could further replace $V_\eps$ with $V_c$, we would be done: controlling the maximum over $V_c$ only involves a union bound over $|V_c| = O(1/c)$ martingales, and leads to an $O(\sqrt{T\log(1/c)})$ upper bound. Thus, it remains to control the error when we replace $V_\eps$ with $V_c$. We divide the interval $[0, 1]$ into sub-intervals of length $c/2, c/4, c/8, \ldots, \eps = c/T^2$. For each $i$ between $1$ and $O(\log T)$, the $i$-th level of the division consists of $2^i/c$ intervals of length $c/2^{i}$. For each level $i$ and $j \in [2^i/c]$, we consider:
\[
    \sum_{t=1}^{T}(x_t - \pstar_t)\cdot\1{(j-1)\cdot (c/2^i)\le \pstar_t \le j\cdot (c/2^i)},
\]
which is the outcome of a $T$-step martingale. Furthermore, since each $\pstar_t$ is sampled from a distribution with density $\le 1/c$, it falls into the length-$(c/2^i)$ interval with probability at most $2^{-i}$. Therefore, the contribution of the $i$-th level to the step calibration error is upper bounded by $\sqrt{(2^{-i}T)\cdot\log (2^i/c)}$. Summing over all $i$ gives the desired upper bound of $O(\sqrt{T\log(1/c)})$.

\subsection{Minimize Step Calibration in the Adversarial Setup}
We sketch the proof of the \Cref{thm:alg} by giving a simple non-constructive argument for the $O(\sqrt{T\log T})$ error rate; we derive an explicit and efficient algorithm in the actual proof.

We apply the minimax argument of~\cite{Hart22} for minimizing the $\ell_1$ calibration error (also called the ECE) in an adversarial prediction setting. First, we restrict the forecaster so that its prediction is always a multiple of $1/\sqrt{T}$. Then, we note that both the adversary and the forecaster have finitely many deterministic strategies---each deterministic strategy of the adversary (resp.\ forecaster) maps the history (all the previous outcomes and predictions) to the next outcome (resp.\ prediction). By the minimax theorem, it suffices to show that, against any given, possibly randomized strategy of the adversary, the forecaster can achieve an $O(\sqrt{T\log T})$ error with respect to $\stepCE$.

In this scenario, at each step $t$, the forecaster can compute the conditional probability%
\if\arxiv1
$\pstar_t$
\else
$\pstar_t = \pr{}{x_t = 1 \mid x_{1:(t-1)}, p_{1:(t-1)}}$
\fi
using the adversary's strategy. Then, the forecaster predicts $p_t$ obtained by rounding $\pstar_t$ to the nearest multiple of $1/\sqrt{T}$. To control the resulting step calibration error, we note that there are only $O(\sqrt{T})$ values of $\alpha$ that need to be considered (namely, the multiples of $1/\sqrt{T}$). For each fixed $\alpha$, we can bound $\left|\sum_{t=1}^{T}(x_t - p_t)\cdot\1{p_t \le \alpha}\right|$ by the sum of two terms: one involving $(x_t - \pstar_t)$ and another involving $(\pstar_t - p_t)$. Since $|\pstar_t - p_t| \le 1/\sqrt{T}$ for all $t$, the latter term is always $O(\sqrt{T})$. For the former, we apply a union bound over the $O(\sqrt{T})$ values of $\alpha$. The resulting bound would scale as $O(\sqrt{T\log T})$.
While this argument does not give an explicit algorithm, we can also cast step calibration minimization as a Blackwell approachability problem~\cite{blackwell_analog_1956, Foster99} and obtain an explicit algorithm using min-max game dynamics~\cite{Hart22,haghtalab2024unifying}.

\section{Non-truthfulness of U-Calibration}
\label{sec:u-calibration}
Previous observations of non-truthfulness in calibration measures centered around a specific source of non-truthfulness: an incentive that calibration measures provide to forecasters to ``cancel out'' previous errors in their forecast by intentionally mispredicting future events.
This form of non-truthfulness can be remedied by randomly subsampling which timesteps are included in the calibration measure computation \cite{HQYZ24}.

In this section, we identify two new and qualitatively distinct sources of non-truthfulness in the U-Calibration measure, neither of which can be remedied with subsampling alone.
The first source of non-truthfulness arises from the inherently discontinuous nature of the U-Calibration measure and is largely inevitable: one can show that any reasonable calibration measure that provides decision-theoretic guarantees must also be non-truthful due to discontinuity.
However, this form of non-truthfulness requires an adversary to precisely choose the event distribution $\cD$ and, as we will later see, largely disappears under smoothed analysis.
The second source of non-truthfulness arises from the incentive that U-Calibration provides to a forecaster to \emph{hedge} their predictions, i.e., to exaggerate their uncertainty in their predictions.
This form of non-truthfulness remains even in the smoothed setting.

\subsection{Non-truthfulness from Discontinuity}
The U-Calibration measure is discontinuous in a forecaster's prediction.
From a decision-theoretic perspective, this is because an agent's mapping from the forecaster's prediction to an action is usually discontinuous: a marginal change in the probability of an event occurring may result in an agent switching actions and perhaps incurring significantly higher or lower regret.
The following proposition describes one such case.

\begin{proposition}
    For any $\epsilon \in (0, 1/2)$ and even number $T \in \integers$, there is a sequence of events $x_{1:T}$ and predictions $p_{1:T}$ such that changing $p_{1:T}$ to a similar alternative set of predictions $\tilde p_{1:T}$ where $\norm{p_{1:T} - \tilde p_{1:T}}_\infty \leq \epsilon$ increases the U-Calibration measure from $\UCal(x_{1:T}, p_{1:T}) = 0$ to $\UCal(x_{1:T}, \tilde p_{1:T}) = \Omega(T)$.
\end{proposition}
Our proof of this proposition also shows that \emph{no} calibration measure can be complete, continuous, and decision-theoretic simultaneously.
\begin{proof}
    We define the events $x_{1:T}$ and original predictions $p_{1:T}$ as $(x_t, p_t) = (1, \tfrac 12 - \eps)$ for the first half of timesteps $t \in [T/2]$ and  $(x_t, p_t) = (0, \tfrac 12 + \eps)$ for the second half of timesteps $t > \tfrac T2$. Recall the equivalent form of the V-Calibration error from \Cref{prop:altform}. For $\alpha = 1/2$, there are $N_{-}^{(\alpha)} = T/2$ steps on which $p_t < \alpha$, and the events on those steps sum up to $X_{-}^{(\alpha)} = T/2$. It follows that
    \[
        \VCal(x, p)
    \ge 2\left(X_{-}^{(\alpha)} - \alpha \cdot N_{-}^{(\alpha)}\right)
    =   \Omega(T).
    \]
    This further implies $\UCal(x, p) = \Omega(T)$ due to the equivalence of $\UCal$ and $\VCal$ (Lemma~\ref{lemma:UCal-vs-VCal}). On the other hand, the alternative predictions $\tilde p_{1:T} = \tfrac 12 \cdot \vec{1}_T$ would guarantee $\UCal(x, p) = 0$ and $\|p - \tilde p\|_{\infty} \le \eps$.
\end{proof}

The discontinuity of the U-Calibration measure provides a source of non-truthfulness that cannot be avoided with the subsampling technique of \cite{HQYZ24}.
The following proposition demonstrates an $O(\sqrt T)$-$\Omega(T)$ truthfulness gap for U-Calibration, as well as for its subsampled variant $\UCalsub$:
 \begin{align}
    \label{eq:subsampled}
\UCalsub(x_{1:T}, p_{1:T}) = \EEs{S \sim \Unif(2^{[T]})}{\UCal(x |_S, p|_S)}.
\end{align}
\begin{proposition}
\label{prop:subsamptruth}
Both the U-Calibration measure $\UCal$ and its subsampled variant $\UCalsub$ suffer from an $O(\sqrt T)$-$\Omega(T)$ truthfulness gap.
\end{proposition}
\begin{proof}
We will analyze the truthfulness gaps of V-Calibration $\VCal$ and its subsampled version $\VCalsub$; the proposition would then follow from Lemma~\ref{lemma:UCal-vs-VCal}.

\paragraph{The distribution of events.} Let $\eps \in (0, 1/4)$ be sufficiently small such that $\eps = O(1/\sqrt{T})$.
We now construct a distribution $\cD$ such that the conditional probability $\pstar_t$ at every timestep $t$ is guaranteed to fall into the interval $[\tfrac 12 - \eps, \tfrac 12 + \eps]$.
Let us define $\cD$ as fixing $\pstar_1 = \tfrac 12$ and, for all $t \in [T]$, letting
\[
\pstar_{t+1} = 
\begin{cases} 
\pstar_t + \frac{\epsilon}{2^t} & \text{if } x_t = 1, \\
\pstar_t - \frac{\epsilon}{2^t} & \text{if } x_t = 0.
\end{cases}
\]
That is, distribution $\cD$ adversarially sets the probabilities $\pstar_t$ in a binary search fashion based on the realizations of the preceding events $x_{1:t-1}$.

\paragraph{Truthful forecasts give a linear penalty.} By our adversarial construction of $\cD$, regardless of the realization of $(x, \pstar)$, the following holds for $\alpha^\star \asseq \pstar_{T+1}$ and every $t \in [T]$: (1) $\pstar_t \ne \alpha^\star$; (2) $\pstar_t < \alpha^\star$ implies $x_t = 1$; and (3) $\pstar_t > \alpha^\star$ implies $x_t = 0$.

Towards lower bounding $\VCal(x, \pstar)$, we consider the equivalent formulation of the V-Calibration error in \Cref{prop:altform} at $\alpha = \alpha^\star$. Our construction guarantees $X_{-}^{(\alpha)} = N_{-}^{(\alpha)}$ and $X_{+}^{(\alpha)} = 0$. Since $\alpha \in [1/2 - \eps, 1/2 + \eps] \subseteq [1/4, 3/4]$, we have
\[
    X_{-}^{(\alpha)} - \alpha \cdot N_{-}^{(\alpha)}
=   (1 - \alpha)\cdot N_{-}^{(\alpha)}
\ge N_{-}^{(\alpha)} / 4
\]
and
\[
    \alpha\cdot N_{+}^{(\alpha)} - X_{+}^{(\alpha)}
=   \alpha\cdot N_{+}^{(\alpha)}
\ge N_{+}^{(\alpha)} / 4.
\]
It follows that
\[
    \max\left\{X_{-}^{(\alpha)} - \alpha N_{-}^{(\alpha)}, \alpha N_{+}^{(\alpha)} - X_{+}^{(\alpha)}\right\}
\ge \max\left\{N_{-}^{(\alpha)}, N_{+}^{(\alpha)}\right\} / 4
\ge T / 8.
\]
By \Cref{prop:altform}, $\VCal(x, \pstar) \ge \Omega(T)$ always holds, which in turn gives $\err_{\VCal}(\cD, \Atruthful(\cD)) = \Omega(T)$.

To lower bound $\err_{\VCalsub}(\cD, \Atruthful(\cD))$, we note that the same argument as above gives
\[
    \VCal(x|_S, \pstar|_S) \ge \Omega(|S|)
\]
for every $S \subseteq [T]$. Taking an expectation over $S \sim \Unif(2^{[T]})$ gives $\err_{\VCalsub}(\cD, \Atruthful(\cD)) = \Omega(T)$.

\paragraph{Dishonest forecasts with an $O(\sqrt{T})$ penalty.} On the other hand, constantly predicting $1/2$ gives an $O(\sqrt{T})$ error with respect to both $\VCal$ and $\VCalsub$.
To see this, we apply \Cref{prop:altform}:
\[
    \VCal(x, \tfrac12\cdot\vec{1}_T)
=   2 \cdot \sup_{\alpha \in [0, 1]}\max\left\{X_{-}^{(\alpha)} - \alpha N_{-}^{(\alpha)}, \alpha N_{+}^{(\alpha)} - X_{+}^{(\alpha)}\right\}.
\]
Since all predictions take value $1/2$, we have
\[
    X_{-}^{(\alpha)} - \alpha N_{-}^{(\alpha)}
=   \1{\alpha > 1/2}\cdot \left(\sum_{t=1}^{T}x_t - \alpha T\right)
\le \left|\sum_{t=1}^{T}x_t - \frac{T}{2}\right|
\]
and
\[
    \alpha N_{+}^{(\alpha)} - X_{+}^{(\alpha)}
=   \1{\alpha < 1/2}\cdot \left(\alpha T - \sum_{t=1}^{T}x_t\right)
\le \left|\sum_{t=1}^{T}x_t - \frac{T}{2}\right|.
\]
Therefore, we have
\[
    \VCal(x, \tfrac12\cdot\vec{1}_T)
\le 2\left|\sum_{t=1}^{T}(x_t - 1/2)\right|
\le 2\left|\sum_{t=1}^{T}(x_t - \pstar_t)\right| + 2\left|\sum_{t=1}^{T}(\pstar_t - 1/2)\right|.
\]

Since $\pstar_t \in [1/2 - \eps, 1/2 + \eps]$ holds for every $t \in [T]$, the second term above is always at most $2\eps T = O(\sqrt{T})$. The first term is the deviation of a $T$-step martingale with bounded differences, and is thus bounded by $O(\sqrt{T})$ in expectation. Therefore, we have
\[
    \OPT_{\VCal}(\cD)
\le \Ex{x \sim \cD}{\VCal(x, \tfrac12\cdot\vec{1}_T)}
\le O(\sqrt{T}).
\]

Again, the argument above can be easily extended to show that
\[
    \Ex{x \sim \cD}{\VCal(x|_S, \tfrac12\cdot\vec{1}_{|S|})}
\le O\left(\sqrt{|S|}\right)
\le O(\sqrt{T})
\]
holds for every fixed $S \subseteq [T]$. Taking an expectation over $S \sim \Unif(2^{[T]})$ gives
\[
    \OPT_{\VCalsub}(\cD)
\le \Ex{x \sim \cD, S \sim \Unif(2^{[T]})}{\VCal(x|_S, \tfrac12\cdot\vec{1}_{|S|})}
\le O(\sqrt{T}).
\]

This establishes the $O(\sqrt{T})$-$\Omega(T)$ truthfulness gaps for $\VCal$ and $\VCalsub$ and finishes the proof.
\end{proof}

We can generalize the proof of Proposition~\ref{prop:subsamptruth} beyond the U-Calibration measure to \emph{any} decision-theoretic calibration measure that satisfies the very weak condition of completeness (Definition~\ref{def:complete-and-sound}). 
Recall that a calibration measure is decision-theoretic if it upper bounds the external regret of an agent that acts on the forecaster's predictions (or equivalently upper bounds U-Calibration) and a calibration measure is complete if it is low for a base-rate forecaster when all events occur with the same constant probability.
This generalization, stated formally in following proposition, implies that---without smoothed analysis---the requirement of a calibration measure providing decision-theoretic guarantees is directly at odds with that of being truthful.

\begin{proposition}
\label{prop:completedecisionlowerbound}
    Consider any complete decision-theoretic calibration measure $\CM_T$.
    Suppose that its completeness guarantee for $\alpha = \tfrac 12$ takes the rate of $f$, i.e.,
    \[\Ex{x_1, \ldots, x_T \sim \Bern(1/2)}{\CM_T(x, \tfrac 12 \cdot\vec{1}_T)} = O(f(T)).\]
    Then, $\CM_T$ has a truthfulness gap of $O(f(T))$-$\Omega(T)$.
\end{proposition}
\begin{proof}
As in the proof of Proposition~\ref{prop:subsamptruth}, we will define the distribution $\cD$ by setting $\pstar_1 = \tfrac 12$ and
\[
\pstar_{t+1} = 
\smash{
\begin{cases} 
\pstar_t + \frac{\epsilon}{2^t} & \text{if } x_t = 1, \\
\pstar_t - \frac{\epsilon}{2^t} & \text{if } x_t = 0
\end{cases}
}
\]
for some small $\epsilon = O(f(T)/T^2)$.
We therefore have from Proposition~\ref{prop:subsamptruth} that the truthful forecaster's U-Calibration measure is lower bounded by $\EEs{x \sim \cD}{\UCal(x, \pstar)} \geq \Omega(T)$.
Because $\CM_T$ is a decision-theoretic calibration measure, it must upper bound the U-Calibration measure, meaning that
\[
    \err_{\CM_T}(\cD, \Atruthful(\cD))
=   \EEs{x \sim \cD}{\CM_T(x, \pstar)}
\geq \Omega\left(\EEs{x \sim \cD}{\UCal(x, \pstar)}\right) \geq \Omega(T).\]

To upper bound the calibration measure for the non-truthful forecaster, we first observe that the construction of $\cD$ guarantees $|p^\star_t - 1/2| \leq \epsilon$ for all $t \in [T]$.
Thus, the total variation distance between $\cD$ and $\Unif(\{0, 1\}^T)$ is $O(\epsilon T)$. Since $\CM_T(\cdot, \cdot)$ takes value in $[0, T]$, by our choice of $\epsilon = O(f(T)/T^2)$ and the completeness of $\CM$, we have
\[
    \EEs{x \sim \cD}{\CM_T(x,  1/2\cdot\vec{1}_T)}
\le \EEs{x \sim \Unif(\{0, 1\}^T)}{\CM_T(x,  1/2\cdot\vec{1}_T)} + T \cdot O(\eps T)
= O(f(T)).
\]
\end{proof}

However, the existence of this conflict---and the proof of Proposition~\ref{prop:completedecisionlowerbound}---hinges on the assumption that an adversary has the ability to choose the event distribution $\cD$, in particular each conditional probability $\pstar_t$, in an arbitrarily precise way so as to exploit discontinuity.
As we will later see, under smoothed analysis where an adversary has limited precision in choosing $\pstar_t$, this inevitability of non-truthfulness for decision-theoretic measures largely disappears.

\subsection{Non-truthfulness from Hedging}
We now demonstrate a second source of non-truthfulness in the U-Calibration measure that arises from the incentivization of \emph{hedging}: forecasters can reduce their expected U-Calibration measure by portraying their beliefs as being more uncertain than they truly are.
This form of non-truthfulness does not depend on the existence of an adversary that is able to choose the distribution $\cD$, or equivalently $\pstar_{1:T}$, with high precision.
It also cannot be remedied with the technique of randomly subsampling timesteps.

One might expect that hedging, by introducing bias to a forecaster's predictions, should increase the U-Calibration measure.
However, we can construct a setting, stated formally in the following proposition, where hedging results in an $\Omega(T)$ bias but a U-Calibration measure of $0$.
Moreover, we can design this example to be asymmetric where the forecaster only hedges its predictions in one direction, e.g., if it believes the event will likely occur.
\begin{proposition}
There is a sequence of events $x_{1:T}$ and predictions $p_{1:T}$ with zero U-Calibration measure $\UCal(x, p) = 0$ but linear bias $\left|\sum_{t=1}^{T}(x_t - p_t)\right| = \Omega(T)$.
\end{proposition}
\begin{proof}
We will construct a simple example for $T = 2$, which can be repeated an arbitrary number of times to attain the claim for arbitrary $T$.
Consider the events $x = (0, 1)$ and forecasts $p = (0, 3/4)$. To show that $\UCal(x, p) = 0$, by \Cref{lemma:UCal-vs-VCal,prop:altform}, it suffices to prove that the following holds for every $\alpha \in [0, 1]$:
\[
    \max\left\{X_{-}^{(\alpha)} - \alpha N_{-}^{(\alpha)}, \alpha N_{+}^{(\alpha)} - X_{+}^{(\alpha)}\right\} \le 0,
\]
where $N_{-}^{(\alpha)}$ (resp., $N_{+}^{(\alpha)}$) denotes the number of timesteps on which the prediction is strictly below (resp., above) $\alpha$, and $X_{-}^{(\alpha)}$ (resp., $X_{+}^{(\alpha)}$) denotes the sum of the events on those steps. This can be done via the following cases analysis:
\begin{itemize}
    \item When $\alpha = 0$, we have $N_{-} = X_{-} = 0$ and $N_{+} = X_{+} = 1$. This gives $X_{-} - \alpha N_{-} = 0$ and $\alpha N_{+} - X_{+} = -1 < 0$.
    \item When $\alpha \in (0, 3/4)$, we have $N_{-} = 1$, $X_{-} = 0$ and $N_{+} = X_{+} = 1$. This gives $X_{-} - \alpha N_{-} = -\alpha < 0$ and $\alpha N_{+} - X_{+} = \alpha - 1 < 0$.
    \item When $\alpha = 3/4$, we have $N_{-} = 1$, $X_{-} = 0$, and $N_{+} = X_{+} = 0$. This gives $X_{-} - \alpha N_{-} = -\alpha < 0$ and $\alpha N_{+} - X_{+} = 0$.
    \item When $\alpha \in (3/4, 1]$, we have $N_{-} = 2$, $X_{-} = 1$, and $N_{+} = X_{+} = 0$. This gives $X_{-} - \alpha N_{-} = 1 - 2\alpha < 0$ and $\alpha N_{+} - X_{+} = 0$.
\end{itemize}

Therefore, we have $\UCal(x,p) = 0$.
On the other hand, the total bias of the predictions is $|(x_1 + x_2) - (p_1 + p_2)| = 1/4$.
Repeated $T / 2$ times, this gives a bias of $T/8 = \Omega(T)$ and a U-Calibration measure of $0$.
\end{proof}

This example demonstrates that the bias of hedging predictions does not negatively affect the U-Calibration measure.
On the other hand, hedging often provides an explicit advantage.
For example, hedging an honest prediction of $\pstar_t = 1$ down to $p_t = 3/4$ may not incur any cost, but may instead be of benefit for a previous or future timestep $t'$ where the event outcome is high variance with $\pstar_{t'} = 3/4$ and the forecaster seeks to ``dilute'' the variance.
In this way, we can construct an $e^{-\Omega(T)}$-$\Omega(\sqrt{T})$ truthfulness gap for the U-Calibration measure and its subsampled variant.
Later, we will extend this construction to the smoothed setting.

\begin{proposition}
    \label{prop:simplesubsamptruth}
    Both the U-Calibration measure $\UCal$ and its subsampled version $\UCalsub$ have an  $e^{-\Omega(T)}$-$\Omega(\sqrt{T})$ truthfulness gap.
\end{proposition}

\begin{proof}
Again, we analyze V-Calibration rather than U-Calibration, i.e., we will establish the truthfulness gap of $\VCal$ and $\VCalsub$, and the proposition would then follow from \Cref{lemma:UCal-vs-VCal}.

Let $T$ be an even number. Let $\pstar = (\tfrac 15, \tfrac 15, \ldots, \tfrac 15, \tfrac 45, \tfrac 45, \ldots, \tfrac 45)$ be the vector with $\tfrac T2$ copies of $\tfrac 15$ and $\tfrac 45$ each and $\D$ be the product distribution $\prod_{t=1}^{T}\Bern(\pstar_t)$. Similarly, let $p$ be the alternative ``non-truthful'' prediction $(\tfrac 25, \tfrac 25, \ldots, \tfrac 25, \tfrac 35, \tfrac 35, \ldots, \tfrac 35)$ where again both $\tfrac 25$ and $\tfrac 35$ appear exactly $\tfrac T2$ times. We will show that, under both $\VCal$ and $\VCalsub$, the truthful forecaster $\Atruthful(\cD)$ that predicts according to $\pstar$ receives an $\Omega(\sqrt{T})$ penalty, whereas predicting according to $p$ leads to an $e^{-\Omega(T)}$ penalty.

\paragraph{Truthful forecasts lead to an $\Omega(\sqrt{T})$ penalty.} We start by showing that $\err_{\VCal}(\cD, \Atruthful(\cD))$ and $\err_{\VCalsub}(\cD, \Atruthful(\cD))$ are both lower bounded by $\Omega(\sqrt{T})$. Towards applying \Cref{prop:altform}, we fix $\alpha = 1/5 + \eps$ for an arbitrarily small $\eps > 0$, and aim to lower bound the quantity
\[
    X_{-}^{(\alpha)} - \alpha\cdot N_{-}^{(\alpha)}
=   X_{-}^{(\alpha)} - (1/5 + \eps)\cdot\frac{T}{2},
\]
where $N_{-}^{(\alpha)} = T / 2$ is the number of timesteps on which the prediction is strictly smaller than $\alpha = 1/5 + \eps$, and $X_{-}^{(\alpha)} = \sum_{t=1}^{T/2}x_t$ is the sum of the $T/2$ events on those steps.

By definition of $\cD$, $X_{-}^{(\alpha)}$ follows the distribution $\Binomial(T/2, 1/5)$. Then, it holds with probability $\Omega(1)$ that $X_{-}^{(\alpha)} \ge (T/2)\cdot(1/5) + \Omega(\sqrt{T}) = T / 10 + \Omega(\sqrt{T})$, i.e., $X_{-}^{(\alpha)}$ exceeds its mean by $\Omega(\sqrt{T})$. Conditioning on this event, by \Cref{prop:altform}, we have
\[
    \VCal(x, \pstar) \ge X_{-}^{(\alpha)} - (1/5 + \eps)\cdot\frac{T}{2}
\ge \Omega(\sqrt{T}) - \eps T.
\]
This implies that
\[
    \err_{\VCal}(\cD, \Atruthful(\cD)) = \Ex{x \sim \cD}{\VCal(x, \pstar)} \ge \Omega(1) \cdot [\Omega(\sqrt{T}) - \eps T].
\]
Taking $\eps \to 0^{+}$ proves $\err_{\VCal}(\cD, \Atruthful(\cD)) = \Omega(\sqrt{T})$.

For $\VCalsub$, we note that the argument above shows that, for any fixed $S \subseteq [T]$,
\[
    \Ex{x \sim \cD}{\VCal(x|_S, \pstar|_S)} \ge \Omega\left(\sqrt{|S \cap [T/2]|}\right).
\]
Then, over the randomness in $S \sim \Unif(2^{[T]})$, $|S \cap [T/2]|$ follows $\Binomial(T/2, 1/2)$. It follows that $\Ex{S \sim \Unif(2^{[T]})}{\sqrt{|S \cap [T/2]|}} = \Omega(\sqrt{T})$, and
\[
    \err_{\VCalsub}(\cD, \Atruthful(\cD))
=   \Ex{x \sim \cD, S \subseteq 2^{[T]}}{\VCal(x|_S, \pstar|_S)}
=   \Omega(\sqrt{T}).
\]

\paragraph{Dishonest forecasts lead to an exponentially small penalty.} Suppose that the forecaster strategically predicts according to $p = (\tfrac 25, \tfrac 25, \ldots, \tfrac 25, \tfrac 35, \tfrac 35, \ldots, \tfrac 35)$ instead. Again, we start with the analysis for $\VCal$. Let $X_1 \coloneq \sum_{t=1}^{T/2}x_t$ and $X_2 \coloneq \sum_{t=T/2+1}^{T}x_t$ denote the sums of the first and second halves of the event sequence, respectively. Note that over the randomness in $x \sim \cD$, $X_1$ follows $\Binomial(T/2, 1/5)$ and $X_2$ follows $\Binomial(T/2, 4/5)$. By a Chernoff bound, with probability at least $1 - e^{-\Omega(T)}$, the following three inequalities hold simultaneously:
\begin{equation}\label{eq:UCal-hedging-three-ineq}
    X_1 \le \frac{T}{5}, \quad X_2 \ge \frac{3T}{10}, \quad \frac{2T}{5} \le X_1 + X_2 \le \frac{3T}{5}.
\end{equation}

It remains to show that the inequalities above together imply $\VCal(x, p) = 0$. Assuming this, since the V-Calibration error is at most $T$, we would then have $\Ex{x \sim \cD}{\VCal(x, p)} \le e^{-\Omega(T)}\cdot T = e^{-\Omega(T)}$ as desired. By \Cref{prop:altform}, it suffices to show that the inequalities in \eqref{eq:UCal-hedging-three-ineq} together imply that, for every $\alpha \in [0, 1]$,
\[
    \max\left\{X_{-}^{(\alpha)} - \alpha\cdot N_{-}^{(\alpha)}, \alpha\cdot N_{+}^{(\alpha)} - X_{+}^{(\alpha)}\right\} \le 0.
\]
This can be done via the following case analysis:
\begin{itemize}
    \item When $\alpha \in [0, 2/5)$, we have $N_{-}^{(\alpha)} = X_{-}^{(\alpha)} = 0$, $N_{+}^{(\alpha)} = T$ and $X_{+}^{(\alpha)} = X_1 + X_2$. It follows that
    \[
        X_{-}^{(\alpha)} - \alpha\cdot N_{-}^{(\alpha)} = 0
    \quad \text{and} \quad
    \alpha\cdot N_{+}^{(\alpha)} - X_{+}^{(\alpha)}
    =   \alpha T - (X_1 + X_2)
    \le \frac{2}{5}T - \frac{2T}{5}
    =   0.
    \]
    \item When $\alpha = 2/5$, we have $N_{-}^{(\alpha)} = X_{-}^{(\alpha)} = 0$, $N_{+}^{(\alpha)} = T/2$ and $X_{+}^{(\alpha)} = X_2$. It follows that
    \[
        X_{-}^{(\alpha)} - \alpha\cdot N_{-}^{(\alpha)} = 0
    \quad \text{and} \quad
    \alpha\cdot N_{+}^{(\alpha)} - X_{+}^{(\alpha)}
    =   \alpha T/2 - X_2
    \le \frac{T}{5} - \frac{3T}{10}
    <   0.
    \]
    \item When $\alpha \in (2/5, 3/5)$, we have $N_{-}^{(\alpha)} = N_{+}^{(\alpha)} = T/2$, $X_{-}^{(\alpha)} = X_1$, and $X_{+}^{(\alpha)} = X_2$. It follows that
    \[
        X_{-}^{(\alpha)} - \alpha\cdot N_{-}^{(\alpha)} = X_1 - \alpha\cdot\frac{T}{2}
    \le \frac{T}{5} - \frac{2}{5}\cdot\frac{T}{2}
    =   0
    \]
    and
    \[
    \alpha\cdot N_{+}^{(\alpha)} - X_{+}^{(\alpha)}
    =   \alpha T/2 - X_2
    \le \frac{3}{5}\cdot\frac{T}{2} - \frac{3T}{10}
    =   0.
    \]
    \item The remaining cases that $\alpha = 3/5$ and $\alpha \in (3/5, 1]$ hold by symmetry.
\end{itemize}
Therefore, we conclude that
\[
    \OPT_{\VCal}(\cD)
\le \Ex{x \sim \cD}{\VCal(x, p)}
\le e^{-\Omega(T)}.
\]

Towards upper bounding $\OPT_{\VCalsub}(\cD)$, we analyze the quantity
\[
    \Ex{x \sim \cD}{\VCalsub(x, p)}
=   \Ex{x \sim \cD, S \sim \Unif(2^{[T]})}{\VCal(x|_S, p|_S)}.
\]
As in the analysis for $\VCal$, we will identify a high-probability event (over the randomness in both $x$ and $S$) that: (1) happens with probability $1 - e^{-\Omega(T)}$; (2) implies $\VCal(x|_S, p|_S) = 0$. It would then immediately follow that $\Ex{x \sim \cD}{\VCalsub(x, p)} = e^{-\Omega(T)}$.

Let $N_1 \coloneqq |S \cap \{1, 2, \ldots, T/2\}|$ (resp., $N_2 \coloneqq |S \cap \{T/2 + 1, T/2 + 2, \ldots, T\}|$) denote the number of timesteps among the first half (resp., the second half) of the sequence that get subsampled into $S$. Let $X_1 \coloneqq \sum_{t \in S}x_t\cdot \1{t \le T/2}$ and $X_2 \coloneqq \sum_{t \in S}x_t \cdot \1{t > T/2}$ denote the sum of the events on those steps, respectively.

Note that each of $N_1$, $N_2$, $X_1$ and $X_2$ can be written as a sum of $\Omega(T)$ independent Bernoulli random variables. Furthermore, over the randomness in both $x$ and $S$, we have $\Ex{}{N_1} = \Ex{}{N_2} = T/4$, $\Ex{}{X_1} = T/20$ and $\Ex{}{X_2} = T/5$. Let $\eps \coloneqq 1/40$ be a small constant. By a Chernoff bound and the union bound, the following conditions hold simultaneously with probability $1 - e^{-\Omega(T)}$:
\begin{itemize}
    \item $N_1, N_2 \in [T/4 - \eps T, T/4 + \eps T]$.
    \item $X_1 \le T/20 + \eps T$ and $X_2 \ge T/5 - \eps T$.
    \item $X_1 + X_2 \in [T/4 - \eps T, T/4 + \eps T]$.
\end{itemize}

It remains to show that, assuming the conditions above, we have $\VCal(x|_S, p|_S) = 0$. Towards applying \Cref{prop:altform}, we analyze the quantity
\[
    \max\left\{X_{-}^{(\alpha)} - \alpha\cdot N_{-}^{(\alpha)}, \alpha\cdot N_{+}^{(\alpha)} - X_{+}^{(\alpha)}\right\}
\]
for $\alpha \in [0, 1]$ with respect to events $x|_S$ and predictions $p|_S$:
\begin{itemize}
    \item When $\alpha \in [0, 2/5)$, we have
    \[
        N_{-}^{(\alpha)} = 0, X_{-}^{(\alpha)} = 0, N_{+}^{(\alpha)} = N_1 + N_2, X_{+}^{(\alpha)} = X_1 + X_2.
    \]
    It follows that
    \[
        X_{-}^{(\alpha)} - \alpha\cdot N_{-}^{(\alpha)} = 0
    \]
    and
    \[
        \alpha\cdot N_{+}^{(\alpha)} - X_{+}^{(\alpha)}
    =   \alpha\cdot(N_1 + N_2) - (X_1 + X_2)
    \le \frac{2}{5} \cdot \left(\frac{T}{2} + 2\eps T\right) - \left(\frac{T}{4} - \eps T\right)
    =   -\frac{T}{20} + \frac{9}{5}\eps T
    <   0.
    \]
    \item When $\alpha = 2/5$, we have
    \[
        N_{-}^{(\alpha)} = 0, X_{-}^{(\alpha)} = 0, N_{+}^{(\alpha)} = N_2, X_{+}^{(\alpha)} = X_2.
    \]
    It follows that
    \[
        X_{-}^{(\alpha)} - \alpha\cdot N_{-}^{(\alpha)} = 0
    \]
    and
    \[
        \alpha\cdot N_{+}^{(\alpha)} - X_{+}^{(\alpha)}
    =   \frac{2}{5}\cdot N_2 - X_2
    \le \frac{2}{5} \cdot \left(\frac{T}{4} + \eps T\right) - \left(\frac{T}{5} - \eps T\right)
    =   -\frac{T}{10} + \frac{7}{5}\eps T
    <   0.
    \]
    \item When $\alpha \in (2/5, 3/5)$, we have
    \[
        N_{-}^{(\alpha)} = N_1, X_{-}^{(\alpha)} = X_1, N_{+}^{(\alpha)} = N_2, X_{+}^{(\alpha)} = X_2.
    \]
    It follows that
    \[
        X_{-}^{(\alpha)} - \alpha\cdot N_{-}^{(\alpha)}
    =   X_1 - \alpha \cdot N_1
    \le \frac{T}{20} + \eps T - \frac{2}{5}\cdot \left(\frac{T}{4} - \eps T\right)
    =   -\frac{T}{20} + \frac{7}{5}\eps T
    <   0.
    \]
    and
    \[
        \alpha\cdot N_{+}^{(\alpha)} - X_{+}^{(\alpha)}
    =   \alpha\cdot N_2 - X_2
    \le \frac{3}{5} \cdot \left(\frac{T}{4} + \eps T\right) - \left(\frac{T}{5} - \eps T\right)
    =   -\frac{T}{20} + \frac{8}{5}\eps T
    <   0.
    \]
    \item Finally, by symmetry, we have an upper bound of $0$ in the cases that $\alpha = 3/5$ and $\alpha \in (3/5, 1]$.
\end{itemize}

Therefore, we conclude that
\[
    \OPT_{\VCalsub}(\cD)
\le \Ex{x \sim \cD}{\VCalsub(x, p)}
\le e^{-\Omega(T)}.
\]
\end{proof}

By adding noises to the marginal probabilities, we can show that an $e^{-\Omega(T)}$-$\poly(T)$ truthfulness gap remains even if we assume a smoothed setting where an adversary cannot precisely choose the conditional probabilities $\pstar_{1:T}$.
Recall that, in the $c$-smoothed setting, the adversary is forced to sample each conditional probability $\pstar_t$ from a distribution with density upper bounded by $1/c$.

\begin{proposition}
\label{prop:subsampsmooth}
Both the U-Calibration measure $\UCal$ and its subsampled version $\UCalsub$ have an $e^{-\Omega(T)}$-$\Omega(T^{1/3})$ truthfulness gap, even when the event distribution is a product distribution with marginal probabilities perturbed by independent noises uniformly sampled from $[-c, c]$ for some constant $c \in (0, 1/5)$.
\end{proposition}

Intuitively, the constructions of both \Cref{prop:simplesubsamptruth,prop:subsampsmooth} involve incentivizing forecasters to exaggerate the uncertainty of their forecasts.
Hedging one's forecasts allows any and all downstream agents to avoid paying for the variance of the events and attain zero regret with high probability.
Notably, this form of non-truthfulness that arises when studying decision-theoretic calibration measures is qualitatively different from the non-truthfulness shown by \cite{HQYZ24}, which is more concerned with incentivizing forecasters to dishonestly ``patch up'' previous erroneous forecasts than with incentivizing forecasters to proactively hedge for (potentially future) uncertainty.

\begin{proof}
Let $T$ be an even number.
We construct a $c$-smoothed prior $\cP$ by sampling each event probability $\pstar_t$ in the first half of timesteps (i.e., $t \le T/2$) from the uniform distribution over $[\tfrac 15-c, \tfrac 15+c]$, and similarly sampling each $\pstar_t$ independently and uniformly from $[\tfrac 45 - c, \tfrac 45 + c]$ for $t \ge T/2+1$. More formally, for every $t \in [T]$ and $b_{1:(t-1)} \in \{0, 1\}^{t-1}$, we have
\[
    \cP(b_{1:(t-1)}) = \begin{cases} \Unif([\tfrac15 - c, \tfrac15 + c]), & t \le T/2,\\
    \Unif([\tfrac45 - c, \tfrac45 + c]), & t \ge T/2 + 1.
    \end{cases}
\]

\paragraph{Dishonest forecasts with low penalties.} Let $p = (\tfrac 25, \tfrac 25, \ldots, \tfrac 25, \tfrac 35, \tfrac 35, \ldots, \tfrac 35)$ be the vector with $T/2$ copies of $\tfrac 25$ and $\tfrac 35$ and represent the alternative ``non-truthful'' prediction. The analysis of the non-truthful forecaster remains largely the same as in the proof of Proposition~\ref{prop:simplesubsamptruth}.
This is because, by drawing $x \in \{0, 1\}^T$ according to $\cP$, the marginal distribution of $x$ remains the same as earlier, i.e., $x \sim \cD \coloneqq \prod_{t=1}^{T}\Bern(\overline{p}_t)$ where $\overline{p} = (\tfrac 15, \tfrac 15, \ldots, \tfrac 15, \tfrac 45, \tfrac 45, \ldots, \tfrac 45)$. It follows that
\[
    \OPT_{\UCal}(\cP)
\le \Ex{x \sim \calP}{\UCal(x, p)}
=   \Ex{x \sim \cD}{\UCal(x, p)}
\le e^{-\Omega(T)},
\]
and $\OPT_{\UCalsub}(\cP) \le e^{-\Omega(T)}$ by the same token.

\paragraph{Truthful forecasts give a high U-Calibration error.}
It remains to analyze the truthful forecaster by lower bounding the expectation of $\UCal(x, \pstar)$ and $\UCalsub(x, \pstar)$. In light of \Cref{lemma:UCal-vs-VCal,prop:altform}, we let $\alpha \coloneqq \tfrac15 - (1 - 2\gamma)c \in [\tfrac 15 - c, \tfrac 15 + c]$ for some $\gamma \in [0, 1]$ to be chosen later, and focus on lower bounding the quantity
\[
    X_{-}^{(\alpha)} - \alpha \cdot N_{-}^{(\alpha)},
\]
where $N_{-}^{(\alpha)} = \sum_{t=1}^{T}\1{\pstar_t < \alpha}$ is the number of steps on which the prediction is strictly below $\alpha$, and $X_{-}^{(\alpha)} = \sum_{t=1}^{T}x_t\cdot\1{\pstar_t < \alpha}$ is the sum of the events on those steps.

By our construction of $\cP$, for each $t \in \{1, 2, \ldots, T/2\}$, $\pstar_t$ is drawn independently and uniformly from $[\tfrac15 - c, \tfrac15 + c]$. It follows that
\[
    \pr{}{\pstar_t < \alpha} = \frac{\alpha - (\tfrac15 - c)}{2c}
=   \frac{2\gamma c}{2c} = \gamma,
\]
and $N_{-}^{(\alpha)}$ follows $\Binomial(T/2, \gamma)$. Furthermore, conditioning on that $\pstar_t < \alpha$, $\pstar_t$ is uniformly distributed over $[\tfrac15-c, \alpha)$, which implies that the conditional probability of $x_t = 1$ is $(\tfrac15 - c + \alpha) / 2 = \tfrac15 - (1 - \gamma)c$. Therefore, conditioning on the value of $N_{-}^{(\alpha)}$, $X_{-}^{(\alpha)}$ follows $\Binomial(N_{-}^{(\alpha)}, \tfrac15 - (1 - \gamma)c)$.

By a Chernoff bound, as long as $\gamma = \Omega(1/T)$, $N_{-}^{(\alpha)} \in [\gamma T / 4, \gamma T]$ holds with probability $\Omega(1)$. Furthermore, conditioning on the realization of $N_{-}^{(\alpha)} \in [\gamma T / 4, \gamma T]$, $X_{-}^{(\alpha)}$ has a conditional expectation of $\left[\tfrac15 - (1 - \gamma)c\right] \cdot N_{-}^{(\alpha)}$ and a conditional variance of
\[
    N_{-}^{(\alpha)} \cdot \left[\frac15 - (1 - \gamma)c\right] \cdot \left[\frac45 + (1 - \gamma)c\right]
\ge \frac{\gamma T}{4} \cdot \left(\frac{1}{5} - c\right) \cdot \frac{4}{5}
\ge \Omega(\gamma T),
\]
where the last step applies $c < 1/5$. By a central limit theorem, it holds with probability at least $\Omega(1)$ that
\begin{equation}\label{eq:CLT-condition}
    X_{-}^{(\alpha)}
\ge \left[\frac15 - (1 - \gamma)c\right] \cdot N_{-}^{(\alpha)} + 2\sqrt{\gamma T}.
\end{equation}
Assuming that both $N_{-}^{(\alpha)} \in [\gamma T / 4, \gamma T]$ and \eqref{eq:CLT-condition} hold, for $\gamma = T^{-1/3}$, we have
\begin{align*}
    X_{-}^{(\alpha)} - \alpha \cdot N_{-}^{(\alpha)}
&\ge \left[\tfrac15 - (1 - \gamma)c\right] \cdot N_{-}^{(\alpha)} - \left[\tfrac15 - (1 - 2\gamma)c\right] \cdot N_{-}^{(\alpha)} + 2\sqrt{\gamma T}\\
&=   2\sqrt{\gamma T} - \gamma c \cdot N_{-}^{(\alpha)}\\
&\ge 2\sqrt{\gamma T} - \gamma c \cdot \gamma T \tag{$N_{-}^{(\alpha)} \le \gamma T$}\\
&\ge T^{1/3}. \tag{$\gamma = T^{-1/3}$, $c \le 1$}
\end{align*}

Therefore, it holds with probability $\Omega(1)$ that $X_{-}^{(\alpha)} - \alpha \cdot N_{-}^{(\alpha)} \ge T^{1/3}$. By \Cref{lemma:UCal-vs-VCal,prop:altform}, we have
\[
    \err_{\UCal}(\cP, \Atruthful)
=   \Ex{(x, \pstar) \sim \cP}{\UCal(x, \pstar)}
\ge 2\Ex{(x, \pstar) \sim \cP}{\max\{X_{-}^{(\alpha)} - \alpha \cdot N_{-}^{(\alpha)}, 0\}}
=   \Omega(T^{1/3}).
\]

\paragraph{Truthful forecasts give a high subsampled U-Calibration error.} The analysis for $\UCalsub$ is almost the same. Consider a random subset $S \sim \Unif(2^{[T]})$. To lower bound $\UCal(x|_S, \pstar|_S)$, we examine the quantity
\[
    X_{-}^{(\alpha)} - \alpha \cdot N_{-}^{(\alpha)},
\]
where $\alpha = \tfrac15 - (1 - 2\gamma)c$, $N_{-}^{(\alpha)} = \sum_{t=1}^{T}\1{t \in S \wedge \pstar_t < \alpha}$, and $X_{-}^{(\alpha)} = \sum_{t=1}^{T}x_t\cdot\1{t \in S \wedge \pstar_t < \alpha}$. 

For each $t \in \{1, 2, \ldots, T/2\}$, the events $t \in S$ and $\pstar_t < \alpha$ are independent and happen with probabilities $1/2$ and $\gamma$, respectively. Therefore, $N_{-}^{(\alpha)}$ follows $\Binomial(T/2, \gamma / 2)$. Moreover, given $N_{-}^{(\alpha)}$, the conditional distribution of $X_{-}^{(\alpha)}$ is still $\Binomial(N_{-}^{(\alpha)}, \tfrac15 - (1 - \gamma)c)$. Then, the rest of the analysis goes through by considering the typical realization of $N_{-}^{(\alpha)} \in [\gamma T / 8, \gamma T / 2]$.
\end{proof}

\section{Step Calibration}
\label{sec:step-calibration}
In light of \Cref{prop:altform}, V-Calibration (\Cref{eq:V-calibration}) corresponds to testing the predictions on intervals of form $[0, \alpha)$ and $(\alpha, 1]$: If, among the steps where the prediction is $< \alpha$ (resp., $> \alpha$), the actual fraction of ones is significantly higher (resp., lower) than $\alpha$, we know that the predictions must be far from calibration.
Naturally, we would get a stronger measure if, in the above comparison, we replace $\alpha$ with the actual average of the predictions, and take an absolute value to penalize both over- and under-estimation.

Formally, we consider the following calibration measure, which we term the \emph{step calibration error}:
\[
    \stepCE(x, p) \coloneq \sup_{\alpha \in [0, 1]}\left|\sum_{t=1}^{T}(x_t - p_t)\cdot\1{p_t \le \alpha}\right|.
\]
Compared with the smooth calibration error of~\cite{KF08}, the step calibration error replaces the family of Lipschitz functions with the family of ``step functions'': $\{p \mapsto \1{p \le \alpha}: \alpha \in [0, 1]\}$.

The definition above is robust in the sense that, if we replace the step functions with the union of a constant number of intervals, the resulting error only increases by a constant factor.
The definition above also has a decision-theoretic interpretation: If we replace the benchmark in V-Calibration (i.e., $\inf_{\beta \in [0, 1]}\sum_{t=1}^{T}S_{\alpha}(x_t, \beta)$) with the sum
\[
    \sum_{t=1}^{T}\EEs{x'_t \sim \Bern(p_t)}{S_{\alpha}(x'_t, p_t)}
=   \sum_{t=1}^{T}(p_t - \alpha)\cdot\sgn(\alpha - p_t)
\]
and add an additional absolute value, we get
\[
    \sup_{\alpha \in [0, 1]}\left|\sum_{t=1}^{T}(x_t - \alpha)\cdot\sgn(\alpha - p_t) - \sum_{t=1}^{T}(p_t - \alpha)\cdot\sgn(\alpha - p_t)\right|
=   \sup_{\alpha \in [0, 1]}\left|\sum_{t=1}^{T}(x_t - p_t)\cdot\sgn(\alpha - p_t)\right|,
\]
which has a similar form to $\stepCE$, except that the step function is replaced with the sign function.
This definition can be interpreted as follows: Assuming that the bits were indeed draw from $\Bern(p_1)$ through $\Bern(p_T)$, we expect a loss of $\Ex{x'_t \sim \Bern(p_t)}{S_{\alpha}(x'_t, p_t)}$ at each step $t$. If our actual loss, $\sum_{t=1}^{T}S_{\alpha}(x_t, p_t)$, is significantly higher or lower, we are certain that the predictions cannot be calibrated.

We prove in \Cref{sec:equivfactor-proof} that the step calibration error is equivalent to the above variant of V-Calibration up to a constant factor.

\begin{restatable}{fact}{equivfactor}
\label{fact:equivfactor}
    For any $x \in \{0, 1\}^T$ and $p \in [0, 1]^T$,
    \begin{align*}
    \frac 13 \sup_{\alpha \in [0, 1]}\left|\sum_{t=1}^{T}(x_t - p_t)\cdot\sgn(\alpha - p_t)\right| 
    \leq \stepCE(x, p) 
    \leq \sup_{\alpha \in [0, 1]}\left|\sum_{t=1}^{T}(x_t - p_t)\cdot\sgn(\alpha - p_t)\right|.
    \end{align*}
\end{restatable}

\subsection{Step Calibration Error is Complete, Sound and Decision-Theoretic}
We show that the step calibration error is complete and sound in the sense of \Cref{def:complete-and-sound}. Furthermore, it is decision-theoretic, i.e., $\stepCE$ always upper bounds the U-Calibration error up to a constant factor.

\begin{proposition}\label{prop:stepCE-complete-sound-decision-theoretic}
    The step calibration error, $\stepCE$, is complete and sound. Moreover, for any $x \in \{0, 1\}^T$ and $p \in [0, 1]^T$, it holds that
    \[
        \stepCE(x, p) \ge \frac{1}{8}\UCal(x, p).
    \]
\end{proposition}

\begin{proof}
    We start by showing that $\stepCE$ is decision-theoretic. Let $A$ be a shorthand for 
    \if\arxiv0
    $\sup_{\alpha \in [0, 1]}\max\left\{X_{-}^{(\alpha)} - \alpha N_{-}^{(\alpha)}, \alpha N_{+}^{(\alpha)} - X_{+}^{(\alpha)}\right\}$.
    \else
    \[\sup_{\alpha \in [0, 1]}\max\left\{X_{-}^{(\alpha)} - \alpha N_{-}^{(\alpha)}, \alpha N_{+}^{(\alpha)} - X_{+}^{(\alpha)}\right\}.\]
    \fi
    By \Cref{lemma:UCal-vs-VCal,prop:altform}, we have
    \[
        A
    =   \frac{1}{2}\VCal(x, p)
    \ge \frac{1}{4}\UCal(x, p),
    \]
    so it suffices to prove $\stepCE(x, p) \ge A/2$.

    \paragraph{Step calibration is decision-theoretic.} Note that, over all possible values of $\alpha \in [0, 1]$, the term $\max\left\{X_{-}^{(\alpha)} - \alpha N_{-}^{(\alpha)}, \alpha N_{+}^{(\alpha)} - X_{+}^{(\alpha)}\right\}$ only takes finitely many (concretely, $O(T)$) different values. Thus, the supremum $A$ can be achieved at some $\alpha^* \in [0, 1]$, i.e.,
    \[
        A = \max\left\{X_{-}^{(\alpha^*)} - \alpha^* N_{-}^{(\alpha^*)}, \alpha^* N_{+}^{(\alpha^*)} - X_{+}^{(\alpha^*)}\right\}.
    \]
    We consider the following two cases, depending on which of the two terms above is larger.
    
    \paragraph{Case 1: $A = X_{-}^{(\alpha^*)} - \alpha^* N_{-}^{(\alpha^*)}$.} If $\alpha^* = 0$, we have $N_{-}^{(\alpha^*)} = X_{-}^{(\alpha^*)} = 0$. Then, $A = 0$, and $\stepCE(x, p) \ge A/2$ vacuously holds. If $\alpha^* > 0$, we can find $\beta \in [0, \alpha^*)$ such that $(\beta, \alpha^*) \cap \{p_1, p_2, \ldots, p_T\} = \emptyset$. Then, for every $t \in [T]$, $p_t < \alpha^*$ holds if and only if $p_t \le \beta$. It follows that
    \begin{align*}
        \stepCE(x, p)
    &\ge\sum_{t=1}^{T}(x_t - p_t)\cdot\1{p_t \le \beta}\\
    &=  \sum_{t=1}^{T}(x_t - p_t)\cdot\1{p_t < \alpha^*} \tag{$p_t < \alpha^* \iff p_t \le \beta$}\\
    &\ge\sum_{t=1}^{T}(x_t - \alpha^*)\cdot\1{p_t < \alpha^*} \tag{$p_t < \alpha^* \implies x_t - p_t \ge x_t - \alpha^*$}\\
    &=  X_{-}^{(\alpha^*)} - \alpha^* N_{-}^{(\alpha^*)}
    =   A,
    \end{align*}
    where the first step relaxes the supremum in the definition of $\stepCE$ to a fixed term at $\beta$, and removes the absolute value.

    \paragraph{Case 2: $A = \alpha^* N_{+}^{(\alpha^*)} - X_{+}^{(\alpha^*)}$.} In this case, we note that
    \[
        \sum_{t=1}^{T}(p_t - x_t) \cdot \1{p_t > \alpha^*}
    \ge \sum_{t=1}^{T}(\alpha^* - x_t) \cdot \1{p_t > \alpha^*}
    =  \alpha^* N_{+}^{(\alpha^*)} - X_{+}^{(\alpha^*)} 
    =   A.
    \]
    Since $\1{p_t > \alpha^*} = \1{p_t \le 1} - \1{p_t \le \alpha^*}$ holds for every $t \in [T]$, we have $B_1 - B_{\alpha^*} = A$, where
    \[
        B_1 = \sum_{t=1}^{T}(p_t - x_t) \cdot \1{p_t \le 1}
    \quad\text{and}\quad
        B_{\alpha^*} = \sum_{t=1}^{T}(p_t - x_t) \cdot \1{p_t \le \alpha^*}.
    \]
    Then, either $B_1$ or $B_{\alpha^*}$ must have an absolute value of at least $A/2$. It follows that
    \[
        \stepCE(x, p)
    \ge \max_{\beta \in \{\alpha^*, 1\}}\left|\sum_{t=1}^{T}(x_t - p_t)\cdot\1{p_t \le \beta}\right|
    =   \max\{|B_{\alpha^*}|, |B_1|\}
    \ge A/2.
    \]

    \paragraph{Completeness and soundness.} To verify the completeness, we note that $\stepCE$ is always upper bounded by the expected calibration error (ECE) defined as
    \[
        \ECE(x, p) \coloneqq \sum_{\alpha \in [0, 1]}\left|\sum_{t=1}^{T}(x_t - p_t)\cdot\1{p_t = \alpha}\right|.
    \]
    This is because, for every $\alpha \in [0, 1]$, it holds that
    \[
        \left|\sum_{t=1}^{T}(x_t - p_t)\cdot\1{p_t \le \alpha}\right|
    \le \sum_{\alpha' \in [0, \alpha]}\left|\sum_{t=1}^{T}(x_t - p_t)\cdot\1{p_t = \alpha'}\right|
    \le \ECE(x, p).
    \]
    Then, since the ECE is complete (e.g., \cite[Table 2]{HQYZ24}), $\stepCE$ is also complete.
    
    Similarly, since $\stepCE$ is lower bounded by $\UCal$ up to a constant factor, $\stepCE$ inherits the soundness of the U-Calibration error (e.g., \cite[Table 2]{HQYZ24}).
\end{proof}

\subsection{Step Calibration Error is Truthful After Subsampling}
As a warm-up, we start by proving that the subsampled version of step calibration, $\stepCEsub$, is $(\alpha, \beta)$-truthful for some parameters $\alpha, \beta = \polylog(T/c)$ in the $c$-smoothed setting. Later, we will improve the parameter $\alpha$ to $O(\sqrt{\log(1/c)})$ and also prove its tightness (for the truthfulness of $\stepCEsub$).

Recall that in the $c$-smoothed setting, the conditional expectation of each $x_t$ (given $x_{1:(t-1)}$) is sampled from a distribution with density upper bounded by $1/c$ (e.g., the uniform distribution over an interval of length $c$). Formally, the setting is described by a function $\calP:\bigcup_{t=1}^{T}\{0, 1\}^{t-1} \to \Delta_c([0, 1])$, where for each $t \in [T]$ and $(x_1, x_2, \ldots, x_{t-1}) \in \{0, 1\}^{t-1}$, $\calP(x_1, x_2, \ldots, x_{t-1})$ specifies a distribution over $[0, 1]$---with a density upper bounded by $c$---from which the conditional expectation of $x_t|x_1, x_2, \ldots, x_{t-1}$ is drawn. The sequence $x \in \{0, 1\}^{T}$ is determined sequentially as follows: For each $t \in [T]$, we draw $\pstar_t \sim \calP(x_1, x_2, \ldots, x_{t-1})$ and then draw $x_t \sim \Bern(\pstar_t)$. Furthermore, the true conditional probability for the event $x_t = 1$, $\pstar_t$, is observable by the forecaster. Thus, the truthful forecaster always predicts $\pstar_t$ at time $t$.

\paragraph{The $c = \Omega(1)$ regime.} Towards showing that $\stepCEsub$ is truthful, we will lower bound the optimal error that can be achieved on $\cP$ (namely, $\OPT_{\stepCEsub}(\cP)$) and upper bound the error incurred by the truthful forecaster (namely, $\err_{\stepCEsub}(\cP, \Atruthful)$). When $c = \Omega(1)$ is a fixed constant, we can easily obtain the following lower bound:
\[
    \OPT_{\stepCEsub}(\cP) = \Omega(\sqrt{T}),
\]
i.e., every forecaster must incur an $\Omega(\sqrt{T})$ penalty measured by $\stepCEsub$ in expectation. The expectation above is over the randomness in the realized outcomes $x$ as well as the forecaster that generates the predictions $p$. To see this, we note that
\[
    \stepCE(x, p)
\coloneq \sup_{\alpha \in [0, 1]}\left|\sum_{t=1}^{T}(x_t - p_t)\cdot\1{p_t \in [0, \alpha]}\right|
\ge \left|\sum_{t=1}^{T}(x_t - p_t)\right|,
\]
and it follows that
\[
    \stepCEsub(x, p)
\ge \Ex{y\sim\Unif(\{0, 1\}^T)}{\left|\sum_{t=1}^{T}y_t\cdot(x_t - p_t)\right|}.
\]

Then, the results of~\cite{HQYZ24} lower bound the right-hand side above by $\Ex{}{\gamma(\Var_T)}$, where
\[
\smash{
    \Var_T \coloneqq \sum_{t=1}^{T}\pstar_t(1 - \pstar_t)}
\]
is the \emph{realized variance} up to time $T$, and $\gamma(x) \coloneq x \cdot\1{x \le 1} + \sqrt{x} \cdot\1{x > 1}$. In the $c = \Omega(1)$ regime, we can show that $\Var_T \ge \Omega(T)$ with high probability, which implies the $\Omega(\sqrt{T})$ lower bound.

It remains to show that, when $c = \Omega(1)$, truthful forecasts lead to an $O(\sqrt{T})$ error with respect to the $\stepCEsub$ measure. As before, it suffices to show this for $\stepCE$, as the argument should extend to the subsampled version easily.
Via an easy ``covering $+$ union bound'' argument, we can prove an upper bound of
\[
    \err_{\stepCEsub}(\cP, \Atruthful)
=   \Ex{(x, \pstar) \sim \cP}{\stepCEsub(x,\pstar)}
=   O\left(\sqrt{T\log(T/c)}\right),
\]
almost matching $\OPT_{\stepCEsub} = \Omega(\sqrt{T})$ in the $c = \Omega(1)$ regime.

\paragraph{The small-$c$ regime.} The argument above, unfortunately, does not directly apply to the case where $c$ is small. This is because, in the worst case, $\Var_T$ can be as low as $\Theta(cT)$ (e.g., when each conditional distribution follows the uniform distribution over $[0, c]$). As a result, we can at best lower bound $\OPT_{\stepCEsub}$ by $\Omega(\sqrt{cT})$. Then, our upper bound on the $\stepCEsub$ incurred by the truthful forecaster---$\Ex{}{\stepCEsub(x,\pstar)}=O(\sqrt{T\log(T/c)})$---would be higher by a factor of $\sqrt{1/c}$.

Instead, we will replace both bounds---the lower bound on $\OPT_{\stepCEsub}(\cP)$ and the upper bound on $\err_{\stepCEsub}(\cP, \Atruthful)$---with ones that depend on the realized variance of the distribution, i.e., %
\[
    \Var_T \coloneqq \sum_{t=1}^{T}\pstar_t(1 - \pstar_t).
\]
It should be noted that the ``right'' way of defining $\Var_T$ should be using $\pstar_t$, rather than using the mean of the distribution $\calP(x_1, x_2, \ldots, x_{t-1})$. This is because revealing the value of $\pstar_t \sim \calP(x_{1:(t-1)})$ might significantly decrease the remaining variance in $x_t$.\footnote{Consider the case that $\calP(x_{1:(t-1)})$ is symmetric around $1/2$ and puts most of the probability mass near $0$ and $1$.}

The following lemma is implicit in~\cite{HQYZ24}:
\begin{lemma}[Theorem 6.5 of~\cite{HQYZ24}]\label{lemma:lower}
    For any choice of $\calP$, it holds that
    \[
        \OPT_{\stepCEsub}(\calP) = \Omega\left(\Ex{\calP}{\gamma(\Var_T)}\right),
    \]
    where $\gamma(x) \coloneqq x \cdot \1{x \le 1} + \sqrt{x} \cdot \1{x > 1}$.
\end{lemma}

\begin{proof}
    Note that
    \[
        \stepCE(x, p)
    =   \sup_{\alpha \in [0, 1]}\left|\sum_{t=1}^{T}(x_t - p_t)\cdot\1{p_t \in [0, \alpha]}\right|
    \ge \left|\sum_{t=1}^{T}(x_t - p_t)\right|.
    \]
    It follows that
    \[
        \stepCEsub(x, p)
    \ge \Ex{y \sim \Unif(\{0, 1\}^{T})}{\left|\sum_{t=1}^{T}y_t\cdot(x_t - p_t)\right|},
    \]
    so the rest of the proof follows from~\cite{HQYZ24} (which relaxes the SSCE to the same expression as above).
\end{proof}

We can show that the $\stepCEsub$ incurred by the truthful forecaster nearly matches the above, up to a multiplicative factor and an additive term of $\polylog(T/c)$:
\begin{lemma}\label{lemma:upper}
    For any $\calP$ that is $c$-smoothed, we have
    \[
        \err_{\stepCEsub}(\calP, \Atruthful)
    \le O\left(\Ex{}{\gamma(\Var_T)}\cdot\sqrt{\log(T/c)} + \log^2(T/c)\right),
    \]
    where $\gamma(x) \coloneqq x \cdot \1{x \le 1} + \sqrt{x} \cdot \1{x > 1}$.
\end{lemma}

Combining \Cref{lemma:lower,lemma:upper} shows that $\stepCEsub$ is $(\alpha, \beta)$-truthful for $\alpha = O(\sqrt{\log(T/c)})$ and $\beta = O(\log^2(T/c))$ in any $c$-smoothed setting.

\begin{proof}[Proof sketch]
    In the definition of $\stepCE$, we replace $\alpha \in [0, 1]$ in the supremum with a $(c/T^2)$-net of $[0, 1]$. Note that the change in the calibration measure is bounded by the maximum number of $\pstar$s that fall into the same length-$(c/T^2)$ bin. This can be shown to be $O(1)$ in expectation and will not be dominating.

    For each fixed choice of $\alpha$ (out of the $O(T^2/c)$ choices), the quantity $\sum_{t=1}^{T}(x_t - \pstar_t)\cdot\1{\pstar_t \in [0, \alpha]}$ induces a martingale. Then, we apply the same doubling trick as in~\cite{HQYZ24}to bound its deviation from zero. For the block in which the realized variance is roughly $2^k$ ($k = 0, 1, 2, \ldots, O(\log T))$, we take a union bound over $O(T^2/c)$ martingales, each with $\le T$ increments bounded between $\pm 1$ and a total realized variance $\le 2^k$. Freedman's inequality and the union bound show that the maximum deviation is at most $O\left(\sqrt{2^k\log(T/c)} + \log(T/c)\right)$ in expectation. Summing over $k$ would then give an upper bound of
    \[
        \Ex{}{\sqrt{\Var_T}} \cdot O(\sqrt{\log (T/c)}) + O(\log T) \cdot O(\log(T/c)).
    \]

    Finally, we note that $\sqrt{x} \le \gamma(x) + 1$ holds for every $x \ge 0$, so the upper bound above can be further relaxed to $\Ex{}{\gamma(\Var_T)} \cdot O(\sqrt{\log (T/c)}) + O(\log^2(T/c))$.
\end{proof}

\subsection{A Tighter Truthfulness Guarantee}
We give a more refined analysis that improves the $\sqrt{\log(T/c)}$ multiplicative factor to $\sqrt{\log(1/c)}$ by implementing a more involved covering strategy.

We start by noting that an $\Omega(\sqrt{\log(1/c)})$ factor is unavoidable; this follows from a ``binary search'' construction similar to the one for the $O(\sqrt{T})$-$\Omega(T)$ truthfulness gap of U-Calibration (\Cref{prop:subsamptruth}).

\begin{proposition}
\label{prop:soundlowerbound}
For any $c \in [\Theta(2^{-T}), \Theta(1)]$, in the $c$-smoothed setting, the step calibration error measure has an $O(\sqrt{cT} + 1)$-$\Omega(\sqrt{T \log(1/c)})$ truthfulness gap and the subsampled step calibration error measure has an $O(\sqrt T)$-$\Omega(\sqrt{T \log(1/c)})$ truthfulness gap.
\end{proposition}

\begin{proof}
Without loss of generality, we assume that $c \leq \tfrac 1 {16}$ and $T$ is divisible by $5\left\lfloor\log_2\frac{1}{8c}\right\rfloor$.

\paragraph{Construction of $\cP$.} We construct the $c$-smoothed setting $\cP$ by specifying the distribution from which each $\pstar_t$ is drawn. We divide the timesteps $[T]$ into $k + 1$ epochs $T_1, \dots, T_k, T_{k+1}$, where $k = \left\lfloor\log_2\frac{1}{8c}\right\rfloor$, epoch $T_{k+1} = \left\{\tfrac T5 + 1, \tfrac T5 + 2, \ldots, T\right\}$, and epoch $T_i = \left\{\tfrac{T(i - 1)}{5k} + 1, \tfrac{T(i - 1)}{5k} + 2, \ldots, \tfrac{T i }{5k}\right\}$ for all $i \in [k]$. This division is well-defined by our assumption that $T$ is divisible by $5k$.

For each $i \in [k]$, every timestep $t$ in epoch $T_i$ will have the same distribution of $\pstar_t$: $\pstar_t$ is uniformly distributed over $[w_i-\tfrac c2, w_i+\tfrac c2]$, where $w_i$ is defined in terms of the realized events of the previous epoch $\{x_{t}\}_{t \in T_{i-1}}$: We set $w_1 = \tfrac 12$ and, for every $i \ge 2$,
\[
w_{i} = 
\begin{cases} 
w_{i-1} + \frac{1}{2^{i+2}} & \text{if } \mu_{i-1} \geq w_{i-1}, \\
w_{i-1} - \frac{1}{2^{i+2}} & \text{if } \mu_{i-1} < w_{i-1},
\end{cases}
\]
where we use $\mu_i = \tfrac 1 {|T_{i}|} \sum_{t \in T_{i}} x_t$ to denote the average outcome in epoch $T_i$.
This guarantees that $|w_i - w_{i'}| \geq 2^{-(k + 3)} \geq c$ for all $i, i' \in [k+1]$ where $i \neq i'$ and that every $\pstar_t$ falls into the interval $[1/4, 3/4]$.
This in turn ensures $\alpha^\star = w_{k+1}$ satisfies, for every $i \in [k]$: (1) $w_i \notin [\alpha^\star - c, \alpha^\star + c]$, (2) $w_i  < \alpha^\star$ implies $\mu_i \geq w_i$, and (3) $w_i > \alpha^\star$ implies $\mu_i < w_i$.

For the last epoch, which spans the last $4T/5$ timesteps, we define $\pstar_t$ to be sampled uniformly from $[0, c]$ for timesteps $t \in [T/5+1, 3T/5]$ and $\pstar_t$ to be sampled uniformly from $[1 - c, 1]$ for timesteps $t \in [3T/5+1, T]$.

\paragraph{Truthful forecaster.}
We will analyze the subsampled step calibration error of the truthful forecaster, with the (non-subsampled) step calibration error bound following identically.
We first lower bound $\stepCEsub$ by fixing $\alpha = \alpha^\star$ and removing the absolute value:
\begin{align*}
\stepCEsub(x, \pstar)
&= \EEs{S \sim \Unif(2^{[T]})}{\sup_{\alpha \in [0, 1]} \abs{\sum_{t=1}^T (x_t - \pstar_t) \cdot \1{\pstar_t \in [0, \alpha] \land t \in S}}}
\\
&\geq \EEs{S \sim \Unif(2^{[T]})}{\sum_{t=1}^T (x_t - \pstar_t) \cdot \1{\pstar_t \in [0, \alpha^\star] \land t \in S}}.
\end{align*}
We next bound in expectation the right-hand summand for each epoch $T_i$ individually.
Fixing such an epoch $i$ and realization of previous outcomes $x'_{1:\max T_{i-1}}$ where $\max T_{i-1}$ is the last timestep of epoch $i-1$, we can lower bound
\begin{align*}
&~\EEsc{\substack{S \sim \Unif(2^{[T]}) \\ (x, \pstar) \sim \cP}}{\sum_{t \in T_i}  (x_t - \pstar_t) \cdot \1{\pstar_t \in [0, \alpha^\star] \land t \in S}}{x_{1:\max T_{i-1}} = x'_{1:\max T_{i-1}}} \\
=   &~\frac 12 \EEsc{\substack{ (x, \pstar) \sim \cP}}{\sum_{t \in T_i}  (x_t - w_i) \cdot \1{w_i \in [0, \alpha^\star]}}{x_{1:\max T_{i-1}} = x'_{1:\max T_{i-1}}}
\\
=   &~\frac 12  \EEsc{\substack{(x, \pstar) \sim \cP}}{ \EEs{X \sim \Binomial(|T_i|, w_i)}{\max\{X - |T_i| \cdot w_i, 0\}}}{x_{1:\max T_{i-1}} = x'_{1:\max T_{i-1}}} \\
\geq &~\frac 12 \min_{p \in [1/4, 3/4]} \EEs{\substack{X \sim \Binomial(|T_i|, p) }}{\max\{X - |T_i| \cdot p, 0\}}.
\end{align*}
In the above, the first step applies $w_i \notin [\alpha^\star - c, \alpha^\star + c]$ and marginalizes out $S$. The second step holds since $w_i \in [0, \alpha^\star]$ if and only if $\mu_i \ge w_i$, which is equivalent to $\sum_{t \in T_i}x_t \ge |T_i|\cdot w_i$. The last inequality follows from $w_i \in [1/4, 3/4]$.
To bound the last expectation, we first observe that the variance $\sigma^2$ of $X \sim \Binomial(|T_i|, p)$ is at least $\sigma^2 \geq \tfrac 3 {16} |T_i|$.
By the Berry-Esseen central limit theorem, we have $\Pr[X - |T_i| p \leq \sigma] \leq \Phi(1) + O(1/\sqrt{|T_i|})$.
Thus,
\[
    \EEs{X}{\max\{X - |T_i| \cdot p, 0\}}
\ge \sigma \cdot \Pr[X - |T_i| p \geq  \sigma]
\ge \Omega(\sigma)
\ge \Omega\left(\sqrt{|T_i|}\right).
\]
We can thus bound the step calibration error in the first $k$ epochs by
\begin{align*}
&\EEs{(x, \pstar) \sim \cP, S \sim \Unif(2^{[T]})}{\sum_{i=1}^k \sum_{t \in T_i}  (x_t - \pstar_t) \cdot \1{\pstar_t \in [0, \alpha^\star] \land t \in S}} 
\geq \Omega\left(\sum_{i=1}^k \sqrt{|T_i|}\right) = \Omega(\sqrt{Tk}).
\end{align*}
By definition of $\cP$, in the last epoch $k+1$, $\pstar_t \in [0, \alpha^\star]$ for the first $2T/5$ timesteps and  $\pstar_t \notin [0, \alpha^\star]$ for the last $2T/5$ timesteps.
Thus,
\begin{align*}
    &~\EEs{(x, \pstar) \sim \cP, S \sim \Unif(2^{[T]})}{\sum_{t \in T_{k+1}} (x_t - \pstar_t) \cdot \1{\pstar_t \in [0, \alpha^\star] \land t \in S}}\\
=   &~\frac 12 \EEs{(x, \pstar) \sim \cP}{\sum_{t=T/5+1}^{3T/5} (x_t - c/2)} 
= 0
\end{align*}
Combining the epochs, we obtain a subsampled step calibration error of at least
\[
    \err_{\stepCEsub}(\cP, \Atruthful) = \EEs{(x, \pstar) \sim \cP}{\stepCEsub(x, \pstar)} \geq \Omega(\sqrt{Tk}) = \Omega\left(\sqrt{T \log(1/c)}\right).
\]
We can easily verify that the same analysis goes through without subsampling, i.e., we also have $\err_{\stepCE}(\cP, \Atruthful) = \Omega(\sqrt{T \log(1/c)})$.

\paragraph{Non-truthful forecaster.}
We now turn to upper bounding the step calibration error of a dishonest forecaster. Suppose that the dishonest forecaster, denoted by $\cA$, predicts $p_t = \tfrac 12$ for the first $T/5$ timesteps and computes $\Delta = \sum_{t=1}^{T/5} x_t - \tfrac T {10}$, which denotes the deviation of realized outcomes from our prediction. Note that $|\Delta| \leq T/10$. If $\Delta < 0$, $p_t = \tfrac 12$ was an overestimate. The forecaster $\cA$ then predicts $p_t = \tfrac c2$ for timesteps $t \in [T/5+1, 3T/5]$. For $t \in [3T/5 + 1, T]$, $\cA$ predicts $\tfrac12$ until the bias at $\tfrac 12$ becomes non-negative. Formally, let $T' \asseq \max \bset{t \in [T] \mid \Delta + \sum_{\tau=3T/5+1}^{t-1} (x_{\tau-1} - \tfrac 12) < 0}$. Forecaster $\cA$ predicts $p_t = \tfrac 12$ for $t \in [3T/5+1, T']$ and $p_t = 1 - \tfrac c2$ for $t \in [T' + 1, T]$.
That is, we intentionally underestimate the true $\pstar_t = 1-\tfrac c2$ by guessing $p_t = \tfrac 12$ to cancel out the bias from the first $k$ epochs.
Note that this implies $\Delta + \sum_{t=3T/5+1}^{T} (x_{t-1} - \tfrac 12) \leq \tfrac 12$.
Let us condition on the event $E$ that $T' \leq 9T/10$ (and $\Delta < 0$), which occurs with probability at least $1 - \exp(-\Omega(T))$ by Hoefdding's inequality as the complementary event would require that at least $T/10$ of the timesteps $t \in [3T/5+1, 9T/10]$ result in $x_t = 0$.
Then,
\begin{align*}
    &~\EEsc{(x, p) \sim (\cP, \A)}{\stepCE(x, p)}{\Delta < 0} \\
=   &~\EEsc{(x, p) \sim (\cP, \A)}{\sup_{\alpha \in [0, 1]} \abs{\sum_{t=1}^T (x_t - p_t) \cdot \1{p_t \in [0, \alpha]}}}{\Delta < 0} \\
\leq&~\EEsc{(x, p) \sim (\cP, \A)}{\abs{\sum_{t=1}^T (x_t - p_t) \cdot \1{p_t = \tfrac 12}}}{E} + \EEsc{(x, p) \sim (\cP, \A)}{\abs{\sum_{t=1}^T (x_t - p_t) \cdot \1{p_t = \tfrac c2}}}{E} \\
& \quad + \EEsc{(x, p) \sim (\cP, \A)}{\abs{\sum_{t=1}^T (x_t - p_t) \cdot \1{p_t = 1 - \tfrac c2}}}{E} + T \cdot \exp(-\Omega(T)) \\
\leq&~\tfrac 12 + \EEs{X \sim \Binomial(2T/5, \tfrac c2)}{\abs{X - \tfrac {cT}5}} + \EEs{X \sim \Binomial(T-T', \tfrac c2)}{\abs{X - \tfrac {c(T - T')}2}} + o(1)\\
\leq&~O(\sqrt{cT} + 1).
\end{align*}

The case where $\Delta \geq 0$ follows symmetrically.
Defining $T' \asseq \max \{t \in [3T/5] \mid \Delta + \sum_{\tau=T/5+1}^{t-1} (x_{\tau-1} - \tfrac 12) < 0\}$, we fix predictions of $p_t = 1 - \tfrac c2$ for timesteps $t \in [3T/5+1, T]$, $p_t = \tfrac 12$ for $t \in [T/5+1, T']$, and $p_t = \tfrac c2$ for $t \in [T' + 1, 3T/5]$.
The event $E$ that $T' \leq 5T/10$ (and $\Delta > 0$) occurs with probability at least $1 - \exp(-\Omega(T))$ by Hoefdding's inequality as the complementary event would require that at least $T/10$ of the timesteps $t \in [T/5+1, 5T/10]$ result in $x_t = 1$.
Then,
\begin{align*}
& \EEsc{(x, p) \sim (\cP, \A)}{\stepCE(x, p)}{\Delta \geq 0}  \leq O(\sqrt{cT} + 1).
\end{align*}

Therefore, we have
\[
    \OPT_{\stepCE}(\cP)
\le \EEs{(x, p) \sim (\cP, \A)}{\stepCE(x, p)} \leq O(\sqrt{cT} + 1).
\]
Applying the inequality $\stepCEsub(x, p) \le \frac{1}{2}\stepCE(x, p) + O(\sqrt{T})$ from \Cref{lemma:stepCE-vs-stepCEsub} gives
\[
    \OPT_{\stepCEsub}(\cP)
\le O(\sqrt{T}).
\]
Since both $\err_{\stepCE}(\cP, \Atruthful)$ and $\err_{\stepCE}(\cP, \Atruthful)$ are lower bounded by $\Omega(\sqrt{T\log(1/c)})$, there is an $O(\sqrt{cT} + 1)$-$\Omega(\sqrt{T\log(1/c)})$ truthfulness gap for the step calibration error and an $O(\sqrt{T})$-$\Omega(\sqrt{T\log(1/c)})$ truthfulness gap for the subsampled step calibration error.
\end{proof}

We now turn to proving the tighter upper bound for the truthful forecaster's error.
\begin{theorem}
\label{thm:strongupperbound}
For any $T \ge 2$, smoothness parameter $c \in (0, 1]$, and $c$-smoothed setting specified by $\cP$, the expected subsampled step calibration error of truthfully predicting conditional probabilities $\pstar_{1:T}$, $\err_{\stepCEsub}(\cP, \cA^{\mathrm{truthful}})$, is upper bounded by
\begin{align*}
    &~\left(547 \sqrt{\ln(1 / c)} + 1161\right) \EE{\sqrt{\Var_T}} + (\log_2 T + 3)[16\log_2(T)\log_2(T/c) + 15] + 8 \sqrt{2\ln(1/c)} + 17\sqrt{2}\\
&~\;\le O\left(\sqrt{\log(1/c)}\right)\cdot\EE{\sqrt{\Var_T}} + O\left(\log^2(T)\log(T/c)\right).
\end{align*}
\end{theorem}

Combined with the lower bound $\OPT_{\stepCEsub}(\cP) \ge \Omega\left(\Ex{\cP}{\gamma(\Var_T)}\right)$ from \Cref{lemma:lower} and the inequality $\sqrt{x} \le \gamma(x) + 1$, the theorem above shows that $\stepCEsub$ is $(O(\sqrt{\log(1/c)}, \polylog(T/c))$-truthful in $c$-smoothed settings.

\begin{proof}
For convenience, given $\alpha \in [0, 1]$ and $\by \in \{0, 1\}^T$, we define the martingale
\begin{equation*}
M_t(\alpha, \by) \coloneqq \sum_{s=1}^t y_s\cdot \1{\pstar_{s} \leq \alpha} \cdot (x_{s}-\pstar_{s}),
\end{equation*}
adapted to the filtration $(\cF_t)_{t \in [T]}$ generated by sampling $\pstar_t \sim \cP(x_1, x_2, \ldots, x_{t-1})$ and $x_t \sim \Bern(\pstar_t)$ for each $t = 1, 2, \ldots, T$.
We can verify that $M_t(\alpha, \by)$ is a martingale by observing that, conditioned on any realization of $x_{1:(s-1)} = x'_{1:(s-1)}$, we indeed have
\begin{align*}
    &~\EEsc{\pstar_s, x_s}{y_s \cdot \1{\pstar_s \leq \alpha} \cdot (x_s - \pstar_s)}{x_{1:(s-1)} = x'_{1:(s-1)}} \\ 
=   &~y_s \cdot \EEs{\pstar_s \sim \cP\left(x'_{1:(s-1)}\right)}{\1{\pstar_s \leq \alpha} \cdot \EEs{x_s \sim \Bern(\pstar_s)}{x_s - \pstar_s}}\\
=   &~0.
\end{align*}
We can thus equivalently write our main claim as:
\begin{align*}
    &~\EEs{\substack{(x, \pstar) \sim \cP\\ \by \sim \Unif(\{0, 1\}^T)}}{\sup_{\alpha \in [0, 1]} \abs{M_T(\alpha, \by)}}\\
\le &~\left(547 \sqrt{\ln(\tfrac 1  c)} + 1161\right) \EE{\sqrt{\Var_T}} + (\log_2 T + 3)[16\log_2(T)\log_2(\tfrac T c) + 15] + 8 \sqrt{2\ln(\tfrac 1c)} + 17\sqrt{2}.
\end{align*}

\paragraph{Dividing into epochs.} For any realization of $x_{1:T} \sim \cP$, we will divide the timesteps $[T]$ into epochs in the same style as \cite{HQYZ24}: Let the $k$-th epoch be the shortest period (following the $(k-1)$-th epoch) whose realized variance is roughly at least $2^{k-1}$. Formally, consider sequence $\tau_0, \tau_1, \ldots \in \mathbb{Z} \cup \{\infty\}$, where each $\tau_k$ denotes the last step of the $k$-th epoch and is defined as $\tau_0=0$ and
\[
      \tau_k \asseq \min \{t \in [\tau_{k-1} + 1, T] \mid \Var_{t} -\Var_{\tau_{k-1}} + (t - \tau_{k-1}) \cdot c \geq 2^{k-1}\} \cup \bset{\infty}.
\]
We will also write $I_k \asseq \{\tau_{k-1} + 1, \tau_{k-1} + 2, \ldots, \min\{T, \tau_k\}\}$ to denote the timesteps in epoch $k$.

We first recall the following useful facts about this epoch division from \cite{HQYZ24}.
The realized variance in epoch $k$, $\Var_{\tau_k} - \Var_{\tau_{k-1}}$, is at most $2^{k-1} + 1/4 \le 2^k$. For the $k$-th epoch to be complete (i.e., $\tau_k < \infty$), a necessary condition is that $\Var_T + cT \ge 2^{k-1}$.
In particular, since $\Var_T + cT \le T/4 + T < 2T$, the $(\ceil{\log_2 T} + 2)$-th epoch is never complete, i.e., $\tau_{\ceil{\log_2 T} + 2} = \infty$.
The random variable $\1{t \in I_k}$ for an epoch $k$ is measurable by $x_{1:(t-1)}$. We also note that the $k$-th epoch cannot last more than $\ceil{2^{k-1}/c}$ steps: If $\tau_{k-1} + \ceil{2^{k-1}/c} = t\le T$, we would have $\Var_t - \Var_{\tau_{k-1}} + (t - \tau_{k-1})\cdot c \ge \ceil{2^{k-1}/c} \cdot c \ge 2^{k-1}$, which implies $\tau_k \le \tau_{k-1} + \ceil{2^{k-1}/c}$.

Next, we observe that $\Ex{}{\sqrt{\Var_T}} = \Omega(\sqrt{cT})$ holds in every $c$-smoothed setting. Recall that each $\pstar_t$ is sampled from a distribution with density bounded by $1/c$. Therefore, with probability at least $1 - (1/c)\cdot (c/4) = 3/4$, we have $\pstar_t \in [c/8, 1 - c/8]$, which implies $\pstar_t (1 - \pstar_t) \geq (c/8) \cdot (1 - c/8) \ge \tfrac{7c}{64}$. Therefore, with probability at least $1/2$, $\pstar_t(1 - \pstar_t) \ge \tfrac{7c}{64}$ holds for at least $T/2$ values of $t \in [T]$, and it follows that
\[
    \EE{\sqrt {\Var_T}} \geq \frac{1}{2}\cdot\sqrt{\frac{7c}{64}\cdot \frac{T}{2}} \ge  \frac{\sqrt{7cT}}{16\sqrt{2}}.
\]

We also want to bound the exponentially weighted sum
\begin{equation*}
\sum_{k=2}^{\ceil{\log_2 T} + 2
} \sqrt{2^k} \Pr[\tau_{k-1} < \infty].
\end{equation*}
To this end, let $j \in \integers$ be the random variable defined such that $\Var_T + c T \in [2^j, 2^{j+1})$. 
We observe that $\sqrt{\Var_T + cT} \geq \sqrt{2^{j}}$.
Also recall that, for any $k \ge 2$, $\tau_{k-1} < \infty$ holds only if $\Var_T + cT \ge 2^{k-2}$, which in turn holds only if $k - 2 \le j$.
We can thus bound
\begin{equation}
\label{eq:gg2}
\begin{aligned}
\sum_{k=2}^{\ceil{\log_2 T} + 2
} \sqrt{2^k} \Pr[\tau_{k-1} < \infty]
&\leq
\EE{\sum_{k=2}^{\ceil{\log_2 T} + 2
} \sqrt{2^k} \cdot \1{\Var_T + cT \geq 2^{k-2}}} \\
&\leq
\EE{\sum_{k=2}^{j+2} \sqrt{2^k}} \\
&\leq (4 + 2\sqrt{2})\EE{2^{j/2}} \\
&\leq (4 + 2\sqrt 2) \EE{\sqrt{\Var_T + cT}} \\
&\leq
20 (2 + \sqrt 2) \EE{ \sqrt{\Var_T}}.
\end{aligned}
\end{equation}
The last step applies $\EE{\sqrt {\Var_T}} \ge  \frac{\sqrt{7cT}}{16\sqrt{2}}$, which gives
\[
    \Ex{}{\sqrt{\Var_T + cT}}
\le \Ex{}{\sqrt{\Var_T}} + \sqrt{cT}
\le \Ex{}{\sqrt{\Var_T}} \cdot \left(1 + \frac{16\sqrt{2}}{\sqrt{7}}\right)
\le 10\Ex{}{\sqrt{\Var_T}}.
\]

\paragraph{Dividing step functions into rounded segments.}
Let $V_{\epsilon} = \{0, \epsilon, 2\epsilon, \ldots, \floor{1/\epsilon} \epsilon\}$ denote the multiples of $\eps$ in $[0, 1]$. Note that $|V_\epsilon| = \floor{1/\epsilon} + 1 \le 2/\eps$ elements. Let us fix an epoch $k$, condition on the epoch being reached (i.e., $\tau_{k-1} < \infty$) and define the following restriction of the martingale $M_t(\alpha, \by)$ to timesteps lying in epoch $k$:
\begin{equation*}
M_t(\alpha, \by, k) \coloneqq \sum_{s=1}^t y_s\cdot \1{\pstar_{s} \leq \alpha} \cdot (x_{s}-\pstar_{s}) \cdot \1{s \in I_k}.
\end{equation*}
Recall that we want to bound the supremum of this martingale over different values of the step threshold $\alpha$. Let us fix a step threshold $\alpha \in [0, 1]$ for now.
Fixing some $\epsilon^* > 0$ and integer $m \ge 1$ which we will specify later, we can define $w_0 = \floor{\alpha / \epsilon^*} \cdot \epsilon^*$ as a rounding down of $\alpha$ onto the grid $V_{\eps^*}$ of resolution $\epsilon^*$.
Then, for all $i \in [m]$, we define recursively $w_i = w_{i-1}+2^{-i} \epsilon^*$
if $w_{i-1}+2^{-i} \epsilon^* \leq \alpha$ and $w_i = w_{i-1}$ otherwise. Equivalently, each $w_i$ is the rounding of $\alpha$ down to the nearest multiple of $2^{-i}\epsilon^*$.
Note that
\[
    \{[0, w_0]\} \cup \{(w_{i-1}, w_i]: i \in [m]\}
\]
forms a partition of $[0, w_m]$, so we have the decomposition
\[
    \1{x \leq \alpha} = \1{x \in [0, w_0]} + \sum_{i=1}^{m} \1{x \in (w_{i-1}, w_i]} + \1{x \in (w_m, \alpha]}.
\]
Thus,
\begin{align*}
\abs{M_t(\alpha, \by, k)}
& = \abs{\sum_{s=1}^t y_s \cdot \Big(\1{\pstar_s \leq w_0} + \sum_{i=1}^{m} \1{\pstar_s \in (w_{i-1}, w_i]} + \1{\pstar_s \in (w_m, \alpha]}\Big) \cdot (x_s - \pstar_s) \cdot \1{s \in I_k}} \\
& \le \abs{\sum_{s=1}^{t}y_s\cdot\1{\pstar_s \in (w_m, \alpha]}\cdot(x_s - \pstar_s)\cdot\1{s \in I_k}}
\\
&\quad
+ 
\abs{\sum_{s=1}^{t}y_s\cdot\1{\pstar_s \le w_0}\cdot(x_s - \pstar_s)\cdot\1{s \in I_k}}
\\
&\quad
+ \sum_{i=1}^{m}\abs{\sum_{s=1}^{t}y_s\cdot\1{\pstar_s \in (w_{i-1}, w_i]}\cdot(x_s - \pstar_s)\cdot\1{s \in I_k}}.
\end{align*}

We can simplify this expression by noting that the first term above is upper bounded by
\[
    \sum_{s=1}^{t}\1{\pstar_s \in (w_m, w_m + 2^{-m}\cdot\epsilon^*]},
\]
since $\alpha < w_m + 2^{-m}\cdot\epsilon^*$ while $y_s \in \{0, 1\}$ and $\abs{x_s - \pstar_s} \le 1$ hold for every $s$.
Let us also, with some abuse of notation, define for every interval $S \subset [0, 1]$:
\begin{equation*}
M_t(S, \by, k) \asseq \sum_{s=1}^t y_t \cdot \1{\pstar_s \in S} \cdot (x_s - \pstar_s) \cdot \1{s \in I_k}.
\end{equation*}
Finally, we recall that $w_i \in V_{2^{-i} \cdot \epsilon^*}$ for every $i \in \{0, 1, \ldots, m\}$.
Putting these together, we can simplify our upper bound on $\abs{M_T(\alpha, \by, k)}$ to:
\begin{equation}
\label{eq:divide}
\begin{aligned}
\sup_{\alpha \in [0, 1]} \abs{M_T(\alpha, \by, k)} \leq  &\; \max_{w_0 \in V_{\epsilon^*}} \abs{M_T(w_0, \by, k)} + \sum_{i=1}^m \max_{w_i \in V_{2^{-i} \cdot \epsilon^*}} \abs{M_T((w_i, w_i + 2^{-i} \cdot \epsilon^*], \by, k)} \\ & \quad 
+ \max_{w_m \in V_{2^{-m}\cdot\epsilon^*}}\sum_{t=1}^{T} \1{\pstar_t \in (w_m, w_m + 2^{-m} \cdot \epsilon^*]}.
\end{aligned}
\end{equation}

\paragraph{Bounding each summand.}
We first note the following lemmas, which we will prove in the sequel.

\begin{restatable}{lemma}{prefixbound}
\label{lemma:prefix_bound}
For any epoch $k$,
\begin{equation*}
\begin{aligned}
\EEc{\max_{w_0 \in V_{\epsilon^*}} \abs{M_T(w_0, \by, k)}}{\tau_{ k-1} < \infty} 
\leq & \sqrt{2^{k+1} \ln(4/\epsilon^*)} + \sqrt{2^{k-1} \pi} + 2 \ln(4/\epsilon^*)  + 2.
\end{aligned}
\end{equation*}
\end{restatable}

\begin{restatable}{lemma}{segmentbound}
\label{lemma:segment_bound}
    For any epoch $k$ and level $i \in [m]$,
    \begin{equation*}
\begin{aligned}
& \EEc{\max_{w_i \in V_{2^{-i} \cdot \epsilon^*}} \abs{M_T((w_i, w_i + 2^{-i} \cdot \epsilon^*], \by, k)}}{\tau_{k-1} < \infty} \\
\leq \; &  \sqrt{(\epsilon^* \cdot 2^{k-i+1} / c^2) \ln(2^{i+2}/\epsilon^*)}    +  \sqrt{\pi \cdot \epsilon^* \cdot 2^{k-i-1} / c^2} + 2 \ln\left(2^{i+2}/\epsilon^*\right) + 2.
\end{aligned}
\end{equation*}
\end{restatable}

\begin{restatable}{lemma}{suffixbound}
\label{lemma:suffix_bound}
For any epoch $k$,
\[
    \Ex{}{\max_{w_m \in V_{2^{-m}\cdot\epsilon^*}}\sum_{t=1}^{T} \1{\pstar_t \in (w_m, w_m + 2^{-m} \cdot \epsilon^*]}}
\le 2T\cdot\frac{\eps^*}{2^m c} + 3(\ln(2^{m+1}/\eps^*) + 1).
\]
\end{restatable}
We now set the parameters $\eps^*$ and $m$ as
\[
    \epsilon^* \coloneqq c^2, \quad m \coloneqq \floor{\log_2 T},
\]
and substitute them into the above lemmas for the following simplified bound for each epoch $k$.
\begin{restatable}{lemma}{stepcalibepoch}
\label{lemma:staepcalibepoch}
The subsampled step calibration error in epoch $k$ is upper bounded by
\begin{equation*}
   \EEc{\sup_{\alpha \in [0, 1]} M_t(\alpha, \by, k)}{\tau_{k-1} < \infty} \leq  8\sqrt{2^k\ln(1/c)} + 17\sqrt{2^k} + 16\log_2(T)\log_2(T/c) + 15.
\end{equation*}
\end{restatable}
Summing Lemma~\ref{lemma:staepcalibepoch} over all epochs, we have
\begin{equation*}
    \begin{aligned}
        &~\EE{\sup_{\alpha \in [0, 1]} M_t(\alpha, \by)} \\
    \leq&~\sum_{k=1}^{\ceil{\log_2 T} + 2} \Pr[\tau_{k-1} < \infty] \EEc{\sup_{\alpha \in [0, 1]} M_t(\alpha, \by, k)}{\tau_{k-1} < \infty} \\
    \leq&~\sum_{k=1}^{\ceil{\log_2 T} + 2} \Pr[\tau_{k-1} < \infty] \Big[8\sqrt{2^{k} \ln(1/c)} + 17 \sqrt{2^k} + 16\log_2(T)\log_2(T/c) + 15 \Big]\\
    \leq&~\sum_{k=1}^{\ceil{\log_2 T} + 2} \Pr[\tau_{k-1} < \infty] \Big[ \sqrt{2^k} \left(8\sqrt{\ln(1 / c)} + 17\right) + 16\log_2(T)\log_2(T/c) + 15 \Big].
    \end{aligned}
\end{equation*}
    Applying \Cref{eq:gg2} gives the upper bound
\begin{equation*}
    \begin{aligned}
        &~\EE{\sup_{\alpha \in [0, 1]} M_t(\alpha, \by)} \\
    \le &~\left(8\sqrt{\ln(1/c)} + 17\right)\cdot\left[\sqrt{2} + \sum_{k=2}^{\ceil{\log_2 T} + 2}\Pr[\tau_{k-1} < \infty]\cdot\sqrt{2^k}\right] + \left(\ceil{\log_2 T} + 2\right)\cdot[16\log_2(T)\log_2(T/c) + 15]\\
    \le &~\left(547 \sqrt{\ln(1 / c)} + 1161\right) \EE{\sqrt{\Var_T}} + (\log_2 T + 3)[16\log_2(T)\log_2(T/c) + 15] + 8 \sqrt{2\ln(1/c)} + 17\sqrt{2}.
    \end{aligned}
\end{equation*}
\end{proof}

\paragraph{Remaining proofs.}
We now prove \Cref{lemma:prefix_bound}, \Cref{lemma:segment_bound}, \Cref{lemma:suffix_bound}, and \Cref{lemma:staepcalibepoch}.

\prefixbound*
\begin{proof}
We first note that the realized variance of $M_t(w_0, \by, k)$, for any particular choice of $w_0 \in V_{\eps^*}$, is bounded by the realized variance within the $k$-th epoch, which is in turn upper bounded by $2^k$. Thus, for any fixed choice of $w_0$ and any $p \in (0, 1)$, Freedman's inequality gives that, with probability at least $1 - p$,
\[
\abs{M_T(w_0, \by, k)} \leq \sqrt{2\cdot 2^{k}\cdot \ln(2/p)} + 2 \ln(2/p).
\]
Applying a union bound over the $|V_{\epsilon^*}| \le 2/\eps^*$ choices of $w_0 \in V_{\eps^*}$, we have with probability $1 - p$:
\begin{align*}
\max_{w_0 \in V_{\epsilon^*}} \abs{M_T(w_0, \by, k)}
\leq \sqrt{2^{k+1} \ln((4/\epsilon^*)/p)} + 2 \ln((4/\epsilon^*)/p).
\end{align*}
We can then use the standard inequality $\sqrt{a + b} \le \sqrt{a} + \sqrt{b}$ (for $a, b \ge 0$) to relax the above into
\[
    \max_{w_0 \in V_{\epsilon^*}} \abs{M_T(w_0, \by, k)}
\leq \sqrt{2^{k+1} \ln(4/\epsilon^*)} + \sqrt{2^{k+1} \ln(1/ p)} + 2 \ln(4/\epsilon^*) + 2 \ln(1/p).
\]
Taking the layer-cake representation, we have
\begin{equation*}
\begin{aligned}
        &~\EEc{\max_{w_0 \in V_{\epsilon^*}} \abs{M_t(w_0, \by, k)}}{\tau_{k-1} < \infty} \\
\leq    &~\int_0^1 \left[\sqrt{2^{k+1} \ln(4/\epsilon^*)} + \sqrt{2^{k+1} \ln(1/ p)} + 2 \ln(4/\epsilon^*) + 2\ln(1/p)\right] \; \mathrm{d}p \\
=       &~\sqrt{2^{k+1} \ln(4/\epsilon^*)} + 2 \ln(4/\epsilon^*) + \sqrt{2^{k+1}} \int_0^1  \sqrt{\ln (1/p)} \; \mathrm{d}p + 2 \int_0^1 \ln(1/p) \; \mathrm{d}p \\
=       &~\sqrt{2^{k+1} \ln(4/\epsilon^*)} + 2 \ln(4/\epsilon^*) + \sqrt{2^{k-1} \pi} + 2,
\end{aligned}
\end{equation*}
where the last equality uses the identities $\int_0^1 \sqrt{\ln(1/x)} \; \mathrm{d} x = \tfrac {\sqrt{\pi}}{2}$ and $\int_0^1 \ln(1/x) \; \mathrm{d} x = 1$.
\end{proof}

\segmentbound*
\begin{proof}
Let us fix a $w_i \in V_{2^{-i} \cdot \epsilon^*}$.
We will now study segments of length $2^{-i} \cdot \epsilon^*$. %
We first bound the realized variance of the martingale  $M_t((w_i, w_i + 2^{-i} \cdot \epsilon^*], \by, k)$.
\begin{fact}
Let $v_t \asseq M_t((w_i, w_i + 2^{-i} \cdot \epsilon^*], \by, k) - M_{t-1}((w_i, w_i + 2^{-i} \cdot \epsilon^*], \by, k)$ denote the $t$-th term of the martingale.
We have that
\begin{equation*}
\begin{aligned}
\sum_{s=1}^T \EEc{|v_s|^2}{x_{1:(s-1)}}
& \leq \epsilon^* 2^{k-i} / c^2.
\end{aligned}
\end{equation*}
\end{fact}
\begin{proof}
Because the prior $\cP$ is smooth, the probability that the conditional probability lies in any short interval is small. More specifically,
\begin{equation*}
\pr{}{\pstar_s \in (w_i, w_i + 2^{-i} \cdot \epsilon^*] \mid x_{1:(s-1)}} \leq 2^{-i} \epsilon^* / c
\end{equation*}
holds for any $x_{1:(s-1)} \in \{0, 1\}^{s-1}$.
Also, let $t'$ denote the first timestep of the $k$-th epoch and observe that the event $\1{t' = t}$ is measurable with $x_{1:t-1}$.
Further recalling that the maximum length of each epoch is $\ceil{2^{k-1} / c}$, we then have
\begin{equation*}
\begin{aligned}
&~   \sum_{s=1}^T  \EEsc{x \sim \cP}{|v_s|^2}{x_{1:(s-1)}} \\
=   &~\sum_{s=1}^T \EEsc{(x, \pstar) \sim \cP}{y_s \cdot \1{\pstar_s \in (w_i, w_i + 2^{-i} \cdot \epsilon^*]} \cdot (x_s - \pstar_s)^2 \cdot \1{s \in I_k}}{x_{1:(s-1)}} \\
\leq   &~\sum_{s=1}^T \EEsc{(x, \pstar) \sim \cP}{  \1{\pstar_s \in (w_i, w_i + 2^{-i} \cdot \epsilon^*]} \cdot \1{s \in I_k}}{x_{1:(s-1)}} \\
\leq  &~ %
\sum_{s=1}^T\EEs{x_{1:(s-1)}}{
\1{t' = s} 
\sum_{\tau=0}^{\ceil{2^{k-1} / c} - 1} \EEsc{(x, \pstar) \sim \cP}{ \1{\pstar_{s+\tau} \in (w_i, w_i + 2^{-i} \cdot \epsilon^*]}}{x_{1:(s-1)}}} \\
\leq  &~ %
\sum_{s=1}^T\EEs{x_{1:(s-1)}}{\1{t' = s} 
\sum_{\tau=0}^{\ceil{2^{k-1} / c} - 1} 2^{-i} \epsilon^* / c} \\
\le &~\ceil{2^{k-1}/c}\cdot 2^{-i}\epsilon^*/c\\
\le &~2^k/c \cdot 2^{-i}\epsilon^*/c
=   \epsilon^* 2^{k-i} / c^2.
\end{aligned}
\end{equation*}
\end{proof}
Freedman's inequality gives for $p \in (0, 1)$, with probability at least $1 - p$,
\begin{equation*}
\abs{M_t((w_i, w_i + 2^{-i} \cdot \epsilon^*], \by, k)} \leq \sqrt{2 (\epsilon^* \cdot 2^{k-i} / c^2) \ln(2/p)} + 2 \ln(2/p).
\end{equation*}
Taking a union bound on $V_{2^{-i} \cdot \epsilon^*}$, with probability at least $1 - p$,
\begin{equation*}
\begin{aligned}
& \max_{w_i \in V_{2^{-i} \cdot \epsilon^*}} \abs{M_t((w_i, w_i + 2^{-i} \cdot \epsilon^*], \by, k)} \\
\leq \; & \sqrt{(\epsilon^* \cdot 2^{k-i+1} / c^2) \ln\left(4\cdot (2^i/\eps^*)/p\right)} + 2 \ln\left(4\cdot (2^i/\eps^*)/p\right).
\end{aligned}
\end{equation*}
We then take the layer cake representation as before
\begin{equation*}
\begin{aligned}
& \EEc{\max_{w_i \in V_{2^{-i} \cdot \epsilon^*}} \abs{M_t((w_i, w_i + 2^{-i} \cdot \epsilon^*], \by, k)}}{\tau_{k-1} < \infty} \\
\leq \; & \int_0^1 \sqrt{(\epsilon^* \cdot 2^{k-i+1} / c^2) \ln(2^{i+2}/\epsilon^*)} + \sqrt{(\epsilon^* \cdot 2^{k-i+1} / c^2) \ln(1/p)} + 2 \ln\left(4\cdot (2^i/\eps^*)/p\right) \; \mathrm{d} p \\
\leq \; & \sqrt{(\epsilon^* \cdot 2^{k-i+1} / c^2) \ln(2^{i+2}/\epsilon^*)}  + 2 \ln\left(2^{i+2}/\epsilon^*\right) + \int_0^1 \sqrt{(\epsilon^* \cdot 2^{k-i+1} / c^2) \ln (1/p)} + 2 \ln(1/p) \; \mathrm{d}p \\
= \; &  \sqrt{(\epsilon^* \cdot 2^{k-i+1} / c^2) \ln(2^{i+2}/\epsilon^*)}  + 2 \ln\left(2^{i+2}/\epsilon^*\right) +  \sqrt{\pi \cdot \epsilon^* \cdot 2^{k-i-1} / c^2} + 2.
\end{aligned}
\end{equation*}
\end{proof}

\suffixbound*
\begin{proof}
    Let $X \coloneqq \max_{w_m \in V_{2^{-m}\cdot\epsilon^*}}\sum_{t=1}^{T} \1{\pstar_t \in (w_m, w_m + 2^{-m} \cdot \epsilon^*]}$ denote the random variable of interest. Since $\cP$ is $c$-smoothed, for each fixed $w_m \in V_{2^{-m}\epsilon^*}$ and every $t \in [T]$, it holds that
    \[
        \pr{}{\pstar_t \in (w_m, w_m + 2^{-m}  \epsilon^*} \mid x_{1:(t-1)}] \leq 2^{-m} \epsilon^* / c.
    \]
    Therefore,
    \[
        \sum_{t=1}^{T}\1{\pstar_t \in (w_m, w_m + 2^{-m}\eps^*]}
    \]
    is stochastically dominated by a binomial random variable that follows $\Binomial(T, 2^{-m}\eps^*/c)$. By the multiplicative Chernoff bound, we have
    \[
        \pr{}{\sum_{t=1}^{T}\1{\pstar_t \in (w_m, w_m + 2^{-m}\eps^*]} \ge (1 + \delta)\mu} \le \exp\left(-\frac{\delta^2}{2 + \delta}\mu\right)
    \le \exp\left(-\frac{\delta\mu}{3}\right),
    \]
    where $\delta \ge 1$ and $\mu = T\cdot \frac{2^{-m}\eps^*}{c}$.

    By the union bound, for every $\delta \ge 1$, we have
    \begin{equation}\label{eq:suffix-tail-bound}
        \pr{}{X \ge (1 + \delta)\mu}
    \le |V_{2^{-m}\cdot\eps^*}|\cdot\exp\left(-\frac{\delta\mu}{3}\right)
    \le \lceil 2^m/\eps^*\rceil\cdot\exp\left(-\frac{\delta\mu}{3}\right).
    \end{equation}
    Therefore, we have
    \begin{align*}
        \Ex{}{X}
    &=  \int_{0}^{+\infty}\pr{}{X \ge \tau}~\rmd\tau\\
    &\le 2\mu + \int_{2\mu}^{+\infty}\pr{}{X \ge \tau}~\rmd\tau\\
    &\le 2\mu + \mu\int_{1}^{+\infty}\pr{}{X \ge (1 + \tau)\mu}~\rmd{\tau}.
    \end{align*}
    Plugging \Cref{eq:suffix-tail-bound} into the integral above gives
    \[
        \int_{1}^{+\infty}\pr{}{X \ge (1 + \tau)\mu}~\rmd\tau
    \le \int_{0}^{+\infty}\min\{\lceil2^m/\eps^*\rceil\cdot e^{-\mu\tau/3}, 1\}~\rmd\tau
    =   \frac{\ln\lceil2^m/\eps^*\rceil + 1}{\mu/3},
    \]
    where the second step applies the identity
    \[
        \int_{0}^{+\infty}\min\{ne^{-ax}, 1\}~\rmd x
    =   \int_{0}^{(\ln n)/a}1~\rmd x + \int_{(\ln n)/a}^{+\infty}ne^{-ax}~\rmd x
    =   \frac{\ln n}{a} + \frac{n}{a}\cdot e^{-(a\ln n) / a}
    =   \frac{\ln n + 1}{a}.
    \]
    Therefore, we conclude that
    \[
        \Ex{}{X}
    \le 2\mu + 3(\ln\lceil2^m/\eps^*\rceil + 1)
    =   2T\cdot\frac{\eps^*}{2^m c} + 3(\ln\lceil2^m/\eps^*\rceil + 1)
    \le 2T\cdot\frac{\eps^*}{2^m c} + 3(\ln(2^{m+1}/\eps^*) + 1).
    \]
\end{proof}

\stepcalibepoch*
\begin{proof}
By our choice of $\eps^* = c^2$ and $m = \floor{\log_2 T}$, we have $1/\eps^* \le 1/c^2$, $2^m \le T < 2^{m+1}$.
Filling these into Lemma~\ref{lemma:prefix_bound}, we have
\begin{equation}
\label{eq:tt1}
\begin{aligned}
    &~\EEc{\max_{w_0 \in V_{\epsilon^*}} \abs{M_t(w_0, \by, k)}}{\tau_{k-1} < \infty}\\
\le &~\sqrt{2^{k+1} \ln(4/c^2)} + \sqrt{2^{k-1} \pi} + 2 \ln(4/c^2)  + 2\\
\le &~2\sqrt{2^k\ln(1/c)} + \sqrt{2\ln 4}\cdot\sqrt{2^k} + \sqrt{\pi/2}\cdot\sqrt{2^k} + 4\ln(1/c) + (2\ln 4 + 2)\\
\le &~2\sqrt{2^k\ln(1/c)} + 3\sqrt{2^k} + 4\ln(1/c) + 5.
\end{aligned}
\end{equation}
The bound in \Cref{lemma:suffix_bound} reduces to
\begin{equation}
\label{eq:tt2}
\begin{aligned}
    \Ex{}{\max_{w_m \in V_{2^{-m}\cdot\epsilon^*}}\sum_{t=1}^{T} \1{\pstar_t \in (w_m, w_m + 2^{-m} \cdot \epsilon^*]}}
&\le 2T\cdot \frac{c^2}{2^m c} + 3 \left(\ln(2^{m+1}/c^2) + 1\right)\\
&\le 4c + 3 \ln(2T/c^2) + 3\\
&\le 10 + 6\ln(T/c).
\end{aligned}
\end{equation}
Similarly, plugging these constants into Lemma~\ref{lemma:segment_bound} and summing over $i \in [m]$ gives
    \begin{equation}
     \label{eq:tt3}
\begin{aligned}
    &~\sum_{i=1}^m \EEc{\max_{w_i \in V_{2^{-i} \cdot \epsilon^*}} \abs{M_t((w_i, w_i + 2^{-i} \cdot \epsilon^*], \by, k)}}{\tau_{k-1} < \infty}\\
\leq&~\sum_{i=1}^m  \left(\sqrt{2^{k-i+1} \ln(2^{i+2} / c^2)}  + \sqrt{2^{k-i-1} \pi}  + 2 \ln(2^{i+2} / c^2) + 2\right).
\end{aligned}
\end{equation}
To further simplify the right-hand side of \Cref{eq:tt3}, we note that
\begin{equation*}
\begin{aligned}
\sum_{i=1}^m  \sqrt{2^{k-i+1} \ln(2^{i+2} / c^2)}
&=  \sum_{i=1}^m  \sqrt{2^k\ln(4/c^2)\cdot2^{1-i} + 2^k\cdot2^{1-i}\ln2^i}\\
&\le \sqrt{2^k\ln(4/c^2)}\cdot\sum_{i=1}^{+\infty}\sqrt{2^{1-i}} + \sqrt{2^k}\cdot\sum_{i=1}^{+\infty}\sqrt{2^{1-i}\ln 2^i}\\
&\le (2 + \sqrt{2})\sqrt{2^k\ln(4/c^2)} + 5\sqrt{2^k},
\end{aligned}
\end{equation*}
where the last inequality uses the geometric series $\sum_{n=1}^\infty \sqrt{2^{1-n}} = 2 + \sqrt 2$ and $\sum_{n=1}^\infty \sqrt{2^{1-n} \ln 2^n} \approx 4.88 < 5$.
We can similarly bound
\[
    \sum_{i=1}^m \sqrt{2^{k-i-1} \pi}
 \leq (1 + 1/\sqrt 2)\sqrt{\pi 2^k}
\]
and
\[
    \sum_{i=1}^m2 \ln(2^{i+2} / c^2)
=   2m\ln(4/c^2) + m(m+1)\ln 2
\]
using the arithmetic series $\sum_{n=1}^m n = \tfrac{m (m+1)}{2}$.
Putting these together, we have the following simplified expression for \Cref{eq:tt3}:
\begin{equation}
\label{eq:tt4}
\begin{aligned}
    &~\sum_{i=1}^m \EEc{\max_{w_i \in V_{2^{-i} \cdot \epsilon^*}} \abs{M_t((w_i, w_i + 2^{-i} \cdot \epsilon^*], \by, k)}}{\tau_{k-1} < \infty} \\
\leq &~(2 + \sqrt 2) \sqrt{2^k \ln(4/c^2)} + (5 + (1+1/\sqrt{2})\sqrt{\pi}) \cdot \sqrt{2^k} + 2m \ln(4/c^2) + m (m+1) \ln 2 + 2m\\
\le &~4\sqrt{2^k \ln(4/c^2)} + 9\sqrt{2^k} + 2m \ln(4/c^2) + 4m^2.
\end{aligned}
\end{equation}

We can then sum \Cref{eq:tt1}, \Cref{eq:tt2} and \Cref{eq:tt4} to obtain the following simplification of \Cref{eq:divide}:
\begin{equation*}
    \begin{aligned}
        &~\EEc{\sup_{\alpha \in [0, 1]} \abs{M_t(\alpha, \by, k)}}{\tau_{k-1} < \infty} \\
    \leq&~\EEc{\max_{w_0 \in V_{\epsilon^*}} \abs{M_t(w_0, \by, k)}}{\tau_{k-1} < \infty} + \sum_{i=1}^m \EEc{\max_{w_i \in V_{2^{-i} \cdot \epsilon^*}} \abs{M_t((w_i, w_i + 2^{-i} \cdot \epsilon^*], \by, k)}}{\tau_{k-1} < \infty}  \\
        &+ \Ex{}{\max_{w_m \in V_{2^{-m}\cdot\epsilon^*}}\sum_{t=1}^{T} \1{\pstar_t \in (w_m, w_m + 2^{-m} \cdot \epsilon^*]}} \\
    \le &~\left(2\sqrt{2^k\ln(1/c)} + 3\sqrt{2^k} + 4\ln(1/c) + 5\right) + \left(4\sqrt{2^k \ln(4/c^2)} + 9\sqrt{2^k} + 2m \ln(4/c^2) + 4m^2\right)\\
    &+ (10 + 6\ln(T/c))\\
    \le &~8\sqrt{2^k\ln(1/c)} + 17\sqrt{2^k} + 8m^2 + 4m\ln(1/c) + 6\ln(T/c) + 4\ln(1/c) + 15.
    \end{aligned}
\end{equation*}
The last step above applies
\[
    \sqrt{2^k \ln(4/c^2)}
=   \sqrt{2^k\ln 4 + 2\cdot2^k\ln(1/c)}
\le \sqrt{2}\cdot\sqrt{2^k\ln(1/c)} + \sqrt{\ln 4}\cdot\sqrt{2^k}
\]
and
\[
    2m\ln(4/c^2)
=   4m\ln2 + 4m\ln(1/c)
\le 4m^2 + 4m\ln(1/c).
\]
Finally, the lemma follows from
\[
    8m^2 + 6\ln(T/c) \le 8(\log_2 T)^2 + 6\log_2(T)\log_2(1/c)
    \le 8\log_2(T)\log_2(T/c)
\]
and
\[
    4m\ln(1/c) + 4\ln(1/c)
\le 8m\ln(1/c) \le 8\log_2(T)\log_2(T/c).
\]
\end{proof}

\subsection{$(O(1), 0)$-Truthfulness on Product Distributions}
For product distributions, the subsampled step calibration error, $\stepCEsub$, enjoys a stronger $(O(1), 0)$-truthfulness guarantee, even in the non-smoothed setting. Recall from \Cref{prop:simplesubsamptruth} that the subsampled U-Calibration error, in contrast, can have an $e^{-\Omega(T)}$-$O(\sqrt{T})$ truthfulness gap on product distributions.

\begin{proposition}\label{prop:stepCEsub-truthful-product}
    On every product distribution $\cD$ over $\{0, 1\}^T$,  it holds that
    \[
        \err_{\stepCEsub}(\cD, \Atruthful(\cD))
    \le O(\OPT_{\stepCEsub}(\cD)),
    \]
    where the $O(\cdot)$ notation hides a universal constant that does not depend on $\cD$.
\end{proposition}

\begin{proof}
    Suppose that $\cD = \prod_{t=1}^{T}\Bern(\pstar_t)$ for some $\pstar \in [0, 1]^T$. By \Cref{lemma:lower}, we have
    \[
        \OPT_{\stepCEsub}(\cD) = \Omega(\gamma(\Var_T)),
    \]
    where $\Var_T \coloneqq \sum_{t=1}^{T}\pstar_t(1-\pstar_t)$ and $\gamma(x) \coloneqq x \cdot \1{x \le 1} + \sqrt{x} \cdot \1{x > 1}$.

    It remains to show that the truthful forecaster $\Atruthful(\cD)$, which predicts $p_t = \pstar_t$ at every step $t$, satisfies
    \[
        \err_{\stepCEsub}(\cD, \Atruthful(\cD))
    =   \Ex{x \sim \cD}{\stepCEsub(x, \pstar)}
    =   O(\gamma(\Var_T)).
    \]
    We consider the following two cases, depending on whether $\Var_T$ is below or above $1$.

    \paragraph{Case 1: $\Var_T \le 1$.} We note that, for every $y \in \{0, 1\}^T$ and $\alpha \in [0, 1]$,
    \[
        \abs{\sum_{t=1}^{T}y_t\cdot(x_t - p_t)\cdot\1{p_t \in [0, \alpha]}}
    \le \sum_{t=1}^{T}|x_t - p_t|.
    \]
    It follows that $\stepCEsub(x, p) \le \sum_{t=1}^{T}|x_t - p_t|$, and
    \begin{align*}
        \Ex{x \sim \cD}{\stepCEsub(x, \pstar)}
    \le \sum_{t=1}^{T}\Ex{x_t \sim \Bern(\pstar_t)}{|x_t - \pstar_t|}
    =   \sum_{t=1}^{T}2\pstar_t(1-\pstar_t)
    =   2\gamma(\Var_T).
    \end{align*}

    \paragraph{Case 2: $\Var_T > 1$.} Without loss of generality, we assume that $\pstar_1 \le \pstar_2 \le \cdots \le \pstar_T$, as the behavior of the truthful forecaster and the resulting $\stepCEsub$ are invariant up to the reordering of timesteps. For fixed $x, y \in \{0, 1\}^T$, we have
    \[
        \sup_{\alpha \in [0, 1]}\abs{\sum_{t=1}^{T}y_t\cdot(x_t - \pstar_t)\cdot\1{\pstar_t \in [0, \alpha]}}
    \le \max_{t \in [T]}\abs{\sum_{i=1}^{t}y_i\cdot(x_i - \pstar_i)}.
    \]
    Taking an expectation over $x \sim \cD$ and $y \sim \Unif(\{0, 1\}^T)$ shows that
    \[
        \Ex{x \sim \cD}{\stepCEsub(x, \pstar_t)}
    \le \Ex{x \sim \cD, y \sim \Unif(\{0, 1\}^T)}{\max_{t \in [T]}\left|\sum_{i=1}^{t}y_i\cdot(x_i - \pstar_i)\right|}.
    \]
    Over the randomness in $(x, y)$, consider the random walk $(X_t)_{t=0}^{T}$ defined as $X_t \coloneqq \sum_{i=1}^{t}y_i(x_i - \pstar_i)$. Then, $(X_t)_{t=0}^{T}$ forms a martingale in which the increment at time $t$ has a variance of $\pstar_t(1 - \pstar_t) / 2$. The total variance is then $\Ex{}{X_T^2} = \sum_{t=1}^{T}\pstar_t(1 - \pstar_t) / 2 = \Var_T / 2$. Then, by Kolmogorov's inequality,
    \[
        \pr{}{\max_{t \in [T]}|X_t| \ge \tau} \le \frac{\Ex{}{X_T^2}}{\tau^2}
    =   \frac{\Var_T}{2\tau^2}
    \]
    holds for every $\tau > 0$. It follows that
    \begin{align*}
        \Ex{}{\max_{t \in [T]}|X_t|}
    &=   \int_{0}^{+\infty}\pr{}{\max_{t \in [T]}|X_t| \ge \tau}~\rmd\tau\\
    &\le \int_{0}^{+\infty}\min\left\{\frac{\Var_T}{2\tau^2}, 1\right\}~\rmd\tau
    =   O(\sqrt{\Var_T})
    =   O(\gamma(\Var_T)).
    \end{align*}
\end{proof}

\subsection{An $O(\sqrt T)$ Step Calibration Error Algorithm}
\begin{theorem}
\label{thm:alg}
    Algorithm~\ref{eq:alg} guarantees an expected step calibration error of $O(\sqrt{T \log T})$, even when the $T$ events are adversarially and adaptively chosen. Moreover, the forecasts made by the algorithm each randomize over at most two probabilities.
\end{theorem}

By the inequality $\stepCEsub(x, p) \le \tfrac{1}{2}\stepCE(x, p) + O(\sqrt{T})$ (\Cref{lemma:stepCE-vs-stepCEsub}), the same algorithm guarantees an $O(\sqrt{T\log T})$ expected error with respect to the subsampled version $\stepCEsub$ as well.

\begin{algorithm}[H]
\caption{Prediction Algorithm with Hedge}
\label{eq:alg}
\begin{algorithmic}[1]
\Require Number of buckets $k \geq 2$, time horizon $T$
\State Initialize weight $w_1$ to be uniform over $\{\pm 1\} \times [k]$

\For{$t = 1$ to $T$}
    \If{$\EEs{(\sigma, i) \sim w_t}{\sigma \cdot \1{\tfrac{j-1}{k-1} \leq \tfrac{i-1}{k-1}}} \geq 0$ for all $j \in [k]$}
        \State Predict $p_t = 1$
    \ElsIf{$\EEs{(\sigma, i) \sim w_t}{\sigma \cdot \1{\tfrac{j-1}{k-1} \leq \tfrac{i-1}{k-1}}} \leq 0$ for all $j \in [k]$}
        \State Predict $p_t = 0$
    \Else
        \State Find $j \in [k-1]$ such that
        \[
            \EEs{(\sigma, i) \sim w_t}{\sigma \cdot \1{\tfrac{j-1}{k-1} \leq \tfrac{i-1}{k-1}}} \cdot \EEs{(\sigma, i) \sim w_t}{\sigma \cdot \1{\tfrac{j}{k-1} \leq \tfrac{i-1}{k-1}}} \le 0
        \]
        \State Find $q \in [0, 1]$ such that
        \[
            q \cdot \EEs{(\sigma, i) \sim w_t}{\sigma \cdot \1{\tfrac{j-1}{k-1} \leq \tfrac{i-1}{k-1}}} + (1-q) \cdot \EEs{(\sigma, i) \sim w_t}{\sigma \cdot \1{\tfrac{j}{k-1} \leq \tfrac{i-1}{k-1}}} = 0
        \]
        \State Predict $p_t = \tfrac{j-1}{k-1}$ with probability $q$ and $p_t = \tfrac{j}{k-1}$ with probability $1-q$
    \EndIf
    \State Observe $x_t \in \{0, 1\}$ from Nature
    \State Compute $w_{t+1}$ by applying the Hedge algorithm to the cost functions $(c_{x_\tau, p_\tau})_{\tau \in [t]}$, where
    \[
    c_{x_\tau, p_\tau}(\sigma, i) =
    1 - \tfrac{1}{2}\sigma \cdot \1{p_\tau \leq \tfrac{i-1}{k-1}} \cdot (x_\tau - p_\tau).
    \]
\EndFor
\end{algorithmic}
\end{algorithm}

\begin{proof}
Given sign $\sigma \in \{\pm 1\}$ and bucket $i \in [k]$, we define $\loss_{\sigma, i}: \{0, 1\} \times [0, 1] \to [-1, 1]$ to map from a realization $x \in \{0, 1\}$  and prediction $p \in [0, 1]$ to a loss \[\loss_{\sigma, i}(x, p) = \sigma \cdot \1{p \leq \tfrac{i-1}{k-1}} \cdot (x - p).\]
Then, on events $x_{1:T} \in \{0, 1\}^T$ and predictions $p_{1:T} \in \left\{\tfrac{i-1}{k-1}: i \in [k]\right\}^T$, the step calibration error can be written as
\[
    \stepCE(x_{1:T}, p_{1:T}) = \max_{\sigma^* \in \{\pm 1\}, i^* \in [k]} \sum_{t=1}^T \loss_{\sigma^*, i^*}(x_t, p_t).
\]
Furthermore, the cost function $c_{x, p}:  \{\pm 1\} \times [k] \mapsto [0, 1]$ used by the Hedge algorithm in \Cref{eq:alg} is simply
\[
    c_{x, p}(\sigma, i) = 1 - \tfrac 12 \cdot \loss_{\sigma, i}(x, p).
\]

\paragraph{Guarantees of Hedge.}
Consider the sequence of distributions $w_{1:T} \in \Delta(\{\pm 1\}\times [k])$ given by applying Hedge to the cost functions $(c_{x_t, p_t})_{t \in [T]}$:
\[w_t = \mathrm{Hedge}(c_{x_1, p_1}, c_{x_2, p_2}, \dots, c_{x_{t-1}, p_{t-1}}).\]
Hedge guarantees that, even though $(x_t, p_t)$ is allowed to depend on $w_{1:t}$ at each step $t \in [T]$, we still have
\[
    \min_{\sigma^* \in \{\pm 1\}, i^* \in [k]}\sum_{t=1}^T c_{x_t, p_t}(\sigma^*, i^*)
\ge \sum_{t=1}^T \EEs{(\sigma, i) \sim w_t}{c_{x_t, p_t}(\sigma, i)} - O\left(\sqrt {T \log k}\right).
\]
Equivalently,
\[
    \max_{\sigma^* \in \{\pm 1\}, i^* \in [k]}\sum_{t=1}^T \loss_{\sigma^*, i^*}(x_t, p_t)
\le \sum_{t=1}^T \EEs{(\sigma, i) \sim w_t}{\loss_{\sigma, i}(x_t, p_t)} + O\left(\sqrt {T \log k}\right).
\]
Therefore, to upper bound $\stepCE(x, p)$, it remains to control the loss $\EEs{(\sigma, i) \sim w_t}{\loss_{\sigma, i}(x_t, p_t)}$ at each step.

\paragraph{Control the per-step loss.} Fix a timestep $t \in [T]$. We condition on the realization of $x_{1:(t-1)}$ and $p_{1:(t-1)}$, and shorthand $w$ for $w_t \in \Delta(\{\pm 1\} \times [k])$. For a fixed prediction distribution $\bp \in \Delta(\{\tfrac{i-1}{k-1}: i \in [k]\})$, we can write
\begin{align*}
\max_{x \in \{0, 1\}} \EEs{p \sim \bp}{\EEs{(\sigma, i) \sim w}{\loss_{\sigma, i}(x, p)}}
&= \max_{x \in \{0, 1\}} \EEs{(\sigma, i) \sim w}{\sigma \cdot \EEs{p \sim \bp}{\1{p \leq \tfrac{i-1}{k-1}} \cdot (x - p)}} \\
&=  \max_{x \in \{0, 1\}}\EEs{p \sim \bp}{x \cdot C_{p}} - \EEs{p \sim \bp}{p \cdot C_{p}}\\
&= \max\left\{\EEs{p \sim \bp}{C_{ p}}, 0\right\} - \EEs{p \sim \bp}{p \cdot C_{ p}},
\end{align*}
where $C_{ p} \coloneqq \EEs{(\sigma, i) \sim w}{\sigma \cdot \1{p \leq \tfrac{i-1}{k-1}}}$.

Suppose that $C_{ p} \geq 0$ holds for all $p \in \left\{\tfrac{i-1}{k-1}: i \in [k]\right\}$. Then, we can let $\bp$ be the degenerate distribution at $1$ and get
\[
    \max_{x \in \{0, 1\}} \EEs{p \sim \bp}{\EEs{(\sigma, i) \sim w}{\loss_{\sigma, i}(x, p)}}
=   \max\{C_{1}, 0\} - 1 \cdot C_{1}
=   0.
\]
Similarly, if $C_{ p} \leq 0$ holds for every $p \in \left\{\tfrac{i-1}{k-1}: i \in [k]\right\}$, we can set $\bp$ to be deterministically $0$ and get 
\[
    \max_{x \in \{0, 1\}} \EEs{p \sim \bp}{\EEs{(\sigma, i) \sim w}{\loss_{\sigma, i}(x, p)}}
=   \max\left\{C_{0}, 0\right\} - 0 \cdot C_{ 0}
=   0.
\]

If neither holds, there must exist $p_1, p_2 \in \left\{\tfrac{i-1}{k-1}: i \in [k]\right\}$ such that $p_2 - p_1 = \tfrac{1}{k-1}$ and $C_{ p_1}  \cdot C_{ p_2} \le 0$. Then, there also exists $q \in [0, 1]$ such that
\[
    q \cdot C_{ p_1} + (1-q) \cdot C_{p_2} = 0. 
\]
We accordingly let $\bp$ take value $p_1$ with probability $q$ and take value $p_2$ with probability $1-q$. This choice ensures $\EEs{p \sim \bp}{C_{ p}} = q \cdot C_{ p_1} + (1-q) \cdot C_{p_2} = 0$, which further implies
\begin{align*}
    \max_{x \in \{0, 1\}} \EEs{p \sim \bp}{\EEs{(\sigma, i) \sim w}{\loss_{\sigma, i}(x, p)}}
&=  \max\left\{\EEs{p \sim \bp}{C_{ p}}, 0\right\} - \EEs{p \sim \bp}{p \cdot C_{ p}}\\
&=  0 - q\cdot (p_1 C_{ p_1}) - (1-q) \cdot (p_2 C_{p_2})\\
&=  - p_1 (q C_{ p_1} + (1-q) C_{ p_2}) - (p_2 - p_1)(1-q) C_{ p_2} \\
&=  -p_1 \cdot \EEs{p \sim \bp}{C_{p}} - (p_2 - p_1)(1-q)C_{p_2}\\
& \leq \frac{1}{k-1},
\end{align*}
with the last inequality following from $\EEs{p \sim \bp}{C_{p}} = 0$, $p_2 - p_1 = \tfrac 1 {k-1}$ and $|C_{ p_2}| \leq 1$.

Note that our construction of $\bp$ coincides with the random choice of $p_t$ at each timestep $t$ in \Cref{eq:alg}. Therefore, it holds for every $t \in [T]$ that
\[
    \EEs{(\sigma, i) \sim w_t}{\loss_{\sigma, i}(x_t, p_t)}
=   \EEs{w_t}{\EEsc{(\sigma, i) \sim w_t}{\loss_{\sigma, i}(x_t, p_t)}{w_t}}
\le \EEs{w_t}{\max_{x' \in \{0, 1\}}\EEs{\substack{(\sigma, i) \sim w_t\\p_t \sim \bp_t}}{\loss_{\sigma, i}(x', p_t)}}
\le \frac{1}{k-1}.
\]
We conclude that
\begin{align*}
    \EE{\stepCE(x, p)}
&=  \EE{\max_{\sigma^* \in \{\pm 1\}, i^* \in [k]}
\sum_{t=1}^T \loss_{\sigma^*, i^*}(x_t, p_t)} \\
&\le\sum_{t=1}^T \EEs{(\sigma, i) \sim w_t}{\loss_{\sigma, i}(x_t, p_t)} + O(\sqrt {T \log k}) \\
&\le\frac T {k-1} + O(\sqrt{T \log k}).
\end{align*}
Choosing $k = T$ gives the $O(\sqrt{T \log T})$ bound.
\end{proof}

\section*{Acknowledgments}
This work is supported by the Google Research Fellowship and the National Science Foundation Graduate Research Fellowship Program under Grant No. DGE 2146752.
The authors thank Kunhe Yang for valuable discussions on the project. 

\bibliographystyle{alpha}
\newcommand{\etalchar}[1]{$^{#1}$}

\newpage

\tableofcontents

\appendix

\section{Basic Facts}

In this section, we prove the equivalent formulation of the V-calibration error~\cite{KLST23} (\Cref{prop:altform}), establish the decision-theoretic interpretation of step calibration (\Cref{fact:equivfactor}), and show that $\stepCEsub$ is close to $\stepCE$ in general (\Cref{lemma:stepCE-vs-stepCEsub}).

\subsection{Proof of Proposition~\ref{prop:altform}}\label{sec:altform-proof}
\altform*

\begin{proof}
Recall that $S_\alpha(x, p) \coloneq (\alpha - x)\cdot \sgn(p - \alpha)$ and that the V-Calibration error is given by
\begin{align*}
    \VCal(x, p)
&=  \sup_{\alpha, \beta \in [0, 1]}\left[\sum_{t=1}^{T}S_\alpha(x_t, p_t) - \sum_{t=1}^{T}S_\alpha(x_t, \beta)\right]\\
&=  \sup_{\alpha \in [0, 1]}\left[\sum_{t=1}^{T}S_\alpha(x_t, p_t) - \sum_{t=1}^{T}S_\alpha(x_t, \mu)\right]\\
&=  \sup_{\alpha \in [0, 1]}\left[\sum_{t=1}^{T}(x_t - \alpha)\cdot\sgn(\alpha - p_t) - \sum_{t=1}^{T}(x_t - \alpha)\cdot\sgn(\alpha - \mu)\right],
\end{align*}
where the optimal choice of $\beta$ is always $\beta = \mu \coloneq \tfrac 1T \sum x_t$, since $S_{\alpha}$ is proper.

For $\alpha \in [0, 1]$, we introduce the shorthands
\[
    f(\alpha) \coloneqq \sum_{t=1}^{T}(x_t - \alpha)\cdot\sgn(\alpha - p_t) - \sum_{t=1}^{T}(x_t - \alpha)\cdot\sgn(\alpha - \mu)
\]
and
\[
    g(\alpha) \coloneqq \max\left\{X_{-}^{(\alpha)} - \alpha N_{-}^{(\alpha)}, \alpha N_{+}^{(\alpha)} - X_{+}^{(\alpha)}\right\}.
\]
Towards proving $\sup_{\alpha \in [0, 1]}f(\alpha) = 2\sup_{\alpha \in [0, 1]}g(\alpha)$, we will first show that, for $S \coloneqq \{p_1, p_2, \ldots, p_T, \mu\}$,
\[
    \sup_{\alpha \in [0, 1]}f(\alpha)
=   \sup_{\alpha \in [0, 1] \setminus S}f(\alpha) 
\quad\text{and}\quad
    \sup_{\alpha \in [0, 1]}g(\alpha)
=   \sup_{\alpha \in [0, 1] \setminus S}g(\alpha).
\]
In other words, ignoring the case that $\alpha$ coincides with a prediction $p_t$ or the overall average $\mu$ does not change either supremum. Then, we will show that, for every $\alpha \in [0, 1] \setminus S$, we indeed have $f(\alpha) = 2g(\alpha)$. The desired identity would then follow from
\[
    \sup_{\alpha \in [0, 1]}f(\alpha)
=   \sup_{\alpha \in [0, 1] \setminus S}f(\alpha)
=   2\sup_{\alpha \in [0, 1] \setminus S}g(\alpha)
=   2\sup_{\alpha \in [0, 1]}g(\alpha).
\]

\paragraph{Ignore atypical values for $f$.}
First, we focus on the value of $f(\alpha_0)$ for some $\alpha_0 \in S = \{p_1, p_2, \ldots, p_T, \mu\}$. We claim that
\begin{equation}\label{eq:f-average-of-two-limits}
    f(\alpha_0) = \frac{1}{2}\left[\lim_{\alpha \to \alpha_0^{-}}f(\alpha) + \lim_{\alpha \to \alpha_0^{+}}f(\alpha)\right],
\end{equation}
i.e., $f(\alpha_0)$ is equal to the average of the left-sided and right-sided limits of $f$ at $\alpha_0$. Assuming \Cref{eq:f-average-of-two-limits}, if $\alpha_0 \in (0, 1)$, when $\alpha$ approaches $\alpha_0$ from one of the two sides, the limit of $f(\alpha)$ is at least $f(\alpha_0)$. Then, excluding $\alpha_0$ does not decrease the supremum of $f(\alpha)$.

It remains to deal with the case that $\alpha_0 \in \{0, 1\}$. When $\alpha_0 = 0$, we still have \Cref{eq:f-average-of-two-limits}; the issue is that the left-sided limit (as $\alpha \to 0^{-}$) does not contribute to the supremum $\sup_{\alpha \in [0, 1]}f(\alpha)$. The previous argument would still go through if we could further show that
\begin{equation}\label{eq:f-alpha=0}
    \lim_{\alpha \to 0^{-}}f(\alpha) \le \lim_{\alpha \to 0^{+}}f(\alpha),
\end{equation}
since \eqref{eq:f-average-of-two-limits}~and~\eqref{eq:f-alpha=0} together imply that $\lim_{\alpha \to 0^{+}}f(\alpha) \ge f(0)$, so that we can safely ignore the $\alpha_0 = 0$ case.  A symmetric argument would deal with the $\alpha_0 = 1$ case via showing $\lim_{\alpha \to 1^{-}}f(\alpha) \ge \lim_{\alpha \to 1^{+}}f(\alpha)$.

To verify \Cref{eq:f-average-of-two-limits}, we consider the term $(x_t - \alpha)\cdot\sgn(\alpha - p_t)$ in $f(\alpha)$. When $\alpha_0 \ne p_t$, the term is continuous at $\alpha_0$, and contributes equally to both sides of \eqref{eq:f-average-of-two-limits}. When $\alpha_0 = p_t$, we have
\[
    \lim_{\alpha \to \alpha_0^-}(x_t - \alpha)\cdot\sgn(\alpha - p_t)
=   -(x_t - \alpha)
\quad\text{and}\quad
    \lim_{\alpha \to \alpha_0^+}(x_t - \alpha)\cdot\sgn(\alpha - p_t)
=   x_t - \alpha.
\]
The average of the two limits is exactly $0$, which is equal to $(x_t - \alpha_0) \cdot \sgn(\alpha_0 - p_t)$ since $\alpha_0 = p_t$. By the same token, each term $-(x_t - \alpha)\cdot\sgn(\alpha - \mu)$ also contributes equally to both sides of \eqref{eq:f-average-of-two-limits}.

To verify \Cref{eq:f-alpha=0}, we again note that each term $(x_t - \alpha)\cdot\sgn(\alpha - p_t)$ contributes to both sides equally if $p_t \ne 0$. When $p_t = 0$, we have
\[
    \lim_{\alpha \to 0^-}(x_t - \alpha)\cdot\sgn(\alpha - p_t)
=   -x_t \le 0
\quad\text{and}\quad
    \lim_{\alpha \to 0^+}(x_t - \alpha)\cdot\sgn(\alpha - p_t)
=   x_t \ge 0.
\]
For the term $-(x_t - \alpha)\cdot\sgn(\alpha - \mu)$, again, it suffices to verify the case that $\mu = 0$, where
\[
    \lim_{\alpha \to 0^-}[-(x_t - \alpha)\cdot\sgn(\alpha - \mu)]
=   x_t = 0
\quad\text{and}\quad
    \lim_{\alpha \to 0^+}[-(x_t - \alpha)\cdot\sgn(\alpha - \mu)]
=   x_t = 0.
\]
In the above, we use the fact that $\mu = 0$ implies that $x_t = 0$ for every $t \in [T]$. This proves \Cref{eq:f-alpha=0} and allows us to ignore the $\alpha_0 = 0$ case (and, by symmetry, the $\alpha_0 = 1$ case).

\paragraph{Ignore atypical values for $g$.} Again, we start with the easier case that $\alpha_0 \in S \cap (0, 1)$. In this case, there exists $\delta > 0$ such that: (1) $\alpha_0 - \delta, \alpha_0 + \delta \in [0, 1]$; (2) $[\alpha_0 - \delta, \alpha_0 + \delta] \setminus \{\alpha_0\}$ contains no elements in $S = \{p_1, p_2, \ldots, p_T, \mu\}$. Concretely, we can choose
\[
    \delta = \min\left(\left\{\frac{1}{2}|\alpha_0 - \beta|: \beta \in S \setminus \{\alpha_0\}\right\} \cup \{\alpha_0, 1 - \alpha_0\}\right) > 0.
\]
Then, recalling that $g(\alpha) = \max\left\{X_{-}^{(\alpha)} - \alpha N_{-}^{(\alpha)}, \alpha N_{+}^{(\alpha)} - X_{+}^{(\alpha)}\right\}$, we have
\[
    X_{-}^{(\alpha_0)} - \alpha_0 N_{-}^{(\alpha_0)}
=   X_{-}^{(\alpha_0 - \delta)} - (\alpha_0 - \delta) N_{-}^{(\alpha_0 - \delta)}
\le g(\alpha_0 - \delta).
\]
The first step above holds since our choice of $\delta$ guarantees $p_t \notin [\alpha_0 - \delta, \alpha_0)$ for each $t \in [T]$, which further implies $p_t < \alpha_0 \iff p_t < \alpha_0 - \delta$. By the same token, we have
\[
    \alpha_0 N_{+}^{(\alpha_0)} - X_{+}^{(\alpha_0)}
=   (\alpha_0 + \delta) N_{+}^{(\alpha_0 + \delta)} - X_{+}^{(\alpha_0 + \delta)}
\le g(\alpha_0 + \delta).
\]
Therefore, we have
\[
    g(\alpha_0) \le \max\{g(\alpha_0 - \delta), g(\alpha_0 + \delta)\}
\]
and $\alpha_0 - \delta, \alpha_0 + \delta \in [0, 1] \setminus S$. This shows that ignoring the case that $\alpha_0 \in S \cap (0, 1)$ does not affect the supremum of $g(\alpha)$.

It remains to deal with the case that $\alpha_0 \in \{0, 1\}$. When $\alpha_0 = 0$, we have
\[
    g(0)
=   \max\left\{X_{-}^{(0)} - 0\cdot N_{-}^{(0)}, 0 \cdot N_{+}^{(0)} - X_{+}^{(0)}\right\}
=   0,
\]
since $X_{-}^{(0)} = 0$ and $X_{+}^{(0)} \ge 0$. By symmetry, we also have $g(1) = 0$. Therefore, ignoring these two values of $g$ does not affect the supremum.

\paragraph{Analysis for typical $\alpha$.} It remains to show that $f(\alpha) = 2g(\alpha)$ holds for every $\alpha \in [0, 1] \setminus S$. Note that, in this case, the factors $\sgn(\alpha - p_t)$ and $\sgn(\alpha - \mu)$ do not take value $0$ in $f(\alpha)$. When $\alpha < \mu$, $f(\alpha)$ is given by
\begin{align*}
    f(\alpha)
&=  \sum_{t \in [T]: p_t < \alpha}(x_t - \alpha) - \sum_{t \in [T]: p_t > \alpha}(x_t - \alpha) + \sum_{t=1}^{T}(x_t - \alpha)\\
&=  \left(X_-^{(\alpha)} - \alpha N_-^{(\alpha)}\right) - \left(X_+^{(\alpha)} - \alpha N_+^{(\alpha)}\right) + \left(X_-^{(\alpha)} + X_+^{(\alpha)} - \alpha T\right)\\
&=  2\left(X_-^{(\alpha)} - \alpha N_-^{(\alpha)}\right).
\end{align*}
Furthermore, since $\alpha < \mu$, we have
\[
    \alpha N_{-}^{(\alpha)} + \alpha N_{+}^{(\alpha)}
=   \alpha T
<   \mu T
=   X_-^{(\alpha)} + X_+^{(\alpha)},
\]
which implies $X_-^{(\alpha)} - \alpha N_-^{(\alpha)} > \alpha N_{+}^{(\alpha)} - X_+^{(\alpha)}$. It follows that
\[
    f(\alpha)
=   2\left(X_-^{(\alpha)} - \alpha N_-^{(\alpha)}\right)
=   2\max\left\{X_-^{(\alpha)} - \alpha N_-^{(\alpha)}, \alpha N_{+}^{(\alpha)} - X_+^{(\alpha)}\right\}
=   2g(\alpha).
\]

Similarly, when $\alpha > \mu$, we have
\[
    f(\alpha) = \left(X_-^{(\alpha)} - \alpha N_-^{(\alpha)}\right) - \left(X_+^{(\alpha)} - \alpha N_+^{(\alpha)}\right) - \left(X_-^{(\alpha)} + X_+^{(\alpha)} - \alpha T\right)
=   2\left(\alpha N_+^{(\alpha)} - X_+^{(\alpha)}\right)
\]
and
\[
    \alpha N_+^{(\alpha)} - X_+^{(\alpha)}
>   X_-^{(\alpha)} - \alpha N_-^{(\alpha)},
\]
which also imply $f(\alpha) = 2g(\alpha)$. This shows that $f(\alpha) = 2g(\alpha)$ holds for every $\alpha \in [0, 1] \setminus S$ and completes the proof.
\end{proof}

\subsection{Proof of \Cref{fact:equivfactor}}\label{sec:equivfactor-proof}
\equivfactor*

\begin{proof}
    Fix $x \in \{0, 1\}^T$ and $p \in [0, 1]^T$. Before proving the inequalities, we first show that
    \begin{equation}\label{eq:total-bias-bound}
        \left|\sum_{t=1}^{T}(x_t - p_t)\right| \le \sup_{\alpha \in [0, 1]}\left|\sum_{t=1}^{T}(x_t - p_t)\cdot\sgn(\alpha - p_t)\right|,
    \end{equation}
    i.e., the total bias is upper bounded by the variant of V-Calibration error.

    \paragraph{Upper bound the total bias.} If $\sum_{t=1}^{T}(x_t - p_t) \ge 0$, we have
    \[
        \sum_{t=1}^{T}(x_t - p_t) \cdot\sgn(1 - p_t)
    =   \sum_{t=1}^{T}(x_t - p_t) - \sum_{t=1}^{T}(x_t - p_t) \cdot \1{p_t = 1}
    \ge \sum_{t=1}^{T}(x_t - p_t),
    \]
    where the last step holds since $(x_t - p_t)\cdot\1{p_t = 1}$ takes value $0$ when $p_t \ne 1$, and takes value $x_t - 1 \le 0$ when $p_t = 1$. It follows that
    \[
        \left|\sum_{t=1}^{T}(x_t - p_t)\right|
    =   \sum_{t=1}^{T}(x_t - p_t)
    \le \sum_{t=1}^{T}(x_t - p_t) \cdot\sgn(1 - p_t)
    \le \sup_{\alpha \in [0, 1]}\left|\sum_{t=1}^{T}(x_t - p_t)\cdot\sgn(\alpha - p_t)\right|.
    \]
    Similarly, when $\sum_{t=1}^{T}(x_t - p_t) < 0$, we have
    \[
        \sum_{t=1}^{T}(x_t - p_t) \cdot\sgn(0 - p_t)
    =   -\sum_{t=1}^{T}(x_t - p_t) + \sum_{t=1}^{T}(x_t - p_t) \cdot \1{p_t = 0}
    \ge -\sum_{t=1}^{T}(x_t - p_t),
    \]
    which implies
    \[
        \left|\sum_{t=1}^{T}(x_t - p_t)\right|
    =   -\sum_{t=1}^{T}(x_t - p_t)
    \le \sum_{t=1}^{T}(x_t - p_t) \cdot\sgn(0 - p_t)
    \le \sup_{\alpha \in [0, 1]}\left|\sum_{t=1}^{T}(x_t - p_t)\cdot\sgn(\alpha - p_t)\right|.
    \]
    
    \paragraph{The upper bound part.} Now we upper bound $\stepCE(x, p)$. It suffices to show that
    \[
        \left|\sum_{t=1}^{T}(x_t - p_t)\cdot\1{p_t \le \alpha}\right|
    \le \sup_{\beta \in [0, 1]}\left|\sum_{t=1}^{T}(x_t - p_t)\cdot\sgn(\beta - p_t)\right|
    \]
    holds for every $\alpha \in [0, 1]$.

    When $\alpha = 1$, the above is exactly given by \Cref{eq:total-bias-bound}. When $\alpha < 1$, we can always find $\alpha' > \alpha$ such that $\{p_1, p_2, \ldots, p_T\} \cap (\alpha, \alpha'] = \emptyset$. Then, for every $t \in [T]$, we have $\1{p_t \le \alpha} = \1{p_t \le \alpha'}$. Furthermore, since $\alpha' \notin \{p_1, p_2, \ldots, p_T\}$, we have $\1{p_t \le \alpha'} = \1{0 \le \alpha' - p_t} = \tfrac 12 (1 + \sgn(\alpha' - p_t))$. It follows that
    \begin{align*}
        \left|\sum_{t=1}^{T}(x_t - p_t)\cdot\1{p_t \le \alpha}\right|
    &=   \left|\sum_{t=1}^{T}(x_t - p_t)\cdot\frac{1}{2}\left(1 + \sgn(\alpha' - p_t)\right)\right|\\
    &\le \frac{1}{2}\left|\sum_{t=1}^{T}(x_t - p_t)\right| + \frac{1}{2}\left|\sum_{t=1}^{T}(x_t - p_t)\cdot\sgn(\alpha' - p_t)\right|\\
    &\le \sup_{\beta \in [0, 1]}\left|\sum_{t=1}^{T}(x_t - p_t)\cdot\sgn(\beta - p_t)\right|. \tag{\Cref{eq:total-bias-bound}}
    \end{align*}
    This proves the upper bound on $\stepCE(x, p)$.

    \paragraph{The lower bound part.} For the other direction, it suffices to prove that
    \[
        \left|\sum_{t=1}^{T}(x_t - p_t)\cdot\sgn(\alpha - p_t)\right|
    \le 3\stepCE(x, p)
    \]
    holds for every $\alpha \in [0, 1]$. Note that
    \[
        \sgn(\alpha - p_t)
    =   \1{p_t \le \alpha} - \1{p_t \ge \alpha}
    =   \1{p_t \le \alpha} - \1{p_t \le 1} + \1{p_t < \alpha}.
    \]
    If $\alpha = 0$, the last indicator $\1{p_t < \alpha}$ is always zero, and we have
    \[
        \left|\sum_{t=1}^{T}(x_t - p_t)\cdot\sgn(\alpha - p_t)\right|
    \le \left|\sum_{t=1}^{T}(x_t - p_t)\cdot\1{p_t \le \alpha}\right| + \left|\sum_{t=1}^{T}(x_t - p_t)\cdot\1{p_t \le 1}\right|
    \le 2\stepCE(x, p).
    \]
    When $\alpha > 0$, we can always find $\alpha' < \alpha$ such that $\{p_1, p_2, \ldots, p_T\} \cap (\alpha', \alpha) = \emptyset$. Then, $\1{p_t < \alpha} = \1{p_t \le \alpha'}$ holds for every $t \in [T]$, and it follows that
    \begin{align*}
        &~\left|\sum_{t=1}^{T}(x_t - p_t)\cdot\sgn(\alpha - p_t)\right|\\
    =   &~\left|\sum_{t=1}^{T}(x_t - p_t)\cdot\left(\1{p_t \le \alpha} - \1{p_t \le 1} + \1{p_t \le \alpha'}\right)\right|\\
    \le &~\left|\sum_{t=1}^{T}(x_t - p_t)\cdot\1{p_t \le \alpha}\right| + \left|\sum_{t=1}^{T}(x_t - p_t)\cdot\1{p_t \le 1}\right| + \left|\sum_{t=1}^{T}(x_t - p_t)\cdot\1{p_t \le \alpha'}\right|\\
    \le &~3\stepCE(x, p).
    \end{align*}
    This proves the lower bound on $\stepCE(x, p)$.
\end{proof}

\subsection{Step Calibration and Its Subsampled Variant}

\begin{lemma}\label{lemma:stepCE-vs-stepCEsub}
    For any $x \in \{0, 1\}^T$ and $p \in [0, 1]^T$,
    \[
        \frac{1}{2}\stepCE(x, p)
    \le \stepCEsub(x, p)
    \le \frac{1}{2}\stepCE(x, p) + O(\sqrt{T}).
    \]
\end{lemma}

\begin{proof}
    Fix $x \in \{0, 1\}^T$ and $p \in [0, 1]^T$. For the lower bound part, suppose that the supremum in $\stepCE(x, p)$ is achieved at $\alpha^\star$,\footnote{$\alpha^\star$ is well defined, as the term in the supremum takes at most $T + 1$ different values over all $\alpha \in [0, 1]$.} i.e.,
    \[
        \stepCE(x, p)
    =   \left|\sum_{t=1}^{T}(x_t - p_t)\cdot\1{p_t \in [0, \alpha^\star]}\right|.
    \]
    Then, we have
    \begin{align*}
            \stepCEsub(x, p)
    &\ge    \Ex{y \sim \Unif(\{0, 1\}^T)}{\left|\sum_{t=1}^{T}y_t\cdot(x_t - p_t)\cdot\1{p_t \in [0, \alpha^\star]}\right|}\\
    &\ge    \left|\Ex{y \sim \Unif(\{0, 1\}^T)}{\sum_{t=1}^{T}y_t\cdot(x_t - p_t)\cdot\1{p_t \in [0, \alpha^\star]}}\right| \tag{convexity of $x \mapsto |x|$}\\
    &=  \frac{1}{2}\left|\sum_{t=1}^{T}(x_t - p_t)\cdot\1{p_t \in [0, \alpha^\star]}\right|
    =   \frac{1}{2}\stepCE(x, p).
    \end{align*}

    For the upper bound, we assume without loss of generality that $p_1 \le p_2 \le \cdots \le p_T$, since both $\stepCE$ and $\stepCEsub$ are invariant to the reordering of the entries (in $x$ and $p$ simultaneously). Let $S \coloneqq \{t \in [T-1]: p_t < p_{t+1}\} \cup \{T\}$. For each $t \in \{0, 1, \ldots, T\}$, define $A_t \coloneqq \sum_{i=1}^{t}(x_i - p_i)$. Over the randomness in $y \sim \Unif(\{0, 1\}^T)$, define
    \[
        X_t \coloneqq \sum_{i=1}^{t}(y_i - 1/2)\cdot(x_i - p_i).
    \]
    Then, $\stepCE(x, p)$ and $\stepCEsub(x, p)$ can be simplified into
    \[
        \stepCE(x, p)
    =   \max_{t \in S}|A_t|,
    \]
    and
    \begin{align*}
        \stepCEsub(x, p)
    &=  \Ex{y}{\max_{t \in S}|A_t / 2 + X_t|}\\
    &\le\frac{1}{2}\max_{t \in S}|A_t| + \Ex{y}{\max_{t \in [T]}|X_t|}\\
    &=  \frac{1}{2}\stepCE(x, p) + \Ex{y}{\max_{t \in [T]}|X_t|}.
    \end{align*}
    
    Therefore, it remains to control the term $\Ex{y}{\max_{t \in [T]}|X_t|}$ by $O(\sqrt{T})$. Note that $(X_t)_{t=0}^{T}$ is a martingale in which each displacement $(X_t - X_{t-1}) \mid X_{t-1}$ has a variance of $(x_t - p_t)^2/4 \le 1/4$. Kolmogorov's inequality implies that, for every $\tau > 0$,
    \[
        \pr{}{\max_{t \in [T]}|X_t| \ge \tau} \le \frac{T/4}{\tau^2}.
    \]
    It follows that
    \begin{align*}
        \Ex{}{\max_{t \in [T]}|X_t|}
    &=   \int_{0}^{+\infty}\pr{}{\max_{t \in [T]}|X_t| \ge \tau}~\rmd\tau\\
    &\le \int_{0}^{+\infty}\min\left\{\frac{T}{4\tau^2}, 1\right\}~\rmd\tau
    =   O(\sqrt{T}).
    \end{align*}
\end{proof}

\end{document}